\pdfoutput=1
\documentclass{article}

\newif\iftr
\trtrue

\iftr
    \usepackage[margin=1in]{geometry}
    \usepackage{authblk}
    \usepackage{lmodern}
    \usepackage[sort,numbers]{natbib}
\else
    \usepackage[preprint]{neurips_2026}
\fi

\usepackage[utf8]{inputenc}
\usepackage[T1]{fontenc}
\usepackage{hyperref}
\usepackage{url}
\usepackage{booktabs}
\usepackage{amsfonts}
\usepackage{amsmath,amssymb,amsthm}
\usepackage{thmtools, thm-restate}
\usepackage{nicefrac}
\usepackage{microtype}
\usepackage{xcolor}
\usepackage{tikz}
\usepackage{standalone}
\usetikzlibrary{positioning, shapes.geometric, arrows.meta}
\usepackage{enumitem}
\usepackage[colorinlistoftodos,textsize=small]{todonotes}
\usepackage[capitalise,noabbrev]{cleveref}
\usepackage{makecell}
\usepackage{mdframed}

\newcommand{\bemph}[1]{\textbf{\emph{#1}}}
\newcommand{\R}{\mathbb{R}}
\newcommand{\Z}{\mathbb{Z}}
\newcommand{\E}{\mathbb{E}}

\DeclareMathOperator{\im}{im}
\DeclareMathOperator{\id}{id}
\DeclareMathOperator{\rank}{rank}
\DeclareMathOperator{\tr}{tr}
\DeclareMathOperator{\mat}{mat}
\DeclareMathOperator{\Span}{span}
\newcommand{\Loss}{\mathcal{L}}
\newcommand{\Tt}{\mathcal{T}}
\newcommand{\GL}{\operatorname{GL}}

\theoremstyle{definition}
\newtheorem{definition}{Definition}[section]
\newtheorem{example}[definition]{Example}
\newtheorem{remark}[definition]{Remark}

\theoremstyle{plain}
\newtheorem{theorem}[definition]{Theorem}
\newtheorem{proposition}[definition]{Proposition}
\newtheorem{lemma}[definition]{Lemma}
\newtheorem{corollary}[definition]{Corollary}
\newenvironment{theoremframe}
  {\begin{mdframed}
   \setlength{\topsep}{0pt}
   \setlength{\partopsep}{0pt}}
  {\end{mdframed}}

\usepackage{tcolorbox}

\newcommand{\expandnote}[1]{}
\newcommand{\checkthis}[1]{}
\newcommand{\cutmaybe}[1]{}
\renewcommand{\comment}[1]{}
\usepackage{tikz}
\colorlet{NodeColor}{blue!25}         %
\colorlet{MarkedColor}{red!25}     %
\colorlet{MarkedEdgeColor}{red!60} %

\tikzset{
  vertex/.style={
    circle,
    draw=black,
    fill=NodeColor,
    inner sep=0pt,
    minimum size=18pt,
    line width=0.8pt
  },
  marked vertex/.style={
    vertex,
    fill=MarkedColor
  },
  edge/.style={
    draw=black,
    line width=0.8pt
  },
  marked edge/.style={
    draw=MarkedEdgeColor,
    line width=2pt
  },
  stub/.style={          %
    draw=black,
    line width=0.8pt
  }
}

\usepackage{listings}
\lstdefinestyle{leanstatement}{
  basicstyle=\ttfamily\footnotesize,
  columns=fullflexible,
  keepspaces=true,
  showstringspaces=false,
  breaklines=true,
  frame=single,
  framerule=0.3pt,
  xleftmargin=0.5em,
  xrightmargin=0.5em,
  aboveskip=0.5em,
  belowskip=0.5em,
  literate=
    {θ}{{$\theta$}}1
    {ℝ}{{$\mathbb{R}$}}1
    {→}{{$\to$}}1
    {∧}{{$\land$}}1
    {¬}{{$\neg$}}1
}

\usepackage{graphicx}
\usepackage{hyperref}
\usepackage{tikz}
\usetikzlibrary{calc}

\newif\ifleanmarkers
\leanmarkerstrue

\newcommand{\leanbase}{https://github.com/yangdabei/ttn-loss-landscape/blob/d42cf6052a49c13a73eefbce2acb4808b42f5e20}

\newcounter{leanstatementmark}

\newcommand{\leanstatementlink}[2]{%
  \ifleanmarkers
    \stepcounter{leanstatementmark}%
    \begingroup
    \edef\currentleanstatementmark{\arabic{leanstatementmark}}%
    \tikz[remember picture,overlay,baseline=0pt]{%
      \coordinate (lean-mark-\currentleanstatementmark) at (0,0);
      \node[
        anchor=center,
        inner sep=0pt,
        yshift=0.45ex
      ] at (
        [xshift=-0.50in]current page.east
        |- lean-mark-\currentleanstatementmark
      )
      {\href{\leanbase/#1\#L#2}{%
        \makebox[0.40in][c]{%
          \includegraphics[
            width=0.31in,
            trim=246bp 148bp 245bp 149bp,
            clip
          ]{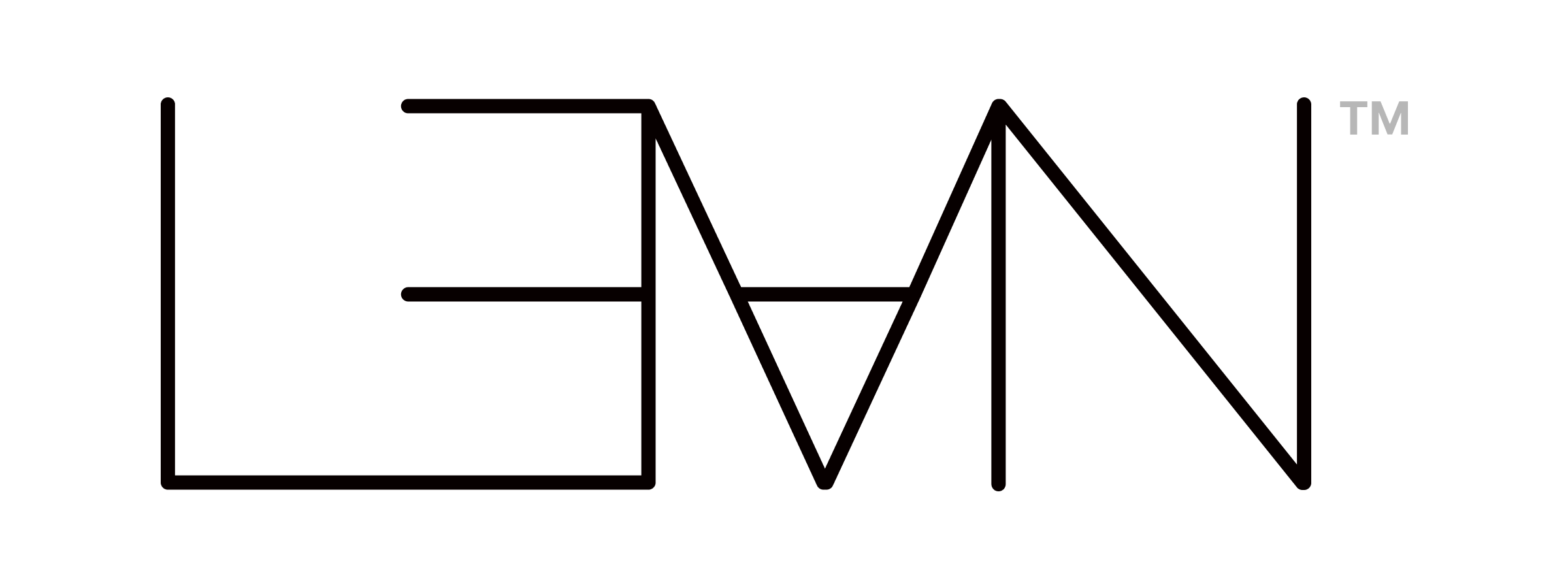}%
        }%
      }};
    }%
    \endgroup
  \fi
}

\title{Benign Loss Landscapes Can Coexist with Worst-Case Hardness}

\iftr
    \newcommand{\AuthorBlock}[3]{%
    \begin{minipage}[t]{0.30\textwidth}
      \centering
      \textbf{#1}\\\vspace{-8pt}
      #2\\\vspace{3pt}
      \normalsize\texttt{#3}
    \end{minipage}%
  }

  \author{%
    \makebox[\textwidth][c]{%
      \AuthorBlock
        {Zach Furman}
        {University of Melbourne}
        {zach.furman1@gmail.com}%
      \hfill
      \AuthorBlock
        {Stephan Wäldchen}
        {Iliad}
        {st.wald@protonmail.ch}%
      \hfill
      \AuthorBlock
        {Yangda Bei}
        {Iliad}
        {yangdabei3426@gmail.com}%
      \hfill
      \AuthorBlock
        {Liam Hodgkinson}
        {University of Melbourne}
        {lhodgkinson@unimelb.edu.au}%
    }%
  }

  \date{}
\else
    \author{%
    Zach Furman\\
    University of Melbourne, Iliad\\
    \texttt{zach.furman1@gmail.com}
    \And
    Stephan Wäldchen\\
    Iliad\\
    \texttt{st.wald@protonmail.ch}
    \And
    Yangda Bei\\
    Iliad\\
    \texttt{yangdabei3426@gmail.com}
    \And
    Liam Hodgkinson\\
    University of Melbourne\\
    \texttt{lhodgkinson@unimelb.edu.au}
    }
\fi
\date{}

\begin{document}

\maketitle

\begin{abstract}

Deep neural networks are expressive enough to contain worst-case targets that can be \emph{evaluated} in polynomial time but cannot be \emph{learned} in polynomial time by gradient descent. For practical tasks they nonetheless learn well, raising the question of what non-generic structure of real-world targets enables this. Existing surrogate models cannot pose this question because they either lack hard-to-learn targets entirely (deep linear networks) or cannot evaluate such targets efficiently (kernel methods, infinite-width limits). We study tree tensor networks (TTNs), a model class that generalizes deep linear networks and Tucker decompositions. We show they embed arbitrary read-once Boolean formulas, and thus contain polynomial-size targets that cannot be learned by gradient descent in polynomial time under the same mechanism as neural networks. Despite this, we prove that their loss landscapes are conditionally benign for every realizable target: every local minimum that is \textit{minimum-norm} is global. Thus, surprisingly, bad local minima are not what distinguishes between typical and worst-case problems in TTNs. Instead, learning difficulty in TTNs can arise from high-order degenerate saddle points, which we show are caused by rank-deficiency. This is explored through a case study of the parity function, illustrating the potential for TTNs to relate landscape geometry to computational hardness.

\end{abstract}

\section{Introduction}
\label{sec:intro}

One of the core mysteries of deep learning theory is why gradient-based optimizers reliably find high-quality minima, even when navigating such a highly non-convex loss landscape. This phenomenon has led to the hypothesis that these landscapes are \emph{benign}, in the sense that all local minima are global, or at least that non-increasing escape paths exist for any suboptimal minima. Such results have been proven for models like deep linear networks \citep{baldi1989neural, kawaguchi2016deep, lu2017depth} or very wide networks \citep{venturi2018spurious,nguyen2019connected,petzka2021non}, and are implicit in approximate models with convex losses like the neural tangent kernel \citep{jacot2018neural}. However, it remains unclear whether these ``toy models'' actually reflect the real difficulty of neural network training, or merely achieve benignity by assuming away the core difficulties of the deep neural network loss landscape.

More pointedly, any account of why training succeeds must contend with the following fact: there exist targets like the \textit{parity function} which neural networks can express perfectly, but cannot efficiently learn via gradient descent \citep{shalev2017failures, shamir2018distribution}.
Given the empirical success of deep learning, such worst-case results are clearly unrealistic. But they challenge any simplistic account of why neural networks learn efficiently, as any explanation which works for arbitrary targets is insufficient. Somehow, deep learning must leverage some property of practical targets, not shared by worst-case targets, that allows them to learn in polynomial time.
Consequently, existing benign landscape results, and existing surrogate models more broadly, do not preserve the phenomenon we want to study. Doing so requires two properties of the surrogate: (A) its hypothesis class must contain worst-case targets which cannot be learned in polynomial time by gradient descent, like the parity function; and (B) hypotheses in the class must themselves be evaluable in polynomial time; otherwise the learning difficulty is trivial and unsurprising.

\begin{figure}[t]
\centering
\begin{minipage}[c]{0.42\linewidth}
\centering
\begin{tikzpicture}[
    scale=0.85, transform shape,
    tensor/.style={circle, draw, fill=blue!25, minimum size=6mm,
                   inner sep=0pt, thick},
    leg/.style={thick},
    bond/.style={thick},
    title/.style={font=\small\bfseries}
]
\foreach \i [evaluate=\i as \y using int(-\i+1)] in {1,...,4} {
    \node[tensor] (d\i) at (0, \y) {};
}
\foreach \i [evaluate=\i as \j using int(\i+1)] in {1,...,3} {
    \draw[bond] (d\i) -- (d\j);
}
\draw[leg] (d1) -- ++(0, 0.8);
\draw[leg] (d4) -- ++(0, -0.8);
\node[title] at (0, -4.4) {DLN};
 
\draw[-{Stealth[length=3mm, width=2.5mm]}, very thick]
    (1.0, -1.5) -- (3.0, -1.5)
    node[midway, above, font=\scriptsize\itshape] {generalizes to};
 
\node[tensor] (r)  at (5.0,  0)    {};
\node[tensor] (m1) at (4.2, -1.2)  {};
\node[tensor] (m2) at (5.8, -1.2)  {};
\draw[bond] (r) -- (m1);
\draw[bond] (r) -- (m2);
\node[tensor] (l1) at (3.8, -2.4) {};
\node[tensor] (l2) at (4.6, -2.4) {};
\node[tensor] (l3) at (5.4, -2.4) {};
\node[tensor] (l4) at (6.2, -2.4) {};
\draw[bond] (m1) -- (l1);
\draw[bond] (m1) -- (l2);
\draw[bond] (m2) -- (l3);
\draw[bond] (m2) -- (l4);
\foreach \i in {1,...,4} {
    \draw[leg] (l\i) -- ++(0, -0.8);
}
\node[title] at (5.0, -4.4) {TTN};
\end{tikzpicture}
\end{minipage}%
\hfill
\begin{minipage}[c]{0.56\linewidth}
\centering
\footnotesize
\setlength{\tabcolsep}{4pt}
\renewcommand{\arraystretch}{1.2}
\begin{tabular}{@{}lccc@{}}
\toprule
& \makecell{Benign loss\\landscape}
& \makecell{Can express\\hard functions}
& \makecell{Can evaluate\\hard functions\\in poly time} \\
\midrule
Typical neural network        & empirically & \checkmark & \checkmark \\
Deep linear network           & provably    & $\times$  & -  \\
Infinite-width limits         & provably    & \checkmark  & $\times$ \\
Kernel approximations         & provably    & \checkmark  & $\times$ \\
\midrule
Tree tensor network           & provably    & \checkmark & \checkmark \\
\bottomrule
\end{tabular}
\end{minipage}

\vspace{1em}
\vspace{0.1em}

\noindent
\begin{minipage}[b]{0.48\linewidth}
  \centering
  \includegraphics[width=\linewidth]{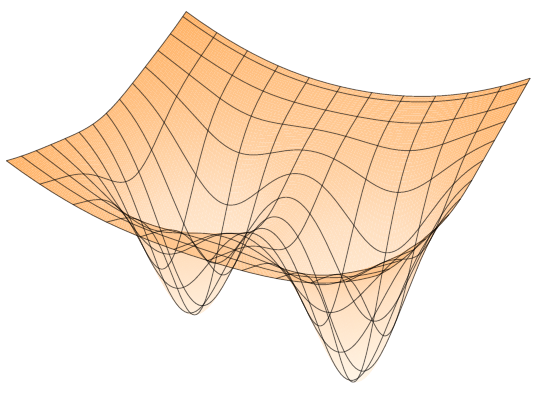}\\[0.3em]
  \textit{Not benign}
\end{minipage}\hfill
\begin{minipage}[b]{0.48\linewidth}
  \centering
  \includegraphics[width=\linewidth]{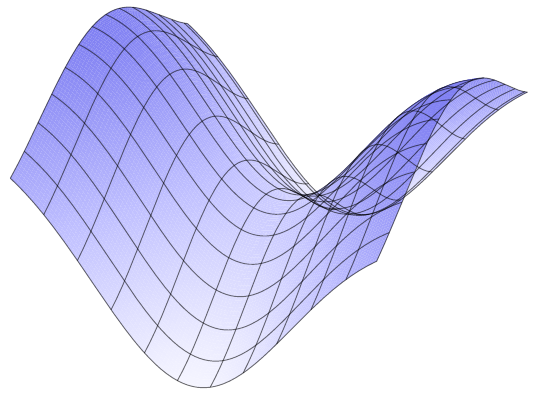}\\[0.3em]
  \textit{Benign}
\end{minipage}

\caption{\textbf{Top:} a deep linear network (DLN) is a chain of tensors; a tree tensor
network (TTN) generalizes the DLN to a tree topology (left). Unlike the DLN, the TTN
captures nonlinear features and functions which are provably hard-to-learn while retaining a provably benign loss landscape (right). \textbf{Bottom:} our main results, \cref{thm:no-spurious} and \cref{thm:no-spurious-local-min} prove that under some conditions, the only critical points of the loss other than global minima are saddle points (right); there are no local minima which are not global (left).}
\label{fig:dln-vs-ttn}
\end{figure}
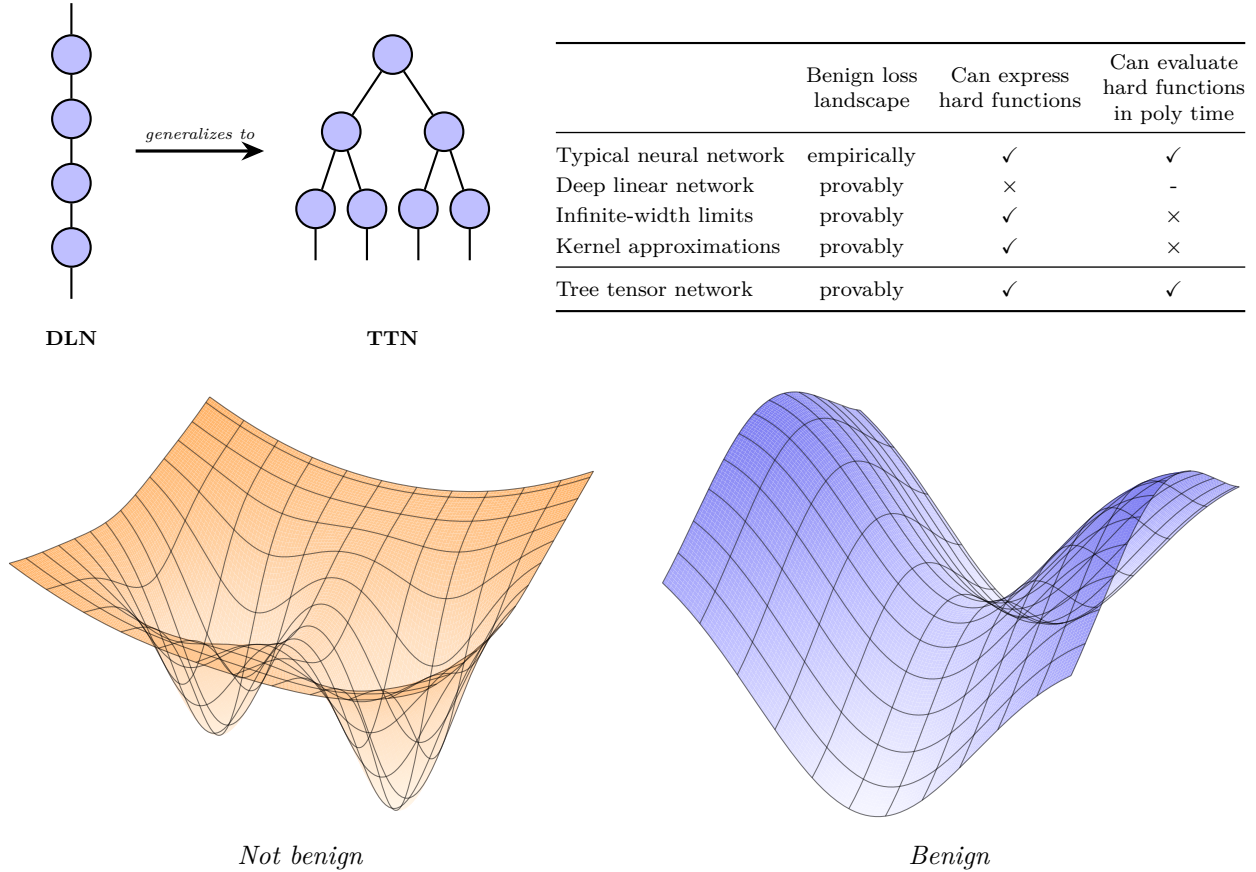

Every commonly-studied surrogate we are aware of fails one of these properties (A) or (B). Deep linear networks parameterize only linear functions and contain no hard-to-learn targets, failing (A). Infinite-width limits (without closed-form kernels) cannot be analytically computed in the first place, failing (B). Kernel methods (random feature models, neural tangent kernels, etc.) must fail either the first or the second condition because a kernel hypothesis $\hat f(x) = \sum_i a_i k(x, x_i)$ takes the same resources to evaluate as to learn, so it cannot be polynomial-time evaluable and super-polynomial to learn.

In this paper, we study \bemph{tree tensor networks} (TTNs) and demonstrate that they pass both properties (A) and (B). A TTN is defined by a tree whose nodes carry multilinear maps (tensors), with the model's output obtained by composing maps along edges (\cref{fig:dln-vs-ttn}). TTNs are an immediate nonlinear generalization of deep linear networks (DLNs) and Tucker decomposition. Most importantly, TTNs satisfy our two desiderata for surrogate models: they can express any read-once Boolean formula, including worst-case targets like the parity function (A), and any TTN can be evaluated in time polynomial in the tree size and bond dimension (B).
TTNs therefore contain targets that cannot be learned in polynomial time by gradient descent (\cref{sec:setup}).

Our contribution is to leverage the analytic tractability of TTNs to understand how worst-case difficulty manifests in their loss landscapes. We show that, surprisingly, even for worst-case targets, their loss landscapes are conditionally benign, so all minimum-norm local minima are global minima (\cref{thm:no-spurious-local-min}).
Even without the minimum-norm condition, we prove that from any non-optimal parameter point there exists a continuous non-increasing loss path to strictly lower loss, that is, there are no ``bad valleys'' (\cref{cor:escape-paths}).

Yet, if the loss landscapes of such a model are benign, how is that compatible with worst-case hardness? We suggest that instead such difficulty can come from high-order degenerate saddle points, as suggested by \citet{dauphin2014identifying, sankar2017saddles}. We characterize these saddle points, proving that any critical points which are not global minima must occur at Tucker-rank-deficient points, where the model is ``effectively smaller'' than its maximum capacity (\cref{thm:no-spurious}). The parity function provides a concrete case study of a known hard-to-learn function to illustrate this phenomenon, proving that the order at which a descent path can escape the saddle point around the ``halfway learned'' solution must be lower bounded by a linearly increasing function of the number of inputs $n$ (making the loss function exponentially flat in $n$).

We hope that this paper will help establish TTNs as a new surrogate model for deep learning training dynamics. Because TTNs express computationally nontrivial functions, they provide the first analytically tractable model class designed to systematically study how computational structure of the target shapes the geometry of the optimization landscape.

\textbf{Contributions.} Our precise contributions are thus:\vspace{.1cm}
\begin{itemize}[leftmargin=*,nosep]
\item \textbf{We prove TTNs lack bad local minima,} under a \textit{minimum-norm} assumption. Despite expressing targets which cannot be learned efficiently by gradient descent (\cref{cor:ttn-hardness}), we show that for any realizable target, every minimum-norm local minimum is globally optimal (\cref{thm:no-spurious-local-min}), and all other parameters admit non-increasing escape paths (\cref{cor:escape-paths}).
\item \textbf{We show rank-deficiency instead creates bad saddle points.} In particular, we prove that all non-optimal critical points are Tucker-rank-deficient (\cref{thm:no-spurious}). Through a case study involving the parity function (\cref{prop:parity-order}), we show that these saddles can be highly degenerate with the potential to slow optimization. 
\item \textbf{We connect this geometry to broader theory.} As an auxiliary technical contribution, we show that our minimum-norm condition coincides with the \textit{Kempf-Ness condition} from geometric invariant theory, generalizes a well-known \textit{balancedness} condition in deep linear networks, and is preserved by gradient flow (\cref{app:minimum-norm}).
\end{itemize}\vspace{.1cm}

We verify the two main theorems (\cref{thm:no-spurious-local-min}, \cref{thm:no-spurious}) using the \textsc{Lean} proof assistant \citep{mouraLean4Theorem2021,LeanMathematicalLibrary2020}: links to the corresponding \textsc{Lean} code are provided in the right margin, and more details are given in \cref{app:formalisation}.

Ultimately, we establish the following picture of the TTN loss landscape, to inform the broader question of how training behavior depends on the function being learned: 

\begin{tcolorbox}
For regularized solutions, there are no spurious local minima. However, challenging saddle points can still occur when the model is low rank. These can impede convergence to worst-case targets, making them hard to learn. 
\end{tcolorbox}

\section{Related Work}
\label{sec:related}

\paragraph{Tensor networks}

Tensor networks originated in quantum state and many-body physics as a memory-efficient approximation to high-dimensional multi-particle states with path-graph~\cite{white1992density,ostlund1995thermodynamic,schollwock2011density,vidal2003efficient}, lattice~\cite{verstraete2004renormalization,cirac2021matrix}, and tree topology~\cite{shi2006classical}.
In machine learning, tensor networks have been employed both for supervised~\cite{stoudenmire2016supervised} and unsupervised learning~\cite{han2018unsupervised,miller2021tensor} (in both cases later extended to tree topologies~\citep{stoudenmire2018learning,cheng2019tree}) leveraging their ability to efficiently represent high-dimensional feature tensors or multi-feature interactions~\cite{novikov2016exponential}. Interestingly, \citet{pearce2024bilinear} and \citet{dooms2025compositionality} show that neural networks with bilinear activation function are equivalent to tree tensor networks with a copying operation, which provides a natural direction to extend our results in future work.

\paragraph{Benign loss landscapes.}

\citet{baldi1989neural} proved that every local minimum of the square loss is global for a two-layer linear network. \citet{kawaguchi2016deep} extended it to arbitrary depth and widths, and further showed that every
critical point which is not a global minimum is a saddle. Later work extended and improved upon these results in deep linear networks \citep{lu2017depth,yun2017global, zhou2017critical,achour2024loss,trager2019pure,mehta2021loss,bah2022learning,arora2018convergence}. Beyond DLNs, \citet{frandsen2022optimization} show a similar lack of local minima for Tucker decomposition, though requiring additional assumptions such as regularization and realizability. Tree tensor networks generalize both Tucker decomposition and deep linear networks; we prove conditional benign loss landscape results for arbitrary tree tensor networks (\cref{thm:no-spurious-local-min}).

It is known that some nonlinear neural networks \emph{can} have spurious local minima \cite{ding2019spurious,swirszcz2016local,safran2018spurious,christof2024omnipresence}; beyond the technical fact that such results do not apply to TTNs, our results do not contradict this even in spirit as they are conditional (see \cref{subsec:tightness}), instead showing that such local minima are not \textit{inherent} to nonlinear networks and can be avoided in a specific regime. For many types of neural networks, one can prove a weaker ``no bad valleys'' condition \cite{venturi2018spurious,nguyen2019connected,nguyen2018loss,petzka2021non}, similar to our \cref{cor:escape-paths}, but only under extreme overparameterization.

\paragraph{Learning computationally hard targets.}

\citet{kearns1998efficient} showed that parities are not learnable in polynomial time by statistical
query algorithms; this applies to full-batch gradient descent, as it can be formulated as a statistical query algorithm \citep{feldman2021statistical, goel2020superpolynomial, shalev2017failures}. \citet{abbe2021power} showed that such results transfer to minibatch SGD, under limits on gradient precision and minibatch size.
A more recent line of work characterizes how such hard targets affect gradient-based learning dynamics. \citet{barak2022hidden} use the problem of sparse parity learning to study the learning dynamics of emergent capabilities. \citet{malach2019learning} study gradient-based learning when the target is a
\emph{tree-structured Boolean circuit}, and explore what differentiates parity targets from learnable targets (a \emph{local correlation assumption}).  \citet{abbe2022merged} and \citet{abbe2023sgd} explore a similar question for Boolean functions, and relate computational properties (\emph{leap complexity}) of the target to saddle-to-saddle learning dynamics. However, while these results explicitly predict gradient learning dynamics, they apply only for particular families of targets; our results apply for every realizable target, including worst-case ones.

\paragraph{Degenerate saddle points.}
For neural networks, it has been argued that the primary obstructions to learning is saddle points rather than local minima~\cite{choromanska2015loss}, and specifically \emph{degenerate} saddle points ~\cite{dauphin2014identifying,sankar2017saddles}.
 That these saddles determine gradient dynamics for small initializations has been shown for both deep linear~\cite{saxe2013exact,arora2019implicit,chou2024gradient,jacot2021saddle,pesme2023saddle} and neural networks~\cite{zhang2026saddle}.
Guarantees for escaping saddle points in polynomial time apply only to \textit{strict} saddles, i.e., those which have a strictly negative curvature direction~\cite{ge2015escaping,jin2017escape,daneshmand2018escaping}.
\citet{anandkumar2016efficient} show that without strictness, finding a fourth-order local minimum\footnote{A critical
point of a function $f$ is a \emph{$p$-th order local minimum} when $f(y) - f(x) \geq - C\|x-y\|^{p+1}$ near $x$.} can be NP-hard even for well-behaved functions. Unlike deep linear networks, where saddle order depends primarily on model depth~\cite{jacot2021saddle}, we are able to explicitly link a hard-to-learn target, the parity function with $n$ inputs, to saddles of order at least $n/2$ (\cref{prop:parity-order}). We conjecture that a similar mechanism may occur for hard targets more generally.

\section{Tree Tensor Networks}
\label{sec:setup}

The theoretical study of deep learning dynamics is most commonly conducted using simple \emph{toy models} retaining some essential feature(s) of the problem while discarding extraneous details. Among the most popular of these is the \emph{deep linear network (DLN)}, defined for a depth $L$ as a composition of $L$ linear maps $W_i : \mathbb{R}^{n_{i-1}} \to \mathbb{R}^{n_i}$:
\begin{equation}
\tag{DLN}
    T = W_L W_{L-1} \cdots W_1.
\end{equation}
Naturally, a DLN is itself a linear map. However, as a product of weight matrices $W_i \in \mathbb{R}^{n_i \times n_{i-1}}$, a DLN is nonlinear and nonconvex in all of its weights, and any loss function involving a DLN exhibits the same multiplicative
interactions between layers that are characteristic of deep neural networks. Consequently, DLNs preserve the nonlinear dynamics characteristic of deep learning, while enjoying an easily interpretable, and theoretically tractable, relationship over its inputs. 
This isolation has been productive. It is known that every local minimum is a global minimum, and
rank-deficient critical points are saddle points \citep{kawaguchi2016deep,
yun2017global}. The gradient flow admits exact solutions \citep{saxe2013exact},
and its dynamics decompose into the independent evolution of singular modes,
exhibiting stagewise learning of components from largest to smallest. Implicit
regularization toward low-rank solutions has been characterized \citep{gunasekar2017implicit, arora2019implicit, li2020towards}, and much of these phenomena appear to occur analogously in more standard nonlinear networks \citep{kalimeris2019sgd, huh2023low, zhang2026saddle}.

Despite their success, DLNs exhibit fundamental properties that limit their mimicry of deep learning dynamics. The intermediate nonlinearities of neural networks allow them to build complex, hierarchical logic, where earlier computations gradually assemble to make later mechanisms useful. This is central to modern accounts of how deep networks learn \citep{olah2020zoom, abbe2021staircase, abbe2023sgd}. In contrast, the linearity of DLNs enforces that each component of the target is learned independently of any other. Because their end-to-end map is linear and their gradient flow is structurally decoupled, DLNs can neither represent this hierarchical logic as a function nor exhibit it dynamically during training.

\emph{Tree tensor networks (TTNs)} extend the DLN strategy to a class of models which
support compositional structure while preserving the algebraic features
that make DLN analysis tractable. The construction relaxes the DLN in two key ways: (i) each matrix is replaced by a higher-order tensor; and (ii) the chain structure is generalized to a tree. A matrix
represents a linear map between two vector spaces, while a tensor with $k$ indices
represents a $k$-linear map, separately linear in each input and \textit{nonlinear}
jointly. The product $(x,y) \mapsto xy$ is the simplest example: linear in $x$
for fixed $y$, linear in $y$ for fixed $x$, but \textit{not linear} as a function of
the pair. Furthermore, a matrix represents
a map with one input and one output, so composing matrices forces a chain topology, where each matrix takes its predecessor's output as its input. Multilinear maps have several inputs, and can therefore receive them from several distinct sources at once. This enables branching, with trees comprising the
simplest branching topology. For examples of tree tensor networks, including deep linear networks and Tucker decomposition, see \cref{fig:setup-examples}.

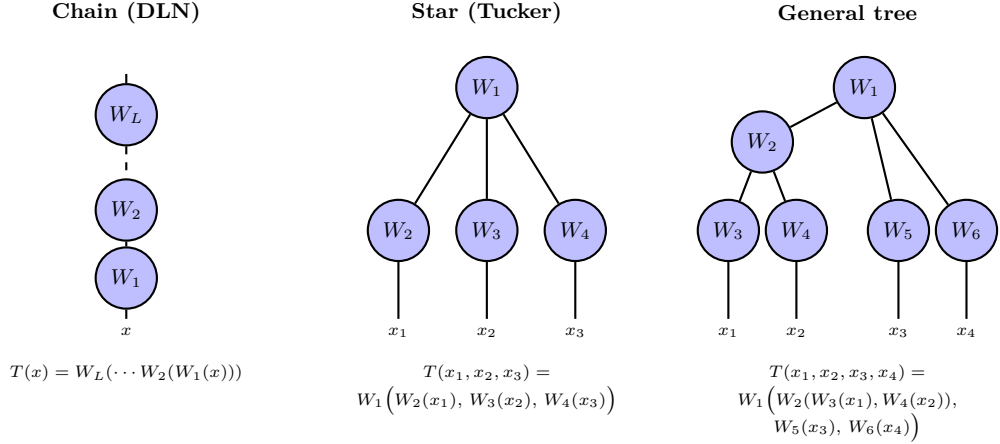
\begin{figure}[t]
\centering
\begin{tikzpicture}[
    scale=0.9, transform shape,
    tensor/.style={circle, draw, fill=blue!25, minimum size=9mm,
                   inner sep=1pt, thick, font=\small},
    leg/.style={thick},
    bond/.style={thick},
    title/.style={font=\small\bfseries},
    formula/.style={font=\scriptsize},
    leglabel/.style={font=\scriptsize}
]

\begin{scope}[xshift=-5.3cm]
    \node[title]   at (0,  2.5) {Chain (DLN)};
    \node[tensor] (wL) at (0,  1.0) {$W_L$};
    \node[tensor] (w2) at (0, -0.4) {$W_2$};
    \node[tensor] (w1) at (0, -1.4) {$W_1$};
    \draw[bond, dashed] (wL) -- (w2);
    \draw[bond] (w2) -- (w1);
    \draw[leg] (wL) -- ++(0,  0.6);
    \draw[leg] (w1) -- ++(0, -0.6) node[leglabel, below] {$x$};
    \node[formula] at (0, -2.8) {$T(x) = W_L(\cdots W_2(W_1(x)))$};
\end{scope}

\begin{scope}[xshift=0cm]
    \node[title]  at (0,     2.5) {Star (Tucker)};
    \node[tensor] (W1) at (0,     1.4) {$W_1$};
    \node[tensor] (W2) at (-1.3, -0.7) {$W_2$};
    \node[tensor] (W3) at (0,    -0.7) {$W_3$};
    \node[tensor] (W4) at (1.3,  -0.7) {$W_4$};
    \draw[bond] (W1) -- (W2);
    \draw[bond] (W1) -- (W3);
    \draw[bond] (W1) -- (W4);
    \draw[leg] (W2) -- ++(0, -1.3) node[leglabel, below] {$x_1$};
    \draw[leg] (W3) -- ++(0, -1.3) node[leglabel, below] {$x_2$};
    \draw[leg] (W4) -- ++(0, -1.3) node[leglabel, below] {$x_3$};
    \node[formula] at (0, -2.8) {$T(x_1,x_2,x_3) =$};
    \node[formula] at (0, -3.2) {$W_1\bigl(W_2(x_1),\, W_3(x_2),\, W_4(x_3)\bigr)$};
\end{scope}

\begin{scope}[xshift=5.3cm]
    \node[title]  at ( 0.0,  2.5) {General tree};
    \node[tensor] (W1) at ( 0.25,  1.4) {$W_1$};
    \node[tensor] (W2) at (-1.25,  0.6) {$W_2$};
    \node[tensor] (W3) at (-1.75, -0.7) {$W_3$};
    \node[tensor] (W4) at (-0.75, -0.7) {$W_4$};
    \node[tensor] (W5) at ( 0.75, -0.7) {$W_5$};
    \node[tensor] (W6) at ( 1.75, -0.7) {$W_6$};
    \draw[bond] (W1) -- (W2);
    \draw[bond] (W1) -- (W5);
    \draw[bond] (W1) -- (W6);
    \draw[bond] (W2) -- (W3);
    \draw[bond] (W2) -- (W4);
    \draw[leg] (W3) -- ++(0, -1.3) node[leglabel, below] {$x_1$};
    \draw[leg] (W4) -- ++(0, -1.3) node[leglabel, below] {$x_2$};
    \draw[leg] (W5) -- ++(0, -1.3) node[leglabel, below] {$x_3$};
    \draw[leg] (W6) -- ++(0, -1.3) node[leglabel, below] {$x_4$};
    \node[formula] at (0, -2.8) {$T(x_1,x_2,x_3,x_4) =$};
    \node[formula] at (0, -3.2) {$W_1\bigl(W_2(W_3(x_1), W_4(x_2)),$};
    \node[formula] at (0, -3.6) {$W_5(x_3),\, W_6(x_4)\bigr)$};
\end{scope}
\end{tikzpicture}

\caption{Examples of tree tensor networks, with every node a trainable weight tensor $W_v$. A path graph (left) recovers the deep linear network. A star graph (center) recovers the Tucker decomposition: leaf maps $W_2, W_3, W_4$ act on inputs and feed their outputs into a multilinear core $W_1$. General trees (right) are hierarchical, combining the outputs of disjoint subtrees through a hierarchy of internal multilinear maps. The formula under each panel gives the explicit nested expression for that topology.}
\label{fig:setup-examples}
\end{figure}

\begin{definition}[Tree tensor network]
\label{def:ttn}
A \emph{tree tensor network} consists of:
\begin{itemize}[leftmargin=*]
    \item a finite set $V$ of \emph{nodes} and a finite set $E$ of \emph{edges}, where each $e \in E$ has either one or two endpoints in $V$, such that the subnetwork of two-ended edges is a directed tree on $V$, with edges in $E$ oriented towards the root;
    \item a dimension $d_e \in \Z_{>0}$ for each edge $e \in E$; and
    \item a trainable multilinear map
    \(
        W_v:
        \prod_{e\in \operatorname{in}(v)}
        \mathbb R^{d_e}
        \longrightarrow
        \mathbb R^{d_{\operatorname{out}(v)}}
    \)
    at each node \(v\in V\), where \(\operatorname{in}(v)\) is the set of edges oriented into \(v\), and \(\operatorname{out}(v)\) is the unique outgoing edge from \(v\) to its parent (with the convention \(d_{\operatorname{out}(v)}=1\) if \(v\) is the root).
\end{itemize}
Edges with two endpoints are \emph{internal} (or \emph{bonds}, with \emph{bond dimension} $r_e := d_e$) and edges with one endpoint \emph{external}.
To parameterize the model, we may choose bases to identify each $W_v$ with a real tensor, as multilinear maps form a tensor product space.
We denote $\theta = \{W_v\}_{v \in V}$ for the collection of all node maps and $\Theta$ for the parameter space. Given vectors on the external edges, the \emph{represented tensor} \(T(\theta)\) is a multilinear map obtained by function composition: composing the maps \(W_v\) from the leaves to the root according to the directed graph structure. By multilinearity, in coordinates this composition is given by tensor contraction along bonds, and it is independent of the choice of root (\cref{app:multilinear-background}).
\end{definition}

We study learning an unknown target tensor $T^*$ from random samples of its entries.\footnote{Note that our choice here is not unique: we could choose a different empirical learning problem with the same population loss. We could instead choose to consider some external modes as outputs rather than all being inputs, add label noise, sample the inputs from a different distribution, etc; our choice here is purely for concreteness.} That is, let $\mathcal{X} = [d_1] \times \cdots \times [d_n]$ index the entries of $T^*$, and sample inputs $x \in \mathcal{X}$ uniformly from $\mathcal{X}$, with labels $y \in \R$ generated via evaluating the tensor contraction $y = T^*(x)$.
Define the per-sample loss:
\[
\ell(x, y; \theta) = \frac{1}{2}(y - T(\theta)(x))^2,
\]

Our results concern the population average loss\footnote{See \cref{app:empirical-population-gap} for discussion on what these results imply for empirical learning with SGD.}
\begin{equation}
\label{eq:pop_loss_main}
\Loss(\theta) = \E_{x,y}[\ell(x,y;\theta)]= \frac{1}{2|\mathcal{X}|}\|T(\theta) - T^*\|_F^2.%
\end{equation} %
The main loss-landscape results assume that the target is
\emph{realizable} by the TTN architecture, that is, there exists at least one parameter point
\(\theta^* \in \Theta\) such that $T(\theta^*) = T^*$ (equivalently, \(T^*\) lies in the image of the TTN parameterization). For TTNs, this condition can also be characterized directly as a type of rank condition on $T^*$ (see \cref{app:setup}).

\paragraph{Tree tensor networks can perform computation.}
\label{subsec:expressivity}

Unlike deep linear networks, TTNs are \emph{not} functionally linear, and in fact can perform basic computation. Consider any read-once Boolean formula $F$: a tree of $s$ unary and binary gates in which each input variable appears at most once. We will implement $F$ within a TTN with $s$ nodes and bond dimension $2$ (\cref{fig:bool-to-ttn}). First, encode the Boolean values $0$ and $1$ by the standard basis vectors $e_0=(1,0)^{\top}$ and $e_1=(0,1)^{\top}$ in $\R^2$. Associate a multilinear node map $W_g$ to each gate $g$.
For a binary gate $g$, define its bilinear map $W_g$ by specifying its action on basis vectors:
\[
    W_g(e_a,e_b)=e_{g(a,b)},
    \qquad\text{equivalently}\qquad
    (W_g)_{a,b,c}=\mathbf{1}\{c=g(a,b)\},
    \quad a,b,c\in\{0,1\}.
\]
For example, $W_{\mathrm{AND}}(e_1,e_1)=e_1$, while the other three basis input pairs yield $e_0$. An analogous procedure applies to unary gates; for instance, NOT is represented by the linear map $W_{\mathrm{NOT}}e_a=e_{1-a}$.

Then, connect the gates together: whenever the output of a gate $h$ is an argument of $g$, connect the output of $W_h$ to the corresponding input of $W_g$ by a dimension-$2$ bond; evaluating the network substitutes $W_h$'s output vector into that argument of $W_g$. An argument which is a formula variable $x_i$ instead becomes an external input edge receiving $e_{x_i}$.  One may verify that the overall output of the network on $k$ inputs $e_{x_1}, \ldots e_{x_k}$ is then given by $e_{F(x_1, \ldots, x_k)}$.\footnote{Strictly speaking, for simplicity, \cref{def:ttn} is formulated in such a way that a TTN must have a scalar output. If one wishes to avoid modifying \cref{def:ttn}, this can be achieved (albeit unaesthetically) by setting the root node to be $W_g(e_a,e_b,e_c)=e_{g(a,b)} \cdot e_c$, where the ``output'' $e_c$ is instead treated mathematically as an input to the TTN.}
For example, for the formula in \cref{fig:bool-to-ttn}, the resulting composition is
\(
    W_{\mathrm{OR}}\bigl(W_{\mathrm{AND}}(e_{x_1},e_{x_2}),
    W_{\mathrm{AND}}(e_{x_3},e_{x_4})\bigr)
    = e_{(x_1\wedge x_2)\vee(x_3\wedge x_4)}.
\)

\paragraph{Tree tensor networks are worst-case hard to learn.}

Perhaps surprisingly, the mere fact that TTNs can perform computation
directly implies learning difficulty for gradient-based methods; and
for the exact same reason as it does in real neural networks. Following the standard modeling of gradient descent as a statistical query algorithm \citep[cf.][]{shamir2018distribution,
goel2020superpolynomial}, we say that a concept class is
\emph{learnable by gradient descent in polynomially many steps} if it
is efficiently learnable from statistical queries in the sense of
\citet{kearns1998efficient}, with each query restricted to an
expected gradient (or value) of a bounded per-sample loss of a
polynomial-size differentiable model (for example, the gradient of our loss function $\Loss$).\footnote{Note that for ease of argument, this definition includes full-batch gradient descent, but not stochastic gradient methods. Nevertheless problems which are hard for SQ learners are also provably hard for methods like SGD at moderate batch size \citep{abbe2021power}.} Merely expressing parity functions, for instance, implies such hardness, as per Theorem \ref{thm:sq-parity}.

\begin{theorem}[Theorem 5, \citet{kearns1998efficient}]
\label{thm:sq-parity}
Let $\mathcal{F}_n$ denote the class of all parity functions over $n$ Boolean variables (where each concept is the parity of
some unknown subset of the Boolean variables $x_1, \ldots x_n$), and let $\mathcal{F} = \bigcup_{n \geq 1} \mathcal{F}_n$. Then $\mathcal{F}$ is not
efficiently learnable from statistical queries.
\end{theorem}
 
\emph{A fortiori}, the class $\mathcal{F}$ is not learnable by gradient descent in polynomially many steps.
Note that this hardness is inherited by any superclass which realizes parity families. Standard neural
networks realize parity at linear size \citep{shalev2017failures}, implying Corollary \ref{cor:nn-hardness}. The same logic applies to tree tensor networks, which can realize every
parity at bond dimension $2$ (as they are read-once Boolean formulae), implying Corollary \ref{cor:ttn-hardness}. 

\begin{corollary}
\label{cor:nn-hardness}
ReLU networks of width
$O(n)$ are not learnable by gradient descent in polynomially many steps.
\end{corollary}

\begin{corollary}
\label{cor:ttn-hardness}
Polynomial-size TTNs with bond dimensions $\geq 2$ are not learnable by gradient descent in polynomially many steps.
\end{corollary}

These are \textit{concept class} learnability statements, as they imply that a worst-case target always exists for gradient descent, even if potentially most targets can be learned efficiently. Note that the hardness of both neural networks and TTNs follows from the
same underlying mechanism: any model class rich enough to realize Boolean formulae automatically contains such hard targets. In particular, this is a property of the hypothesis class alone and does not depend on the choice of
gradient method.

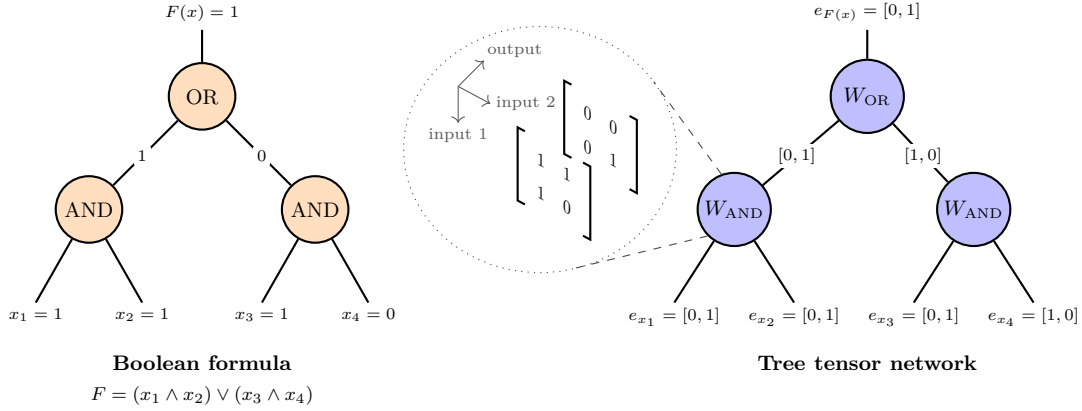
\begin{figure}[t]
\centering
\begin{tikzpicture}[
    scale=0.88, transform shape,
    gate/.style={circle, draw, fill=orange!25, minimum size=10mm,
                 inner sep=1pt, thick, font=\small},
    tensor/.style={circle, draw, fill=blue!25, minimum size=11mm,
                   inner sep=1pt, thick, font=\small},
    leg/.style={thick},
    bond/.style={thick},
    edgelab/.style={fill=white, inner sep=1.5pt, font=\scriptsize},
    inputlab/.style={font=\scriptsize, inner sep=2pt},
    panel/.style={font=\small\bfseries},
]

\node[gate] (or)  at (0,    0)    {OR};
\node[gate] (a1)  at (-1.7, -1.7) {AND};
\node[gate] (a2)  at ( 1.7, -1.7) {AND};
\draw[bond] (or) -- (a1) node[edgelab, pos=0.55] {$1$};
\draw[bond] (or) -- (a2) node[edgelab, pos=0.55] {$0$};
\draw[leg] (or) -- ++(0, 1.0) node[above, font=\scriptsize] {$F(x) = 1$};
\draw[leg] (a1) -- ++(-0.8, -1.4) node[below, inputlab] {$x_1 = 1$};
\draw[leg] (a1) -- ++( 0.8, -1.4) node[below, inputlab] {$x_2 = 1$};
\draw[leg] (a2) -- ++(-0.8, -1.4) node[below, inputlab] {$x_3 = 1$};
\draw[leg] (a2) -- ++( 0.8, -1.4) node[below, inputlab] {$x_4 = 0$};
\node[panel] at (0, -4.0) {Boolean formula};
\node[font=\footnotesize] at (0, -4.5) {$F = (x_1 \wedge x_2) \vee (x_3 \wedge x_4)$};

\node[tensor] (Wor)  at (10.0,    0)   {$W_{\mathrm{OR}}$};
\node[tensor] (Wa1)  at ( 8.0,  -1.7) {$W_{\mathrm{AND}}$};
\node[tensor] (Wa2)  at (11.6,  -1.7) {$W_{\mathrm{AND}}$};
\draw[bond] (Wor) -- (Wa1) node[edgelab, pos=0.55] {$[0,1]$};
\draw[bond] (Wor) -- (Wa2) node[edgelab, pos=0.55] {$[1,0]$};
\draw[leg] (Wor) -- ++(0, 1.0) node[above, font=\scriptsize] {$e_{F(x)} = [0,1]$};
\draw[leg] (Wa1) -- ++(-0.9, -1.4) node[below, inputlab] {$e_{x_1} = [0,1]$};
\draw[leg] (Wa1) -- ++( 0.9, -1.4) node[below, inputlab] {$e_{x_2} = [0,1]$};
\draw[leg] (Wa2) -- ++(-0.9, -1.4) node[below, inputlab] {$e_{x_3} = [0,1]$};
\draw[leg] (Wa2) -- ++( 0.9, -1.4) node[below, inputlab] {$e_{x_4} = [1,0]$};
\node[panel] at (10.0, -4.0) {Tree tensor network};

\begin{scope}[shift={(4.5, -0.5)}]
  \pgfmathsetmacro{\cellw}{0.50}
  \pgfmathsetmacro{\cellh}{0.50}
  \pgfmathsetmacro{\matw}{2*\cellw}
  \pgfmathsetmacro{\math}{2*\cellh}
  \pgfmathsetmacro{\colx}{0.78}    %
  \pgfmathsetmacro{\coly}{-0.42}   %
  \pgfmathsetmacro{\depthx}{0.7}
  \pgfmathsetmacro{\depthy}{0.7}

  \begin{scope}[cm={\colx, \coly, 0, 1, (0.4, -0.1)}]
    \node[font=\footnotesize] at (0.5*\cellw, -0.5*\cellh) {$1$};  %
    \node[font=\footnotesize] at (1.5*\cellw, -0.5*\cellh) {$1$};  %
    \node[font=\footnotesize] at (0.5*\cellw, -1.5*\cellh) {$1$};  %
    \node[font=\footnotesize] at (1.5*\cellw, -1.5*\cellh) {$0$};  %
    \draw[thick] (-0.06, 0.05) -- (-0.20, 0.05) -- (-0.20, -\math-0.05) -- (-0.06, -\math-0.05);
    \draw[thick] (\matw+0.06, 0.05) -- (\matw+0.20, 0.05) -- (\matw+0.20, -\math-0.05) -- (\matw+0.06, -\math-0.05);
  \end{scope}

  \begin{scope}[cm={\colx, \coly, 0, 1, (0.4+\depthx, -0.1+\depthy)}]
    \node[font=\footnotesize] at (0.5*\cellw, -0.5*\cellh) {$0$};
    \node[font=\footnotesize] at (1.5*\cellw, -0.5*\cellh) {$0$};
    \node[font=\footnotesize] at (0.5*\cellw, -1.5*\cellh) {$0$};
    \node[font=\footnotesize] at (1.5*\cellw, -1.5*\cellh) {$1$};
    \draw[thick] (-0.06, 0.05) -- (-0.20, 0.05) -- (-0.20, -\math-0.05) -- (-0.06, -\math-0.05);
    \draw[thick] (\matw+0.06, 0.05) -- (\matw+0.20, 0.05) -- (\matw+0.20, -\math-0.05) -- (\matw+0.06, -\math-0.05);
  \end{scope}

  \pgfmathsetmacro{\arrlen}{0.55}
  \pgfmathsetmacro{\collen}{sqrt(\colx*\colx + \coly*\coly)}
  \pgfmathsetmacro{\dlen}{sqrt(\depthx*\depthx + \depthy*\depthy)}
  \coordinate (axO) at (-0.65, 0.65);
  \draw[->, gray!70!black, thin] (axO) -- ++(0, -\arrlen)
      node[below=-1pt, font=\scriptsize, gray!70!black] {input 1};
  \draw[->, gray!70!black, thin] (axO) -- ++({\arrlen*\colx/\collen}, {\arrlen*\coly/\collen})
      node[right=-1pt, font=\scriptsize, gray!70!black] {input 2};
  \draw[->, gray!70!black, thin] (axO) -- ++({\arrlen*\depthx/\dlen}, {\arrlen*\depthy/\dlen})
      node[above right=-3pt, font=\scriptsize, gray!70!black] {output};

  \draw[dotted, gray!45!black, thin] (0.58, -0.30) ellipse (2.05 and 1.85);

  \coordinate (tanU) at (2.29,  0.72);
  \coordinate (tanL) at (1.12, -2.08);
\end{scope}

\draw[dashed, gray!55!black, thin] (Wa1.110) -- (tanU);
\draw[dashed, gray!55!black, thin] (Wa1.225) -- (tanL);

\end{tikzpicture}
\caption{Embedding a Boolean formula into a tree tensor network. \textbf{Left:} the syntax tree of $F = (x_1 \wedge x_2) \vee (x_3 \wedge x_4)$, evaluated on the assignment $(x_1, x_2, x_3, x_4) = (1,1,1,0)$. Inputs are external legs at the leaves; each internal edge is labelled with the output of the gate below it. \textbf{Right:} the corresponding TTN, obtained by replacing each gate with its $\{0,1\}$-indicator tensor and each Boolean value with its one-hot encoding ($0 \mapsto e_0 = [1,0]$, $1 \mapsto e_1 = [0,1]$). Each internal bond carries the encoded value of the corresponding subformula. \textbf{Callout:} the entries of $W_{\mathrm{AND}}$, in axonometric perspective. The two slices correspond to the two values of the output bond; the only $1$ on the back slice is at $(\text{input } 1, \text{input } 2) = (1, 1)$, the only input pair for which $\mathrm{AND}$ fires.}
\label{fig:bool-to-ttn}
\end{figure}

\section{The Loss Landscape is Benign}
\label{sec:benign}

\cref{cor:ttn-hardness} implies that there exist realizable TTN targets which cannot be learned in polynomial time by gradient descent on $\Loss$, but it does not specify a concrete mechanism by which gradient descent fails on such targets (and similarly, the absence of this mechanism on other realizable targets). It is clear that efficient learning must fail in \emph{some} way, but we do not know in \emph{what} way.
The TTN loss function is non-convex, and we know that optimization can easily become trapped in local minima for such functions \citep{jain2017nonconvex}. Historically, this failure mode has been a central focus of the neural network optimization literature \citep{auer1996exponentially, choromanska2015loss, safran2018spurious}. Thus, we may naturally suspect that optimization fails for worst-case targets by getting stuck in underperforming local minima.

We show that this explanation \emph{cannot be complete}, at least in many cases. For \emph{any} realizable TTN target, including worst-case targets, every local minimum is global (\cref{thm:no-spurious-local-min}) if one assumes a regularization condition, or at least there always exists a non-increasing-loss escape path to lower loss (\cref{cor:escape-paths}) if one does not assume the condition. Thus, even in the presence of worst-case hardness, the loss landscape can be benign.

\begin{definition}[Minimum-norm point]
\leanstatementlink{TTN/Contraction.lean}{75}\label{def:minimum-norm}
A point in parameter space $\theta$ is \emph{minimum norm} if it has minimum $L^2$ norm among all parameters producing the same output tensor. Explicitly, $\|\theta\| \leq \|\theta'\|$ for all $\theta'$ with $T(\theta') = T(\theta)$, where $\|\theta\|^2 = \sum_{v \in V} \|W_v\|_F^2$.
\end{definition}

A minimum-norm point is the most parameter-efficient representation of a given output tensor. Intuitively, it ensures that invisible ``useless'' structure cannot exist, so there is always sufficient internal capacity in the network to continue learning. We believe this condition to be a natural one. For example, \cref{app:minimum-norm} studies it in detail, showing that it coincides with the \textit{Kempf--Ness condition} of geometric invariant theory \citep{kempf1979length} for the gauge group of the network. It is also equivalent to a \emph{balancedness} condition which generalizes the well-known balancedness condition in deep linear networks \citep{arora2018optimization}, and reproduces known canonical forms of tree tensor networks \citep{acuaviva2023minimal}. We believe this condition to be a significant technical contribution of our work.
Most importantly for the present work, the minimum-norm assumption does not appear unduly restrictive for the purposes of discussing the difficulty of gradient-based optimization due to \Cref{prop:min-norm-preserved}.

\begin{restatable}{proposition}{minnormpreserved}
\label{prop:min-norm-preserved}
Let \(\theta(t)\) be a solution to the gradient flow equation \(\dot\theta(t)=-\nabla\Loss(\theta(t))\). If \(\theta(0)\) is a minimum-norm point, then \(\theta(t)\) is a minimum-norm point for
every \(t\), and so is every limit point of the trajectory.
\end{restatable}

See \cref{app:minimum-norm} for the proof\footnote{Note the setup in \cref{app:minimum-norm} allows a strictly more general loss $\Loss$ than the square loss we consider in (\ref{eq:pop_loss_main}).}. Thus, if the network is initialized to be minimum-norm (such as if it is initialized near the origin), it will remain so under gradient flow. In fact, merely \textit{locally minimizing} the norm over parameters producing the same $T(\theta)$ is equivalent to the minimum-norm condition (\cref{prop:min-norm-equivalences}), and \textit{$L^2$-regularized} gradient flow converges to minimum-norm points at an exponential rate (\cref{cor:regularization-balances}).
Under the minimum-norm assumption, in \Cref{thm:no-spurious-local-min}, we prove the absence of spurious local minima.

\begin{theoremframe}
\begin{restatable}{theorem}{nospuriouslocalmin}
\leanstatementlink{TTN/Landscape/LocalMin.lean}{181}
\label{thm:no-spurious-local-min}
Every local minimum of $\Loss$ that is a minimum-norm point is a global minimum with $\Loss = 0$.
\end{restatable}
\end{theoremframe}

See \cref{app:main} for proof. As for points without the minimum-norm assumption, in \Cref{cor:escape-paths}, we establish the absence of bad valleys at every point in parameter space.

\begin{restatable}{corollary}{escapepaths}
\label{cor:escape-paths}
From any $\theta_0$ with $T(\theta_0) \neq T^*$, there exists a continuous path $\gamma: [0,1] \to \Theta$ with $\gamma(0) = \theta_0$ such that $t \mapsto \Loss(\gamma(t))$ is non-increasing and $\Loss(\gamma(1)) < \Loss(\gamma(0))$.
\end{restatable}

We prove \Cref{cor:escape-paths} in \cref{app:escape}. Crucially, both \cref{thm:no-spurious-local-min} and \cref{cor:escape-paths} apply for \textit{all} realizable targets, despite \cref{cor:ttn-hardness}.

It is worth remarking on the assumptions and setup we use in proving \cref{thm:no-spurious-local-min}, namely minimum-norm parameters, realizability, square loss, and acyclicity of the network topology. Our core claim, that benign landscapes can coexist with worst-case computational hardness, is an existence claim, for which these assumptions are sufficient. However, to identify which parts of our argument are critical for future generalizations, \cref{subsec:tightness} demonstrates that arbitrary relaxations of these conditions permit spurious local minima. We suspect, however, that several of these constraints can be significantly weakened without compromising the underlying landscape geometry.

\section{Where Can Hardness Arise in TTNs?}
\label{sec:where-hardness}

Our results in \cref{sec:benign} suggest that local minima, and the resulting failure of convergence, is not the mechanism behind worst-case optimization difficulty in TTNs. We propose instead that even if the loss landscape is free from spurious local minima, it is not free from optimization difficulty, which can still arise from highly degenerate (very flat) \emph{saddle points} \citep{dauphin2014identifying, sankar2017saddles, zhang2026saddle}. Rather than causing outright failure of convergence, such saddle points can instead cause convergence to be \emph{exponentially slow} \citep{du2017gradient}.
We show that in TTNs, saddle points can occur only where the network is rank-deficient - that is, where it uses only a fraction of its capacity - since they are confined to the Tucker-rank-deficient subset of parameter space. Finally, we provide a case study of the parity function to concretely demonstrate this degenerate saddle phenomenon in a known hard-to-learn target.

\subsection{Rank Deficiencies in Spurious Critical Points}

Famously, there are several inequivalent notions of ``rank'' for tensors, much unlike the situation for matrices, where these notions coincide \citep{lim2021tensors}. For our purposes, the relevant notion of tensor rank is the \textit{Tucker rank} (also called \textit{multilinear rank}).

\begin{definition}[Full Tucker rank]
\leanstatementlink{TTN/Landscape/FullRank.lean}{48}
\label{def:full-tucker-rank}
A parameter point $\theta$ has \emph{full Tucker rank} if, for every internal edge $e$ and each of its endpoints $v$, the mode-$e$ unfolding $\mat_e(W_v)$ has rank equal to the bond dimension $r_e$.
\end{definition}

The following \cref{thm:no-spurious} highlights that spurious critical points can only occur under rank deficiency.

\begin{theoremframe}
\begin{restatable}{theorem}{nospurious}
\leanstatementlink{TTN/Landscape/FullRank.lean}{2020}
\label{thm:no-spurious}
Any critical point of $\Loss$ which is not a global minimum is not full Tucker rank.
\end{restatable}
\end{theoremframe}

See \cref{app:full-rank} for the proof. In particular, any saddle points must be confined to rank-deficient points. Intuitively, such rank-deficient points represent points of ``partial progress'' where the model's effective capacity is lower than its actual capacity, and could be realized by a TTN with lower bond dimensions. This generalizes the phenomenon observed in DLNs where saddle points occur when the model has learned some fraction of the singular values of the target \citep{saxe2013exact}.

\subsection{The Parity Example}
\label{subsec:parity-ttn}

Parity functions are a canonical example of a hard-to-learn target: the class of
parities is not learnable by gradient descent in polynomially many steps
(\cref{thm:sq-parity}), and parities are frequently analyzed as a difficult limiting scenario for training dynamics \citep{malach2019learning, barak2022hidden, abbe2023sgd}. Concretely, consider the parity function on all $n$ input bits, $x \mapsto x_1 \oplus \cdots \oplus x_n$. It is a read-once Boolean formula, and thus is realizable as a TTN of bond dimension $2$, by the
construction of \cref{subsec:expressivity}.\footnote{We take $n$ to be a power
of two, and $n \geq 4$; the construction may be extended to arbitrary $n$ by adjusting the
tree.} To see this, place the XOR tensor $W \in
(\R^2)^{\otimes 3}$ given by
\begin{align*}
    W_{0,0,0} = W_{0,1,1} = W_{1,0,1} = W_{1,1,0} = 1, \qquad
    W_{0,0,1} = W_{0,1,0} = W_{1,0,0} = W_{1,1,1} = 0,
\end{align*}
at every node of a balanced binary tree with $n/2$ leaves and one root,
for a total of $|V| = n - 1$ nodes. Each leaf has two external input modes and one bond mode, each
non-root internal node has two child bonds and one parent bond, and
the root has two child bonds and one external output mode. The convention
is that the third index of $W$ is the parent bond at internal nodes
and the external output mode at the root. The contraction at input
$(x_1, \ldots, x_n)$ produces the one-hot encoding of $x_1 \oplus
\cdots \oplus x_n$ at the output mode, so the full output tensor is
\[
    T^*_{x_1, \ldots, x_n, y}
    \;=\;
    \mathbf{1}[\,y = x_1 \oplus \cdots \oplus x_n\,]
    \;\in\;
    (\R^2)^{\otimes(n+1)}.
\]
\Cref{thm:no-spurious-local-min} therefore applies, asserting that no (minimum-norm)
spurious local minima obstruct learning of parity. However, this does not imply the learning problem here is without difficulty.

Consider a parameter $\theta_0$ obtained by restricting every
bond in the TTN to dimension $1$ and minimizing the squared loss
within this rank-$1$ stratum (a best rank-1 approximation). For concreteness, we take the
specific $\theta_0$ at which every node tensor $W$ has uniform entries
\[
    W_{a,b,c} \;=\; \tfrac{1}{2}
    \quad\text{for all } a, b, c \in \{0, 1\}.
\]
This represents a state of having ``halfway'' learned the parity target.
Nevertheless, even with partial learning progress, further learning
from this point may be quite difficult.

\begin{restatable}{proposition}{parityorder}
\label{prop:parity-order}
The point $\theta_0$ is a critical point of $\Loss$, and a saddle of
order at least $n/2$. In particular, there exist constants $C, \rho > 0$ such that
\[
    \Loss(\theta) - \Loss(\theta_0) \;\ge\; -C\,\|\theta-\theta_0\|^{n/2 + 1}
    \qquad\text{for all } \theta \in \Theta \text{ such that } \|\theta-\theta_0\| \le \rho.
\]
\end{restatable}

\cref{cor:escape-paths} shows that a descent path from $\theta_0$ exists, but \Cref{prop:parity-order} shows that any such path has order at least $n/2$. Consequently, the saddle point is highly ``flat'', with flatness growing in input size. The proof of \Cref{prop:parity-order} is in \cref{app:parity}.

\section{Conclusion}

In this work, we propose tree tensor networks as a toy model family that can host more nontrivial learning behavior while remaining analytically tractable. To this end, we show that despite hosting worst-case targets that cannot be learned by gradient descent in polynomial time (\cref{sec:setup}), tree tensor networks (conditionally) lack suboptimal local minima (\cref{sec:benign}). We suggest that what makes worst-case targets hard to learn may be that they have very degenerate saddle points in their loss landscape (\cref{sec:where-hardness}).

Tractable toy models have repeatedly shaped our understanding of deep learning, and the development of better and more illustrative toy models remains one of the most important open problems in the field \cite{simon2026there}. Deep linear networks alone have helped to clarify the roles of depth, initialization, and implicit regularization \citep{saxe2013exact, kawaguchi2016deep, arora2019implicit}. We believe tree tensor networks can play this role for questions existing toy models have not been able to address.

Empirical results suggest that many capabilities in deep neural networks are performed by mechanistic structures \citep{olah2020zoom, olsson2022context,nanda2023progress}, but our understanding of how and why such structures emerge over the course of training is limited.
A toy model can only address this question if it can represent such structures in the first place. Deep linear networks, which compute only linear maps, cannot, while tree tensor networks can, since they are able to represent nonlinear computation (read-once Boolean formulas) while remaining analytically tractable.

Our analysis of the parity function (\cref{subsec:parity-ttn}), for example, provides a proof-of-concept for how computational structure can influence loss landscape geometry in tree tensor networks. We hope that future work can study this interplay more systematically, and expect the mathematical tools developed here (e.g. \cref{app:minimum-norm}) to be useful in that effort. For instance, \citet{saxe2013exact} provided closed-form solutions to learning dynamics in deep linear networks, and an analogous result in tree tensor networks would allow us to characterize exactly how, why, and when computational structure emerges in such networks.

\section*{Acknowledgements}

We would like to thank Daniel Murfet for his advice and guidance on this project. We would also like to thank Rumi Salazar, Guillaume Corlouer, and Daniel Wilhelm for valuable discussions and helpful feedback on this manuscript.

Zach Furman was supported by the Melbourne Research Scholarship and Rowden White Scholarship during the completion of this research.

\iftr
\bibliographystyle{plainnat}
\else
\bibliographystyle{plainnat}
\fi
\bibliography{references}

\newpage
\appendix
\crefalias{section}{appendix}
\crefalias{subsection}{appendix}
\part*{Appendix}

\paragraph{Background and setup.}
\begin{itemize}
    \item \textbf{\cref{app:multilinear-background}} reviews the multilinear algebra used throughout, reconciling the feedforward (composition) view of TTNs from the main text with the undirected (tensor-contraction) view used in the proofs.
    \item \textbf{\cref{app:setup}} records basic setup used throughout the proofs: edge matricizations, cut factorizations, gauge freedom, and the characterization of realizability as a rank condition.
\end{itemize}

\paragraph{Landscape results.}
\begin{itemize}
    \item \textbf{\cref{app:full-rank}} proves \cref{thm:no-spurious}: critical points that are not global minima cannot have full Tucker rank.
    \item \textbf{\cref{app:main}} proves the main theorem, \cref{thm:no-spurious-local-min}: every minimum-norm local minimum is a global minimum.
    \item \textbf{\cref{app:escape}} proves \cref{cor:escape-paths}: even without the minimum-norm assumption, from every non-optimal point there is a continuous non-increasing path that strictly decreases the loss.
\end{itemize}

\paragraph{The minimum-norm condition.}
\begin{itemize}
    \item \textbf{\cref{app:minimum-norm}} studies the minimum-norm condition itself, using geometric invariant theory: minimum norm is equivalent to the \textit{Kempf-Ness condition} and a \textit{balancedness} condition (\cref{prop:min-norm-equivalences}), and it is preserved by gradient flow (\cref{prop:min-norm-preserved}).
    \item \textbf{\cref{app:counterexample}} shows the assumption is necessary: an explicit three-leaf star (Tucker) TTN with a spurious local minimum (\cref{thm:counterexample}).
\end{itemize}

\paragraph{Hardness in practice.}
\begin{itemize}
    \item \textbf{\cref{app:parity}} proves \cref{prop:parity-order}, characterizing the parity saddle point from \cref{subsec:parity-ttn}.
    \item \textbf{\cref{app:empirical-population-gap}} discusses why population-loss geometry bears on sample-based learning, and how degenerate saddles can slow stochastic gradient methods.
\end{itemize}

\paragraph{Formal verification.}
\begin{itemize}
    \item \textbf{\cref{app:formalisation}} describes the autoformalization of \cref{thm:no-spurious-local-min,thm:no-spurious}
    in \textsc{Lean}, its verification and correspondence with the
    manuscript, and the case for formalizing theoretical results
    about toy models in machine learning.
\end{itemize}

\vspace{0.2in}

\section{Background: Multilinear Algebra and Tensor Networks}
\label{app:multilinear-background}

The main text introduces tree tensor networks (TTNs) intuitively as feedforward models: directed trees where each node computes a multilinear map, and the network output is obtained by function composition. However, the geometric and algebraic proofs in \cref{app:setup} through \cref{app:parity} rely on an undirected, coordinate-based framework where networks are evaluated via \emph{tensor contraction}.

This appendix reviews the basic multilinear algebra foundations that justify transitioning between these two perspectives. Our treatment is necessarily abridged, and we direct the reader towards standard resources such as \citet{merris1997multilinear, roman2005advanced, lim2021tensors} for proofs and further discussion.

\subsection{Tensors, Arrays, and Multilinear Maps}

In the context of machine learning, we work over finite-dimensional Euclidean spaces \(\R^d\) equipped with the standard inner product and standard basis. The standard inner product gives an isomorphism
\[
    \R^d \cong (\R^d)^*,
\]
and the standard basis identifies elements of tensor product spaces with multidimensional arrays. Throughout the paper, we use these standard identifications freely.

Thus, in our setting, we move between three equivalent representations of the same data:
\begin{enumerate}
    \item \textbf{Tensors:} elements of a tensor product space \(\R^{d_1} \otimes \cdots \otimes \R^{d_k}\).
    \item \textbf{Arrays:} multidimensional grids of real numbers \(W \in \R^{d_1 \times \cdots \times d_k}\).
    \item \textbf{Multilinear maps:} functions \(\R^{d_1} \times \cdots \times  \R^{d_n} \rightarrow \R^{d_{n+1}} \otimes \cdots \otimes  \R^{d_k}\) linear in each of their $n$ arguments when all other arguments are held fixed.
\end{enumerate}

The equivalence between tensors and arrays is the usual coordinate representation. The equivalence between tensors and multilinear maps uses the Euclidean identifications \(\R^{d_i} \cong (\R^{d_i})^*\). Once these identifications are fixed, we may partition the \(k\) modes of an array into a set of input modes \(\mathcal I\) and output modes \(\mathcal O\).

\begin{proposition}
\label{prop:map-as-tensor}
Fix a partition \(\mathcal I \sqcup \mathcal O = \{1,\dots,k\}\). Using the standard Euclidean identifications of input spaces with their duals, there is a standard isomorphism between
    $\bigotimes_{i=1}^k \R^{d_i}$
and the space of multilinear maps
\[
    W:
    \prod_{i\in \mathcal I} \R^{d_i}
    \longrightarrow
    \bigotimes_{j\in \mathcal O} \R^{d_j}.
\]
\end{proposition}

This licenses us to switch between viewing \(W\) as a static array of parameters and as a computational map. The choice of directionality, meaning which modes are regarded as inputs and which are regarded as outputs, is not an intrinsic property of the tensor \(W\), and can be changed arbitrarily.

\begin{example}
An order-3 tensor $W \in \R^{d_1} \otimes \R^{d_2} \otimes \R^{d_3}$ can be interpreted in several equivalent ways:
\begin{itemize}
    \item A tensor in $\R^{d_1} \otimes \R^{d_2} \otimes \R^{d_3}$ (zero inputs, three outputs)
    \item A linear map in $\R^{d_1} \to \R^{d_2} \otimes \R^{d_3}$ (one input, two outputs).
    \item A bilinear map in $\R^{d_1} \times \R^{d_2} \to \R^{d_3}$ (two inputs, one output).
    \item A trilinear map in $\R^{d_1} \times \R^{d_2} \times \R^{d_3} \to \R$ (three inputs, zero outputs).
\end{itemize}
\end{example}

\subsection{Composition Is Tensor Contraction}

When we view tensors as multilinear maps, we compose them via ordinary function composition. When we view them as arrays, we compose them via tensor contraction.

\begin{definition}[Tensor contraction]
\label{def:tensor-contraction}
Let
\[
    A \in \R^{d_1 \times \cdots \times d_p \times r},
    \qquad
    B \in \R^{r \times e_1 \times \cdots \times e_q}.
\]
The \emph{contraction} of \(A\) and \(B\) along their shared \(r\)-dimensional
mode is the tensor
\[
    C = A \times_r B
    \in
    \R^{d_1 \times \cdots \times d_p \times e_1 \times \cdots \times e_q}
\]
with entries
\[
    C_{i_1,\ldots,i_p,j_1,\ldots,j_q}
    =
    \sum_{\alpha=1}^r
    A_{i_1,\ldots,i_p,\alpha}
    B_{\alpha,j_1,\ldots,j_q}.
\]
More generally, one may contract any chosen mode of \(A\) with any chosen mode
of \(B\) of the same dimension, after reordering the modes.
\end{definition}

If \(A\) and \(B\) are vectors, contraction is the standard inner product, while for matrices \(A\) and \(B\), contraction over one mode is standard matrix multiplication. Crucially, contraction is the algebraic mechanism that computes the composition of multilinear maps (\cref{prop:comp-is-contraction}).

\begin{proposition}
\label{prop:comp-is-contraction}
Let \(f\) and \(g\) be multilinear maps such that the output space of \(f\) matches one of the input spaces of \(g\). Let \(A_f\) and \(A_g\) be their corresponding coordinate arrays. Then the multilinear map obtained by composition,
\[
    (x,y) \mapsto g(f(x),y),
\]
corresponds exactly to the array obtained by contracting \(A_f\) and \(A_g\) along their shared mode.
\end{proposition}

For a fixed pattern of internal contractions, the final tensor is independent of the order in which those contractions are performed, up to the ordering of the remaining uncontracted modes. Thus tensor contraction effectively performs function composition, but without requiring a global choice of computational direction.

\subsection{Equivalence of TTN Definitions}

We can now bridge the directed, composition-based definition of TTNs used in the main text (\cref{def:ttn}) with the undirected, contraction-based definition used in the proofs (\cref{def:ttn-alternate}). Definitions of this form are more common in the literature, and generalize better to non-tree networks.

\begin{definition}[Tree tensor network, equivalent definition]
\leanstatementlink{TTN/Contraction.lean}{36}
\label{def:ttn-alternate}
A \emph{tree tensor network} consists of:
\begin{itemize}[leftmargin=*]
    \item a finite set $V$ of \emph{nodes} and a finite set $E$ of \emph{edges}, where each $e \in E$ has either one or two endpoints in $V$, such that the subnetwork of two-ended edges is a tree on $V$;
    \item a dimension $d_e \in \Z_{>0}$ for each edge $e \in E$; and
    \item a tensor $W_v \in \bigotimes\limits_{e \,\ni\, v} \R^{d_e}$ at each node $v \in V$, with one mode for each incident edge $e$.
\end{itemize}

\vspace{-0.2em}
Edges with two endpoints are \emph{internal} (or \emph{bonds}, with \emph{bond dimension} $r_e := d_e$) and edges with one endpoint \emph{external}. The \emph{represented tensor} $T(\theta)$ is obtained by tensor contraction along all internal edges.
\end{definition}

\textbf{Equivalence.}
To recover \cref{def:ttn} from \cref{def:ttn-alternate}, choose a root node. This choice orients each internal edge toward the root. For each non-root node \(v\), the edge connecting \(v\) to its parent is regarded as the output mode, while the remaining incident modes are regarded as input modes. At the root, all incident internal modes are regarded as input modes, and the root tensor produces a scalar.

By \cref{prop:map-as-tensor}, each node tensor can therefore be interpreted as a multilinear map. By \cref{prop:comp-is-contraction}, composing these multilinear maps from the leaves to the root computes the same represented tensor as contracting the undirected network along its internal edges.

Consequently, the root and orientation used in the main-text intuition are presentation choices. The undirected, contraction-based definition \cref{def:ttn-alternate} gives the same tensor network formalism as \cref{def:ttn}, but is often easier to manipulate algebraically. We therefore adopt this undirected perspective for the remainder of the appendix.

\begin{remark}
This algebraic flexibility is most clearly seen within graphical notation such as the \textit{Penrose graphical calculus}, or \textit{string diagrams} in the symmetric monoidal category \textbf{FdVect}, which formalize the standard visual language for tensor networks \citep{penrose1971applications, biamonte2011categorical}. Tensors are drawn as nodes, and modes as edges. Contraction is represented by connecting edges. Once the tensors and their labelled modes are fixed, the resulting contraction depends on which modes are connected, not on an arbitrary choice of data-flow direction. This is why tensor network diagrams are often drawn as undirected graphs.
\end{remark}

\section{Additional Setup}
\label{app:setup}

Here we record some facts and definitions around tree tensor networks which will be useful for the proofs: edge matricizations, cut factorizations, gauge freedom, and an equivalent characterization of realizability.

\paragraph{Notation.}

For each node $v \in V$, we write $n_v := \prod_{e \in E_v^{\mathrm{ext}}} d_e$ for the combined dimension of the external edges incident to $v$, where $E_v^{\mathrm{ext}}$ denotes the set of external edges at $v$ (with $n_v = 1$ when $v$ has no external edges). We denote the \emph{residual} of the loss $\Loss(\theta) = \frac{1}{2|\mathcal{X}|}\|T(\theta) - T^*\|_F^2$ by $R := T(\theta) - T^*$. Since the normalization constant $\frac{1}{|\mathcal{X}|}$ is fixed and positive, it has no effect on the loss landscape: rescaling $\Loss$ by a positive constant preserves critical points, local and global minima, saddle orders, and descent directions. Throughout the appendices we therefore work with the un-normalized loss $\tfrac12\|T(\theta) - T^*\|_F^2$, which (abusing notation) we continue to denote by $\Loss$.

\paragraph{Edge matricization.}

For any edge $e \in E$, removing $e$ bipartitions the nodes into sets
$V_e$ and $V_e^c$, because the network is a tree. After choosing an ordering of the external modes on
each side, this gives the \emph{edge matricization}
\[
    T^{(e)}
    \in
    \R^{m_e \times m_e^c},
    \qquad
    m_e = \prod_{v \in V_e} n_v,
    \qquad
    m_e^c = \prod_{v \in V_e^c} n_v.
\]

\paragraph{Cut factorization.} For a TTN parameter $\theta$, cutting the network at $e$ produces two subtrees, as the network is a tree. Contracting the two resulting subtrees with the bond index left open gives
matrices
    $F_e \in \R^{m_e \times r_e}$, and
    $H_e \in \R^{m_e^c \times r_e}$,
such that
    $T(\theta)^{(e)} = F_e H_e^{\top} .$

\paragraph{Gauge freedom.} The cut factorization also exposes a basic non-identifiability of the
parameterization. If \(M \in \GL(r_e)\), then
    $F_e \mapsto F_e M$, and $H_e \mapsto H_e M^{-\top}$
leaves the represented tensor unchanged, since
    $(F_e M)(H_e M^{-\top})^{\top} = F_e H_e^{\top} .$
Equivalently, this operation changes basis on the bond space associated with
\(e\) and so one may apply \(M\) to the bond-\(e\) mode on one side of the edge and
the inverse transformation on the other side. These transformations are the
\emph{gauge freedom} of the TTN parameterization.
The \emph{fiber} over an output tensor \(T_0\) is
    $\mathcal F_{T_0}
    :=
    \{\theta \in \Theta : T(\theta)=T_0\}.$
Thus gauge transformations move within a fiber, although a fiber may contain
parameter points not related by gauge transformations.

\paragraph{Realizability.} Recall that we define a tensor $T^*$ to be \textit{realizable} by a given TTN $T(\theta)$ if there exists at least one parameter $\theta^* \in \Theta$ such that $T(\theta^*) = T^*$ (\cref{sec:setup}). We now provide an alternative characterization of the realizability condition. This characterization is well-known (see e.g. \citet{grasedyck2011introduction}) but we prove it here for our specific setup.

\begin{proposition}
\label{prop:realizability-equivalence}
Let $r_e^* := \rank(T^{*(e)})$ for each internal edge \(e\). A tensor \(T^*\) is realizable by the TTN architecture
with bond dimensions \(\{r_e\}\) if and only if $r_e^* \le r_e$
for every internal edge \(e\).

Moreover, whenever these conditions hold, the realization may be chosen to use
only \(r_e^*\) bond directions at each edge. More precisely, for any collection
of \(r_e^*\)-dimensional subspaces
$S_e \subseteq \R^{r_e}$,
there exists a realization \(\theta^* \in \Theta\) of \(T^*\) such that each endpoint
tensor of \(e\) is supported on \(S_e\) in its \(e\)-mode. For this realization,
the cut factors satisfy
\[
    \rank(F_e^*)=\rank(H_e^*)=r_e^*
\]
for every internal edge \(e\).
\end{proposition}

\begin{proof}
The forward direction follows from the $T(\theta)^{(e)} = F_e H_e^{\top}$ factorization above, which
gives $\rank(T(\theta)^{(e)}) \le r_e$.
For the converse, root the tree at an arbitrary internal node. For each directed edge $e$, let $V_e$ denote the nodes in the subtree below $e$, and
define
    $U_e
    :=
    \im(T^{*(e)})
    \subseteq
    \bigotimes_{v \in V_e} \R^{n_v}.$
By assumption,
    $\dim U_e = \rank(T^{*(e)}) \le r_e $.
We claim that these spaces are nested. Precisely, let $e$ have lower endpoint $u$ with
child edges $e_1,\ldots,e_k$ and $n_u$ external dimensions. Then we claim that $U_e \subseteq \R^{n_u} \otimes U_{e_1} \otimes \cdots \otimes U_{e_k}.$ 
Each element of $U_e$ is $T^*$ contracted against a fixed vector on the modes
outside the subtree, hence a tensor over the modes of $V_e = \{u\} \cup
\bigcup_j V_{e_j}$. Its fibers along a child edge $e_j$ are themselves partial
contractions of $T^*$ across $e_j$, so they lie in $U_{e_j} = \im(T^{*(e_j)})$.
A tensor all of whose mode-$e_j$ fibers lie in $U_{e_j}$ (for every $j$) lies in
$\R^{n_u} \otimes U_{e_1} \otimes \cdots \otimes U_{e_k}$, which is the claim.

Now choose a basis of each $U_e$. The inclusion above writes each basis vector of
$U_e$ as a tensor in $\R^{n_u}$ and the child bases; take these coefficients as
the node tensor at $u$. At the root $\rho$ the same nestedness gives $T^* \in
\R^{n_\rho} \otimes U_{e_1} \otimes \cdots \otimes U_{e_k}$; expand $T^*$ in the chosen bases and
use the coefficients as the root tensor. Contracting from the leaves upward, the
open bond on each edge $e$ ranges over the chosen basis of $U_e$, so the root
expansion reconstructs $T^*$. %

This constructs a realization with bond dimension exactly \(r_e^*\) on every
internal edge. To place these bond spaces (dimension $r_e^*$) inside the prescribed ambient bond spaces (dimension $r_e$), choose for each \(e\) an isometry
\[
    \iota_e:\R^{r_e^*}\hookrightarrow\R^{r_e}
\]
with image \(S_e\), and apply \(\iota_e\) to the \(e\)-mode of both endpoint
tensors. Since each contraction across \(e\) then factors through
\(\iota_e^{\top}\iota_e=I_{r_e^*}\), the represented tensor remains \(T^*\), while
the endpoint tensors are supported on \(S_e\).

Consequently, the corresponding cut factors \(F_e^*\) and \(H_e^*\) both vanish
on \(S_e^\perp\), and hence have rank at most \(r_e^*\). On the other hand, the cut factorization
\[
    T^{*(e)}=F_e^*H_e^{*\top}
\]
has rank \(r_e^*\), so both factors must have rank at least \(r_e^*\).
Therefore
\[
    \rank(F_e^*)=\rank(H_e^*)=r_e^*.
\]
\end{proof}

We freely use this equivalent characterization of realizability going forward.

\section{No Spurious Critical Points at Full Tucker Rank}
\label{app:full-rank}

In this section we prove that full Tucker rank critical points must be global minima; or contrapositively, that critical points which are not global minima must be rank-deficient (\cref{thm:no-spurious}). The proof proceeds by induction, stripping leaf nodes until the network is reduced to a single tensor, where criticality trivially guarantees global optimality. The induction relies on a property of full-rank networks: at any leaf, the remainder of the tree acts as a full-rank linear map (\cref{lem:full-env-rank}). Informally, this non-degeneracy, together with criticality and realizability, constrains the network to match the target at each leaf. We are then able to inductively remove each leaf, transferring the loss, criticality, full-rank, and realizability conditions to a strictly smaller tree.

\begin{lemma}
\label{lem:full-env-rank}
Under the full Tucker rank assumption, for any edge $e$ and either subtree on one side of $e$, the corresponding factor matrix ($F_e$ or $H_e$) has full column rank $r_e$.
\end{lemma}

\begin{proof}
The proof follows by induction on the size of the subtree. We prove the result for $H_e$, and the result for $F_e$ follows identically. First, for the base case, if the subtree is a single leaf $v$, then $H_e = W_v \in \R^{n_v \times r_e}$, which has rank $r_e$ by full Tucker rank. Now for the inductive step, suppose node $u$ is the root of the subtree, with edge $e$ connecting $u$ to the rest of the tree and edges $e_1, \ldots, e_k$ connecting $u$ to its child subtrees (with factors $H_{e_1}, \ldots, H_{e_k}$). Then
\[
H_e = (I_{n_u} \otimes H_{e_1} \otimes \cdots \otimes H_{e_k}) \mat_e(W_u)
\]
where the identity matrix $I_{n_u}$ accounts for $u$'s own external modes. By induction, each $H_{e_j}$ has full column rank $r_{e_j}$, so the Kronecker product is full column rank. Since $\mat_e(W_u)$ has rank $r_e$ by the full Tucker rank hypothesis, $\rank(H_e) = \rank(\mat_e(W_u)) = r_e$.
\end{proof}

\nospurious*

\begin{proof}
We prove the contrapositive that if $\theta \in \Theta$ is a full-Tucker-rank critical point, it is a global minimum. We proceed by strong induction on $|V|$. The base case of $|V| = 1$ is trivial, since with a single node the model is linear, and thus critical implies global minimality. Now suppose that $|V| \geq 2$ and $\Tt$ is a tree. Then there exists a leaf $v \in L$ with unique incident edge $e$ connecting $v$ to parent $u$. By the full Tucker rank hypothesis, the factor tensor $W_v$ at node $v$ has full column rank $r_e$. Considering its thin QR decomposition $W_v = Q_v S$, with $Q_v^{\top} Q_v = I_{r_e}$ and $S \in \mathbb{R}^{r_e \times r_e}$ non-singular, we absorb $S$ into the neighboring tensor at node $u$ via $W_u \leftarrow W_u \times_e S$. This reparameterization leaves the represented tensor $T$, the residual $R$, and all Tucker ranks invariant. Hence, we assume without loss of generality that $W_v^{\top} W_v = I_{r_e}$.

Since $v$ is a leaf, the cut factorization at edge $e$ becomes $T^{(e)} = W_v H_e^{\top}$, where $H_e$ has full column rank $r_e$ by \cref{lem:full-env-rank}, so $H_e^{\top} H_e$ is invertible and positive definite. Stationarity at node $v$ implies that the residual $R^{(e)} = T^{(e)} - T^{*(e)}$ satisfies $R^{(e)}H_e = 0$. Therefore,
\[
T^{*(e)} H_e = T^{(e)} H_e = W_v (H_e^{\top} H_e),
\]
which has rank $r_e$ and image $\im(W_v)$. By realizability, $\rank(T^{*(e)}) \leq r_e$, so $\im(T^{*(e)}) = \im(W_v)$. Since both $T^{(e)}$ and $T^{*(e)}$ have columns in $\im(W_v)$, so does $R^{(e)} = T^{(e)}-T^{*(e)}$. 

Now define the \textit{reduced residual} $\widetilde{R} := W_v^{\top} R$ and observe that $R = W_v \widetilde{R}$ since $W_v^{\top} W_v = I$. Define the reduced tree $\widetilde{\Tt}$ by removing $v$ and $e$, where $u$'s former bond mode along $e$ becomes an external mode of dimension $r_e$. Observe that the residual of the loss on $\widetilde{T}$ becomes $\widetilde{R}$. We verify all hypotheses transfer:
\begin{itemize}
    \item \emph{Loss:} $\|R\|_F^2 = \|W_v \widetilde{R}\|_F^2 = \|\widetilde{R}\|_F^2$ (by $W_v^{\top} W_v = I$).
    \item \emph{Criticality:} Any perturbation $\delta \widetilde{T}$ of the reduced network $\widetilde{T}$ lifts to a perturbation $\delta T$ of the original network $T$ (where $W_v$ is left fixed), and $\delta T = W_v \delta \widetilde{T}$ by multilinearity. Hence
    \[
    \langle \widetilde{R},\, \delta\widetilde{T}\rangle
    = \langle W_v^{\top} R,\, \delta\widetilde{T}\rangle
    = \langle R,\, W_v\,\delta\widetilde{T}\rangle
    = \langle R,\, \delta T\rangle,
    \]
    and thus criticality transfers because $\langle \widetilde{R},\, \delta\widetilde{T}\rangle = \langle R,\, \delta T\rangle =0$.
    \item \emph{Full Tucker rank:} Unfolding ranks at remaining nodes are unchanged.
    \item \emph{Realizability:} $\rank(\widetilde{T}^{*(e')}) \leq \rank(T^{*(e')}) \leq r_{e'}$ for edges $e'$ in $\widetilde{\Tt}$.
\end{itemize}

By the inductive hypothesis, $\widetilde{R} = 0$, so $R = W_v \widetilde{R} = 0$.
\end{proof}

\begin{remark}
\label{rem:rank-deficiency}
The full-rank assumption on $W_v$ is essential in the stripping argument.
If $\rank(W_v)=s<r_e$, then stationarity still implies
\[
     T^{*(e)}H_e = W_v(H_e^{\top} H_e),
\]
so the columns of $T^{*(e)}H_e$ lie in $\im(W_v)$. But this only tests
$T^{*(e)}$ against the $s$-dimensional image of $H_e$. It does not force the
entire image of $T^{*(e)}$ to lie in $\im(W_v)$. Thus the target unfolding may
have components outside the model image at the leaf, and the residual need not
factor through $W_v$.
\end{remark}

\section{No Spurious Local Minima at Minimum-Norm Points}
\label{app:main}

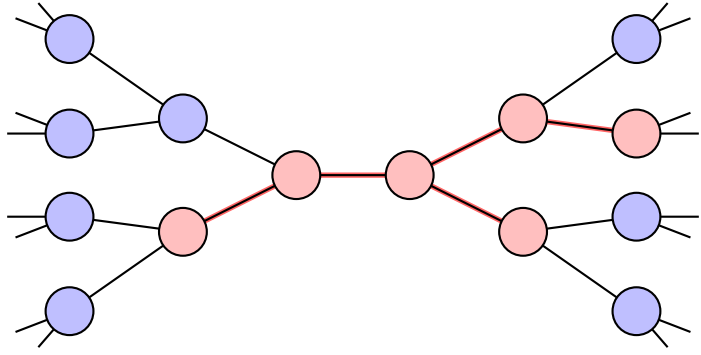
\begin{figure}[t]
    \centering
    \begin{tikzpicture}[
    x=1.5cm, y=1cm,   %
    >=latex
]

\draw[marked edge] (0,0) -- (1,0);           %
\draw[marked edge] (0,0) -- (-1,-0.75);       %
\draw[marked edge] (1,0) -- (2, 0.75);       %
\draw[marked edge] (1,0) -- (2,-0.75);       %
\draw[marked edge] (2,0.75) -- (3, 0.55);    %

\draw[edge] (0,0) -- (1,0);          %

\draw[edge] (0,0)  -- (-1, 0.75);    %
\draw[edge] (0,0)  -- (-1,-0.75);    %
\draw[edge] (-1,0.75)  -- (-2, 1.80); %
\draw[edge] (-1,0.75)  -- (-2, 0.55); %
\draw[edge] (-1,-0.75) -- (-2,-0.55); %
\draw[edge] (-1,-0.75) -- (-2,-1.80); %

\draw[edge] (1,0)  -- (2, 0.75);    %
\draw[edge] (1,0)  -- (2,-0.75);    %
\draw[edge] (2, 0.75) -- (3, 1.80); %
\draw[edge] (2, 0.75) -- (3, 0.55); %
\draw[edge] (2,-0.75) -- (3,-0.55); %
\draw[edge] (2,-0.75) -- (3,-1.80); %

\newcommand{\stubsat}[4]{%
  \draw[stub] (#1,#2) -- ++({#3}:0.55);
  \draw[stub] (#1,#2) -- ++({#4}:0.55);
}

\stubsat{-2}{ 1.80}{120}{150}   %
\stubsat{-2}{ 0.55}{150}{180}   %
\stubsat{-2}{-0.55}{180}{210}   %
\stubsat{-2}{-1.80}{210}{240}   %

\stubsat{3}{ 1.80}{ 30}{ 60}   %
\stubsat{3}{ 0.55}{  0}{ 30}   %
\stubsat{3}{-0.55}{-30}{  0}   %
\stubsat{3}{-1.80}{-60}{-30}   %

\node[marked vertex] (C0) at (0, 0) {};
\node[marked vertex] (C1) at (1, 0) {};

\node[vertex]        (L1a) at (-1,  0.75) {};
\node[marked vertex] (L1b) at (-1, -0.75) {};

\node[vertex] (L2a) at (-2,  1.80) {};
\node[vertex] (L2b) at (-2,  0.55) {};
\node[vertex] (L2c) at (-2, -0.55) {};
\node[vertex] (L2d) at (-2, -1.80) {};

\node[marked vertex] (R1a) at (2,  0.75) {};
\node[marked vertex] (R1b) at (2, -0.75) {};

\node[vertex]        (R2a) at (3,  1.80) {};
\node[marked vertex] (R2b) at (3,  0.55) {};
\node[vertex]        (R2c) at (3, -0.55) {};
\node[vertex]        (R2d) at (3, -1.80) {};

\end{tikzpicture}
    \caption{The rank-deficient core is a connected subgraph induced by rank-deficient edges (marked in red). As each vertex in this subgraph (light red) corresponds to a tensor that is perturbed in the escape path, the size of this subgraph equals the exponent of the escape direction in the saddle, so in this case $\sim \left\| \theta - \theta_0 \right\|^6$. Note that there can be multiple, mutually disconnected rank-deficient cores in the whole tensor network. In that case, each core represents one escape path in the saddle point, and the \emph{smallest} core delivers the least suppressed escape (\cref{rem:escape-degree}).}
\label{fig:rank-deficient-core}
\end{figure}

This section proves that every minimum-norm local minimum is globally optimal (\cref{thm:no-spurious-local-min}). The key idea is that any non-optimal critical point must contain a \emph{rank-deficient region}: a connected component of internal edges whose Tucker ranks are strictly smaller than their available bond dimensions. Intuitively, such a region is using only part of the capacity available to it, behaving as though those bonds had smaller dimensions than the architecture permits.

We isolate this connected component and collapse the surrounding full-rank parts of the network into fixed boundary tensors, producing a smaller \emph{rank-deficient core} (\cref{subsec:core}, \cref{fig:rank-deficient-core}). The total loss decomposes into a constant term plus the loss of this core, so it suffices to study optimization within the core itself. Rank deficiency then creates latent directions in parameter space that are invisible to first order but can be activated through a coordinated perturbation of all tensors in the core. We construct such a perturbation explicitly and show that whenever the residual is nonzero it yields a strictly decreasing descent direction (\cref{subsec:descent}). Consequently, any minimum-norm critical point with positive loss admits an arbitrarily local loss-decreasing perturbation, contradicting local optimality. The only remaining local minima are therefore global minima.

Note that while both \cref{thm:no-spurious} and \cref{thm:no-spurious-local-min} use stripping arguments to reduce the network, the reductions operate differently. In \cref{thm:no-spurious}, the full-rank assumption forces the model to match the target's column space at every leaf. This alignment allows us to iteratively remove individual leaves while preserving the \textit{exact} loss. Here, rank deficiency breaks this alignment, preventing us from collapsing the network entirely. Instead, we must rely on a weaker form of alignment (\cref{subsec:alignment-analysis}), which allows us to isolate a rank-deficient core while preserving the loss only up to an \textit{additive constant}. However, because we only need to construct a descent path to rule out a local minimum, rather than prove the loss is zero, this partial reduction is sufficient: any loss-decreasing path within the reduced core lifts to a loss-decreasing path in the full network (\cref{thm:no-spurious-local-min}).

\subsection{Dormant Subspaces}
\label{subsec:minimality-dormant}

\begin{definition}[Dormant subspace]
\label{def:dormant}
For an internal edge $e = \{u,v\}$, the \emph{dormant subspace} $D_e \subseteq \R^{r_e}$ at $e$ is
\[
D_e := \ker(\mat_e W_u) \cap \ker(\mat_e W_v),
\]
the subspace along which both endpoint tensors vanish on their $e$-mode. That is, for every $d \in D_e$, contracting $W_u$ or $W_v$ against $d$ on its $e$-mode gives the zero tensor.
\end{definition}

\begin{proposition}
\label{prop:minimality-clean}
If $\theta$ is minimum norm, then for every internal edge $e = \{u, v\}$, we have 
\[
D_e = \ker(\mat_e W_u) = \ker(\mat_e W_v) = \ker(F_e) = \ker(H_e).
\]
Further, the dimension of the dormant subspace is determined purely by the rank of the model. In particular,
\[
\dim D_e = r_e - \rank(T^{(e)}).
\]
\end{proposition}

In other words, the minimum-norm property prevents parameter structure that does not contribute to the output: any slice of a node tensor annihilated by the rest of the network must itself be zero. In particular, the four notions of ``dormancy'' at an edge coincide, and are represented by the single object $D_e$.

\begin{proof}
We prove the kernel equalities by establishing the cycle of inclusions
\[
\ker(\mat_e W_u) \subseteq \ker(F_e) \subseteq \ker(\mat_e W_v) \subseteq \ker(H_e) \subseteq \ker(\mat_e W_u).
\]
The cycle forces all four subspaces to be equal, and hence also equal to $D_e = \ker(\mat_e W_u) \cap \ker(\mat_e W_v)$.

Each cut factor $F_e, H_e$ is linear in its adjacent node tensor $W_u, W_v$. Concretely, $F_e$ is obtained by contracting the remaining tensors on the $u$-side of the cut into the non-$e$ modes of $W_u$; this operation is linear and leaves the $e$-mode untouched, so $F_e = E_u \, \mat_e W_u$ for some matrix $E_u$. Likewise $H_e = E_v \, \mat_e W_v$. This gives two of the four inclusions: $\ker(\mat_e W_u) \subseteq \ker(F_e)$ and $\ker(\mat_e W_v) \subseteq \ker(H_e)$.

The other two inclusions use the minimum-norm hypothesis. Consider $\ker(F_e) \subseteq \ker(\mat_e W_v)$, with $\ker(H_e) \subseteq \ker(\mat_e W_u)$ following via the same argument. Let $d \in \ker(F_e)$ be a unit vector, and let $\theta'$ be obtained from $\theta$ by zeroing the $d$-slice of $W_v$: that is, by replacing $\mat_e W_v$ with $(\mat_e W_v)(I - dd^{\top})$. Since $H_e = E_v \, \mat_e W_v$, this changes $H_e$ into $H_e (I - dd^{\top})$, and the output does not change:
\[
T(\theta')^{(e)} = F_e (I - dd^{\top}) H_e^{\top} = T(\theta)^{(e)} - (F_e d)(H_e d)^{\top} = T(\theta)^{(e)},
\]
where the last equality uses $F_e d = 0$. On the other hand, $\|\theta'\|^2 = \|\theta\|^2 - \|(\mat_e W_v)\, d\|^2$. Because $\theta$ is minimum norm, $\|\theta'\| \geq \|\theta\|$, so $(\mat_e W_v)\, d = 0$; that is, $d \in \ker(\mat_e W_v)$.

It remains to compute $\dim D_e$, the dimension of the dormant subspace. In the factorization $T^{(e)} = F_e H_e^{\top}$, the image of $H_e^{\top}$ is $\ker(H_e)^\perp = D_e^\perp$. On $D_e^\perp$ the matrix $F_e$ is injective, because $\ker(F_e) = D_e$. Therefore $\rank(T^{(e)}) = \dim(D_e^\perp) = r_e - \dim D_e$.
\end{proof}

Consequently, at a minimum-norm point every internal edge satisfies $\rank(\mat_e W_u) = \rank(\mat_e W_v) = \rank(T^{(e)})$, so $e$ is rank-deficient in the sense of \cref{def:full-tucker-rank} if and only if $\rank(T^{(e)}) < r_e$, if and only if $D_e \neq 0$. As every parameter considered in this appendix is minimum norm, we use these three characterizations of rank deficiency (and the corresponding characterizations of full-rank edges) interchangeably in what follows.

\subsection{Active Bond Directions are Target-Aligned}
\label{subsec:alignment-analysis}

We use the minimum-norm and criticality conditions to ensure that no structure in the model is ``unused'': informally, subtrees in the model must have non-zero alignment with their corresponding subtree in the target. Without such a property, counterexamples like \cref{app:counterexample} are possible.

Fix a root vertex and orient each non-root internal edge away from the root. Fix also an internal edge $e$ with subtree factors $F_e, H_e$ (model) and
$F_e^*, H_e^*$ (target), for an arbitrary realization of the target (\cref{app:setup}). Recall the dormant subspace
$D_e = \ker(F_e) = \ker(H_e)$ from
\cref{prop:minimality-clean}, and write
$A_e = D_e^\perp$ for the \emph{active subspace} and
$M_e := H_e^{*\top} H_e \in \R^{r_e \times r_e}$ for the
\emph{cross-Gram matrix} at~$e$.
 
The cross-Gram matrix measures how the model's bond usage at~$e$
overlaps with the target's: the $u$-direction in bond space
contributes to $\langle T, T^*\rangle$ only insofar as $M_e u \neq 0$.
The following proposition says that at a minimum-norm critical point,
every active bond direction has nonzero target alignment.
 
\begin{proposition}
\label{prop:active-aligned}
At a minimum-norm critical point, $M_e$ is injective on~$A_e$.
\end{proposition}
 
\begin{proof}
We may assume that $F_e^{\top} F_e = P_{A_e}$, the orthogonal projector onto $A_e$. To arrange this, note that the Gram matrix $F_e^{\top} F_e$ is positive semidefinite, with kernel $\ker F_e = D_e$ (\cref{prop:minimality-clean}), so it maps $A_e$ to itself and restricts to a positive-definite operator there. Let $S \in \GL(r_e)$ act as the identity on $D_e$ and as $\bigl((F_e^{\top} F_e)|_{A_e}\bigr)^{-1/2}$ on $A_e$; then $S$ is symmetric and preserves both $D_e$ and $A_e$. The gauge transformation $F_e \mapsto F_e S$, $H_e \mapsto H_e S^{-\top} = H_e S^{-1}$ at $e$ (\cref{app:setup}) then yields $(F_e S)^{\top} (F_e S) = S (F_e^{\top} F_e) S = P_{A_e}$. This gauge leaves $T$ unchanged, and, being an invertible linear change of parameters that preserves the loss, it maps critical points to critical points. It preserves the kernel equalities $\ker(F_e)=\ker(H_e)=D_e$ from \cref{prop:minimality-clean}, since $\ker(F_e S) = S^{-1} D_e = D_e$ and $\ker(H_e S^{-1}) = S D_e = D_e$. And it replaces $M_e$ by $M_e S^{-1}$, which is injective on $A_e$ if and only if $M_e$ is, because $S^{-1}$ maps $A_e$ bijectively to itself. The minimum-norm hypothesis itself need not survive the gauge, but the remainder of the proof uses only criticality and the equalities $\ker F_e = \ker H_e = D_e$.

Now suppose toward a contradiction that $M_e u = 0$ for some unit vector $u \in A_e$. Expanding the squared loss,
\begin{equation}
\label{eq:aligned-loss-expansion}
\Loss = \frac{1}{2}\langle T, T\rangle - \langle T, T^*\rangle + \frac{1}{2}\langle T^*, T^*\rangle.
\end{equation}
We construct a smooth path $\theta(t)$ with $\theta(0) = \theta$ along which $\langle T, T^*\rangle$ is constant and $\langle T, T\rangle$ strictly decreases at $t = 0$; the loss then strictly decreases at $t = 0$, which is impossible at a critical point.

Let $W_v$ be the endpoint tensor of $e$ on the $H_e$-side, so that $H_e = E_v \, \mat_e W_v$ for a fixed matrix $E_v$, as in the proof of \cref{prop:minimality-clean}. Define $\theta(t)$ by scaling down the $u$-component of $W_v$: that is, replacing $\mat_e W_v$ with $(\mat_e W_v)(I - t\, uu^{\top})$, keeping every other node tensor fixed. Write $T(t) := T(\theta(t))$ and $H_e(t) := H_e(\theta(t))$. Along this path $F_e$ is constant, while $H_e$ changes according to
\[
H_e(t) = H_e (I - t\, uu^{\top}) = H_e - t\, (H_e u)\, u^{\top}.
\]

We first show the target alignment term $\langle T(t), T^*\rangle$ of \cref{eq:aligned-loss-expansion} is constant along the path. Through the matricization of $T$ and $T^*$ along edge $e$, $\langle T, T^* \rangle = \tr\bigl(T^{(e)\top}\, T^{*(e)}\bigr) = \tr\bigl(H_e F_e^{\top} F_e^* H_e^{*\top}\bigr)$, which is linear in $H_e$. Substituting $H_e(t)$,
\[
\langle T(t), T^* \rangle = \langle T, T^* \rangle - t\, \tr\bigl((H_e u)\, u^{\top} F_e^{\top} F_e^* H_e^{*\top}\bigr) = \langle T, T^* \rangle - t\, u^{\top} F_e^{\top} F_e^* (M_e u) = \langle T, T^* \rangle,
\]
where the second equality uses cyclicity of the trace and the definition $M_e = H_e^{*\top} H_e$, and the third uses $M_e u = 0$.

Now we show the self-overlap term $\langle T(t), T(t) \rangle$ of \cref{eq:aligned-loss-expansion} strictly decreases for small $t > 0$. By the same matricization identity and the normalization $F_e^{\top} F_e = P_{A_e}$,
\[
\langle T(t), T(t) \rangle = \tr\bigl(H_e(t)\, P_{A_e}\, H_e(t)^{\top}\bigr) = \|H_e(t)\|_F^2,
\]
where the second equality holds because every row of $H_e(t)$ lies in $A_e$: the rows of $H_e$ span $(\ker H_e)^\perp = A_e$, and $u \in A_e$. Expanding the rank-one update, with $\|u\| = 1$,
\[
\|H_e(t)\|_F^2 = \|H_e\|_F^2 - 2t\, \|H_e u\|^2 + t^2\, \|H_e u\|^2.
\]
Moreover $H_e u \neq 0$, because $u$ is a nonzero vector of $A_e = (\ker H_e)^\perp$.

Since the target alignment term $\langle T(t), T^* \rangle$ of \cref{eq:aligned-loss-expansion} is constant, the loss along the path is
\[
\Loss(\theta(t)) = \Loss(\theta) - \Bigl(t - \frac{t^2}{2}\Bigr) \|H_e u\|^2,
\]
and its derivative at $t = 0$ is $-\|H_e u\|^2 < 0$. But $t \mapsto \theta(t)$ is a smooth path with $\theta(0) = \theta$, so at a critical point this derivative is $\langle \nabla \Loss(\theta), \dot{\theta}(0) \rangle = 0$ --- a contradiction.
\end{proof}

\begin{corollary}
\label{cor:rank-bound}
At a minimum-norm critical point, $\rank(T^{(e)}) \leq r^*_e$ for every
internal edge~$e$.
\end{corollary}

\begin{proof}
Neither side of the claimed inequality refers to any particular realization of the target $T^*$; by \cref{prop:realizability-equivalence}, we may choose a realization with $\rank(H_e^*) = r_e^*$ at every internal edge. By \cref{prop:minimality-clean}, $\dim(A_e) = r_e - \dim(D_e) = \rank(T^{(e)})$. By \cref{prop:active-aligned}, $M_e$ is injective on $A_e$, so $\dim(A_e) \leq \rank(M_e)$. And $M_e = H_e^{*\top} H_e$ has rank at most $\rank(H_e^*) = r_e^*$.
\end{proof}

\begin{corollary}
\label{cor:full-rank-invertible}
At a minimum-norm critical point, if $\rank(T^{(e)}) = r_e$ then $M_e$ is
invertible.
\end{corollary}

\begin{proof}
If $\rank(T^{(e)}) = r_e$, then $\dim(D_e) = r_e - \rank(T^{(e)}) = 0$ by \cref{prop:minimality-clean}, so $A_e = D_e^\perp = \R^{r_e}$. By \cref{prop:active-aligned}, $M_e$ is then injective on all of $\R^{r_e}$, and an injective square matrix is invertible.
\end{proof}

\subsection{Reduction to Full Target Rank}
\label{subsec:full-rank-reduction}

Before going further, it is convenient to reduce to the case where the target $T^*$ has full rank along every internal edge. This subsection justifies that reduction. Throughout, $\theta$ denotes the minimum-norm critical point under consideration.

\paragraph{The reduced problem.}
Note that we may assume without loss of generality that the target bond rank $r^*_e > 0$ for every internal edge $e$ (equivalently, the target $T^* \neq 0$), since if this was not the case, \cref{cor:rank-bound} would imply $\rank(T^{(e)}) \leq r^*_e = 0$, and therefore $T(\theta) = 0 = T^*$. For each internal edge $e$, choose an $r_e^*$-dimensional subspace $S_e \subseteq \R^{r_e}$ containing $A_e$; this is possible by \cref{cor:rank-bound}. Apply \cref{prop:realizability-equivalence} with these subspaces to fix a target realization $\{W_v^*\}$ supported on the $S_e$ directions.

After choosing orthonormal coordinates on each $S_e$, define a reduced TTN on the same tree with bond dimensions $r_e^*$ by restricting every internal bond space from $\R^{r_e}$ to $S_e$. Equivalently, each node tensor is replaced by its component supported on the $S_e$ directions along all incident internal edges. Let $\rho:\Theta \to \widetilde{\Theta}$ denote this orthogonal restriction map, and let $\iota:\widetilde{\Theta}\to\Theta$ denote the reverse coordinate inclusion, obtained by zero-padding the discarded $S_e^\perp$ slices. Set $\widetilde{\theta}:=\rho(\theta)$, let $\widetilde{T}^*$ be the target represented by the restricted target tensors, and write $\widetilde{\Loss}$ for the squared loss of the reduced problem.

\begin{lemma}
\label{lem:reduction}
The reduced problem satisfies:
\begin{enumerate}
    \item $\rank(\widetilde{T}^{*(e)}) = r_e^*$ for every internal edge.
    \item The zero-padding map $\iota:\widetilde{\Theta}\to\Theta$ is a linear isometry with
    \[
        \iota(\widetilde{\theta})=\theta,
        \qquad
        T\circ\iota=\widetilde{T},
        \qquad
        \Loss\circ\iota=\widetilde{\Loss}.
    \]
    \item $\widetilde{\theta}$ is a minimum-norm critical point of $\widetilde{\Loss}$.
    \item $\widetilde{T}(\widetilde{\theta}) = \widetilde{T}^*$ if and only if $T(\theta) = T^*$.
    \item If $\theta$ is a local minimum of $\Loss$, then $\widetilde{\theta}$ is a local minimum of $\widetilde{\Loss}$.
\end{enumerate}
\end{lemma}

\begin{proof}
Since $A_e \subseteq S_e$, we have $S_e^\perp \subseteq A_e^\perp = D_e$ by \cref{prop:minimality-clean}. Both endpoint tensors of $e$ therefore vanish on $S_e^\perp$. Thus $\theta$ is already supported on the chosen subspaces, so
\[
    \iota(\widetilde{\theta})=\iota(\rho(\theta))=\theta .
\]
The target realization is supported on the same subspaces by construction, so restricting it does not change the represented external tensor: $\widetilde{T}^*=T^*$.
Consequently, the edge matricizations of the reduced target have the same ranks as those of $T^*$:
\[
    \rank(\widetilde{T}^{*(e)})=\rank(T^{*(e)})=r_e^* .
\]
This proves (1).

For any reduced parameter $\widetilde{\phi}\in\widetilde{\Theta}$, zero-padding commutes with TTN contraction, so $T(\iota(\widetilde{\phi}))=\widetilde{T}(\widetilde{\phi})$. Since $\widetilde{T}^*=T^*$, this gives $\Loss(\iota(\widetilde{\phi}))=\widetilde{\Loss}(\widetilde{\phi})$. The map $\iota$ is a linear isometry because it only inserts zero components in the orthogonal complements $S_e^\perp$. Hence (2) follows.

We now derive the remaining assertions from (2). If $\widetilde{\theta}'$ lies in the reduced fiber over $\widetilde{T}(\widetilde{\theta})$, then $\iota(\widetilde{\theta}')$ lies in the original fiber over $T(\theta)$. Therefore, by minimum-norm of $\theta$ and the isometry of $\iota$,
\[
    \|\widetilde{\theta}\|=\|\theta\|
    \leq \|\iota(\widetilde{\theta}')\|
    =\|\widetilde{\theta}'\|.
\]
Thus $\widetilde{\theta}$ is minimum norm in its reduced fiber. Criticality also transfers: for every reduced tangent direction $\widetilde{\delta}$,
\[
    \left.\frac{d}{dt}\right|_{t=0}
    \widetilde{\Loss}(\widetilde{\theta}+t\widetilde{\delta})
    =
    \left.\frac{d}{dt}\right|_{t=0}
    \Loss(\theta+t\,\iota(\widetilde{\delta})),
\]
and the right-hand side vanishes because $\theta$ is critical. This proves (3).

Since $\iota(\widetilde{\theta})=\theta$ and $\widetilde{T}^*=T^*$, (2) gives
\[
    \widetilde{T}(\widetilde{\theta})-\widetilde{T}^*
    =
    T(\theta)-T^*,
\]
which proves (4). Finally, if $\theta$ is a local minimum of $\Loss$, then the continuity of $\iota$ and the identity $\Loss\circ\iota=\widetilde{\Loss}$ imply that no sufficiently small reduced perturbation can decrease $\widetilde{\Loss}$. Hence $\widetilde{\theta}$ is a local minimum of $\widetilde{\Loss}$, proving (5).
\end{proof}

Thus, for the remainder of \cref{app:main}, we may work in the reduced problem and assume full target rank without loss of generality. Any local descent path found after this reduction zero-pads to a local descent path for the original problem.

\subsection{The Rank-Deficient Core}
\label{subsec:core}

Here we reduce the network to a smaller \emph{rank-deficient core} (\cref{lem:core-properties} and \cref{fig:rank-deficient-core}). This is a second reduction, distinct from the full-target-rank reduction above (\cref{subsec:full-rank-reduction}). This reduction freezes all tensors outside one connected component of rank-deficient edges and absorbs them into the target tensor, in such a way as to preserve the loss up to an additive constant.

\begin{lemma}
\label{lem:core-properties}
Assume $T(\theta)\neq T^*$. Then there is a connected component $C$ of rank-deficient internal edges and a reduced TTN, with parameter $\widetilde{\theta}$, target $\widetilde{T}^*$, and ordinary squared loss $\widetilde{\Loss}$, such that:

\begin{enumerate}
    \item the internal edges of the reduced TTN are exactly the edges of $C$;
    \item every internal edge of the reduced model $\widetilde{T}(\widetilde{\theta})$ has rank strictly less than $r_e$;
    \item every internal edge of the reduced target $\widetilde{T}^*$ has rank $r_e$;
    \item $\widetilde{T}(\widetilde{\theta}) \neq \widetilde{T}^*$;
    \item if $\theta$ is a local minimum of $\Loss$, then $\widetilde{\theta}$ is a local minimum of $\widetilde{\Loss}$;
    \item every strict local descent path for $\widetilde{\Loss}$ at $\widetilde{\theta}$ lifts to a strict local descent path for $\Loss$ at $\theta$.
\end{enumerate}
\end{lemma}

\begin{proof}
If no internal edge is rank-deficient, then \cref{thm:no-spurious} implies $T(\theta)=T^*$, contrary to assumption. Hence at least one internal edge is rank-deficient. Choose a connected component $C$ of the subgraph induced by the rank-deficient internal edges, and root the tree at a node of $C$. Let $V_C$ be the set of vertices incident to edges of $C$. Every internal edge leaving $V_C$ is full-rank, since $C$ is a connected component.

We first describe a single boundary stripping step. Let $e$ be a full-rank boundary edge, and let $H_e$ be the factor produced by the frozen subtree outside $V_C$. Let $F_e$ be the matricized contraction of the retained side across $e$, so that the represented tensor has edge matricization $F_e H_e^{\top}$. Since $e$ is full-rank, \cref{prop:minimality-clean} gives $D_e=0$, so $H_e$ has full column rank. Set $G:=H_e^{\top} H_e$. The change of variables
\[
    H_e \mapsto H_eG^{-1/2},
    \qquad
    F_e \mapsto F_eG^{1/2}
\]
leaves $F_e H_e^{\top}$ unchanged. Thus it is only an invertible reparametrization of the retained boundary factor, and we may compute in coordinates where $H_e^{\top} H_e=I_{r_e}$.
The reduced parameter $\widetilde{\theta}$ below is defined in these coordinates; when we later lift reduced perturbations back to the original network, we undo this change of variables. In these coordinates the cross-Gram $M_e=H_e^{*\top}H_e$ differs from the original cross-Gram by right multiplication by $G^{-1/2}$, so it is still invertible.

Now apply the same change of variables to every maximal subtree attached to $V_C$. Let $B$ be the set of boundary edges leaving $V_C$. After these changes, each frozen factor satisfies $H_e^{\top} H_e=I_{r_e}$ for $e\in B$.

For any choice of the retained core tensors, let $X$ be the tensor obtained by contracting only the vertices in $V_C$, leaving open the original external modes incident to $V_C$ and one additional mode of dimension $r_e$ for each $e\in B$. Reattaching the frozen subtrees is the linear map
\[
    \Gamma X := \text{the tensor obtained from $X$ by contracting its $e$-mode with $H_e$ for every $e\in B$}.
\]
Since the columns of each $H_e$ are orthonormal, $\Gamma$ is an isometry. Therefore, for every such open tensor $X$,
\[
    \frac12\|\Gamma X-T^*\|_F^2
    =
    \frac12\|X-\Gamma^*T^*\|_F^2
    +
    \frac12\|(I-\Gamma\Gamma^*)T^*\|_F^2 .
\]
Observe the second term is a constant independent of $X$. We define the reduced TTN to be the network on the vertices $V_C$, with the boundary edges in $B$ treated as external modes, and we set
\[
    \widetilde{T}^* := \Gamma^*T^* .
\]
Concretely, this target reduction acts as follows on each stripped boundary edge. For a boundary edge $e\in B$, the adjoint $\Gamma^*$ contracts the corresponding target factor with $H_e$. Thus, using the target cut factorization $T^{*(e)}=F_e^*H_e^{*\top}$, the target factor on the new boundary mode is
\[
    T^{*(e)}H_e = F_e^*H_e^{*\top}H_e = F_e^*M_e .
\]
The reduced parameter $\widetilde{\theta}$ is the retained collection of core tensors in the chosen coordinates, and the reduced loss is
\[
    \widetilde{\Loss}(\widetilde{\phi})
    :=
    \frac12\|\widetilde{T}(\widetilde{\phi})-\widetilde{T}^*\|_F^2 .
\]
The internal edges of this reduced TTN are exactly the edges of $C$, proving (1). Since the construction only changes coordinates on boundary modes, every internal edge of $\widetilde{T}(\widetilde{\theta})$ remains rank-deficient, proving (2).

It remains to check the target ranks on the core edges. For a stripped boundary edge, $M_e$ is invertible by \cref{cor:full-rank-invertible}. Thus replacing $F_e^*$ by $F_e^*M_e$ only multiplies the corresponding boundary mode of the core target by an invertible matrix. In any matricization across an internal core edge, this appears as multiplication on the left or right by an invertible matrix, and therefore does not change the rank. Hence, for each $e\in C$,
\[
    \rank(\widetilde{T}^{*(e)})=\rank(T^{*(e)})=r_e.
\]
This proves (3). For any $e\in C$, the model has $\rank(\widetilde{T}(\widetilde{\theta})^{(e)})<r_e$ while the target has $\rank(\widetilde{T}^{*(e)})=r_e$, so $\widetilde{T}(\widetilde{\theta})\neq \widetilde{T}^*$, proving (4).

It remains to record the local variational relation between the two losses. The stripping construction defines a local embedding
\[
    \jmath:\widetilde{\Theta}\to\Theta
\]
near $\widetilde{\theta}$: given reduced core tensors, undo the boundary coordinate changes and reattach all frozen tensors outside $V_C$. This map sends $\widetilde{\theta}$ to $\theta$. By the isometry identity above,
\[
    \Loss(\jmath(\widetilde{\phi}))=\widetilde{\Loss}(\widetilde{\phi})+C
\]
for all $\widetilde{\phi}$ sufficiently close to $\widetilde{\theta}$, where $C$ is independent of $\widetilde{\phi}$. The map $\jmath$ is continuous, and the coordinate changes on the retained boundary modes are invertible. Therefore local minimality of $\theta$ implies local minimality of $\widetilde{\theta}$, proving (5). The same identity shows that any strict local descent path for $\widetilde{\Loss}$ lifts through $\jmath$ to a strict local descent path for $\Loss$, proving (6).
\end{proof}

\subsection{Descent on the Rank-Deficient Core}
\label{subsec:descent}
We now construct a descent path on the rank-deficient core, using the reduced squared loss supplied by \cref{lem:core-properties}.

For a core vertex $v$, let $\widetilde n_v$ be defined analogous to $n_v$ (\cref{app:setup}), as the product of the dimensions of the external modes incident to $v$ in the reduced TTN, with $\widetilde n_v=1$ if there are none. Thus a rank-one tensor on the reduced external modes can be written as $\bigotimes_v u_v$, with $u_v\in\R^{\widetilde n_v}$.

By \cref{lem:core-properties}, the reduced residual $\widetilde{R} := \widetilde{T}(\widetilde{\theta})-\widetilde{T}^*$ is nonzero. Hence there is a rank-one tensor
$U = \bigotimes_{v} u_v$, with $u_v \in \R^{\widetilde n_v}$, such that $\langle \widetilde{R}, U \rangle < 0$.

For each internal core edge $e$, choose a unit vector $\eta_e \in D_e$. For each core node $v$ with incident core edges $e_1, \ldots, e_k$, perturb
\[
    \delta W_v := \eta_{e_1} \otimes \cdots \otimes \eta_{e_k} \otimes u_v.
\]
Let $\widetilde{\theta}(t)$ be the reduced parameter obtained by replacing each core tensor $W_v$ by $W_v+t\,\delta W_v$.

\begin{lemma}
\label{lem:mixed-vanishing}
Let $N$ be the number of core nodes. Then
\[
    \widetilde{T}(\widetilde{\theta}(t))
    =
    \widetilde{T}(\widetilde{\theta}) + t^N U .
\]
\end{lemma}

\begin{proof}
For \(S\subseteq V_C\), define \(\widetilde{\theta}^S\) by
\[
    W_v(\widetilde{\theta}^S)
    =
    \begin{cases}
        \delta W_v, & v\in S,\\
        W_v(\widetilde{\theta}), & v\notin S.
    \end{cases}
\]
By multilinearity of the reduced TTN contraction,
\begin{equation}
\label{eq:mixed-vanishing-expansion}
    \widetilde{T}(\widetilde{\theta}(t))
    =
    \sum_{S\subseteq V_C} t^{|S|}\widetilde{T}(\widetilde{\theta}^S).
\end{equation}
We claim that \(\widetilde{T}(\widetilde{\theta}^S)=0\) for every proper
nonempty subset \(S\subsetneq V_C\). Indeed, connectedness of the core gives
an internal core edge \(e=\{x,y\}\) with \(x\in S\) and \(y\notin S\). The
tensor \(\delta W_x\) has factor \(\eta_e\) on its \(e\)-mode, while the tensor
at \(y\) is unperturbed. Since \(D_e=\ker(\mat_e W_y)\) by
\cref{prop:minimality-clean}, 
contracting \(W_y\) against \(\eta_e\) along its
\(e\)-mode yields zero. Hence the entire term
\(\widetilde{T}(\widetilde{\theta}^S)\) is zero.

Of the terms in \cref{eq:mixed-vanishing-expansion}, only \(S=\varnothing\) and \(S=V_C\) remain. The first term is
\(\widetilde{T}(\widetilde{\theta})\). In the second, every internal core edge
contracts two copies of \(\eta_e\), giving
\(\langle\eta_e,\eta_e\rangle=1\), and the remaining external factors give
\(\bigotimes_v u_v=U\). Since this term uses one perturbation at each of the
\(N\) core vertices, its coefficient is \(t^N\).
\end{proof}

\begin{lemma}
\label{lem:strict-descent}
For all sufficiently small $t > 0$,
\[
    \widetilde{\Loss}(\widetilde{\theta}(t))
    <
    \widetilde{\Loss}(\widetilde{\theta}).
\]
\end{lemma}

\begin{proof}
By \cref{lem:mixed-vanishing} and the definition of the loss function,
\[
    \widetilde{\Loss}(\widetilde{\theta}(t))
    -
    \widetilde{\Loss}(\widetilde{\theta})
    =
    t^N \langle \widetilde{R}, U \rangle
    +
    \frac{1}{2} t^{2N} \|U\|_F^2 .
\]
Since $\langle \widetilde{R}, U \rangle < 0$, this is negative for all sufficiently small $t > 0$.
\end{proof}

The path $\widetilde{\theta}(t)$ is continuous and satisfies $\widetilde{\theta}(0)=\widetilde{\theta}$, so this descent is arbitrarily local in parameter space.

We now prove the main theorem by merely lifting this descent path on the core to the entire network.

\subsection{Main Theorem}
\label{subsec:main-theorem}

\nospuriouslocalmin*

\begin{proof}
Suppose for contradiction that $\theta$ is a minimum-norm local minimum with $T(\theta) \neq T^*$. By \cref{lem:core-properties}, there is a reduced core parameter $\widetilde{\theta}$ which is a local minimum of the reduced loss $\widetilde{\Loss}$. But \cref{lem:strict-descent} gives a path $\widetilde{\theta}(t)$ with $\widetilde{\Loss}(\widetilde{\theta}(t))<\widetilde{\Loss}(\widetilde{\theta})$ for all sufficiently small $t>0$, a contradiction.
\end{proof}

\begin{remark}
\label{rem:escape-degree}
The descent direction has order $N$, the number of nodes in the rank-deficient core: $\widetilde{\Loss}(\widetilde{\theta}(t)) - \widetilde{\Loss}(\widetilde{\theta}) \sim t^N \langle \widetilde{R}, U \rangle$. Intuitively, this means that the more the network needs to learn ``in one step,'' the higher the order of the descent direction and the harder it is to escape the saddle point. \citet{abbe2023sgd} see similar phenomenon in neural networks and hypothesize that this is a general mechanism for learning difficulty.
\end{remark}

\subsection{Tightness of Assumptions}
\label{subsec:tightness}

This subsection details the necessity of the core assumptions underlying \cref{thm:no-spurious-local-min}. Note that the core claim of this paper is an \textit{existence} claim (that a model class with worst-case hardness can have a benign loss landscape) for which these assumptions are sufficient, but future generalizations may wish to weaken them.

For each condition, we outline explicit counterexamples or known results demonstrating that removing it permits the existence of spurious local minima. However, we believe that generalized results may still be possible by merely weakening these assumptions or the desired conclusions.

\paragraph{Minimum-norm is necessary.}
The minimum-norm condition in \cref{thm:no-spurious-local-min} is essential.  \cref{app:counterexample} constructs an explicit spurious local minimum for a realizable $3$-leaf Tucker decomposition at a parameter that is not minimum norm.

\paragraph{Realizability is necessary.}
 Without realizability, \cref{thm:no-spurious-local-min} fails even in the simplest nontrivial TTN topology. Consider the star graph with all bond dimensions $r_e = 1$: the TTN output is $T(\theta) = \lambda u_1 \otimes \cdots \otimes u_m$, and fitting a non-rank-1 target $T^*$ reduces to the best rank-1 tensor approximation problem. For $m \geq 3$ this problem is known to have spurious local minima for generic targets, in contrast to the matrix case ($m=2$), where \citet{baldi1989neural} rules them out. An explicit $3\times 3\times 3 \times 3$ supersymmetric target with a nonglobal local minimum is given in \citet[Example 4]{kofidis2002best}.

\paragraph{The square loss is necessary.}

Consider the case of a two-layer deep linear network $W_1 \in \R^{m \times n}, W_2 \in \R^{n \times m}$. If the hidden width $m$ is less than the input/output width $n$, then \citet{trager2019pure} shows one can have spurious critical points if one uses losses other than the square loss.\footnote{This does not appear to be prevented by our minimum-norm assumption.} As tree tensor networks generalize deep linear networks, this shows that the square loss assumption is necessary. 

However, note that this does not preclude the possibility of results like that of \citet{trager2019pure}, which show that the mere parameterization cannot create spurious critical points in deep linear networks if there are no such critical points in function space. We conjecture that a similar result holds in tree tensor networks under the minimum-norm assumption, but leave this to future work.

\paragraph{The tree topology is necessary.} 

\citet{chen2020tensor} show that tensor networks with a ring topology can have spurious local minima. They exhibit, for any $d\geq 3$, a target of bond dimension $r+1$ and external dimension $n \geq r^2+1$ together with a non-strict spurious local minimum in the ring class of bond dimension $r^{d-1}$, so the example persists under substantial over-parameterization. They use cycles to construct a mechanism similar to our counter-example in \cref{app:counterexample}, where local variations to match the target create cross-terms that dominate the loss. We suspect that this example does not violate the minimum-norm assumption, and thus it seems plausible that acyclicity is an independently necessary assumption.

\section{Escape Paths}
\label{app:escape}

As a corollary of the analysis in \cref{app:main}, we establish absence of bad valleys without the minimum-norm assumption.

\escapepaths*

\begin{proof}
The fiber $\mathcal{F} := \{\theta' : T(\theta') = T(\theta_0)\}$ is the preimage of a point under the polynomial map $T$, hence a real algebraic variety. Let $C$ be the connected component of $\mathcal{F}$ containing $\theta_0$. Since $\mathcal{F}$ is closed and the norm is proper, the norm attains its minimum on $C$ at some $\theta_{\min} \in C$.

Semialgebraic sets have finitely many connected components, each semialgebraic \citep{bochnak2013real}. Hence $C$ is open as a subset of $\mathcal{F}$, and so $\theta_{\min}$ is a local minimum of the norm over the full fiber $\mathcal{F}$. By \cref{prop:min-norm-equivalences}, $\theta_{\min}$ locally minimizing the norm over $\mathcal{F}$ implies the minimum-norm condition. Since connected semialgebraic sets are path-connected \citep{bochnak2013real}, there exists a continuous path from $\theta_0$ to $\theta_{\min}$ within $C$. Along this path, $T$ is constant, hence $\Loss$ is constant.

Since $T(\theta_{\min}) = T(\theta_0) \neq T^*$, the proof of \cref{thm:no-spurious-local-min} gives a continuous path from $\theta_{\min}$ along which $\Loss$ strictly decreases. Concatenating this path with the prior constant-loss path from $\theta_0$ to $\theta_{\min}$ gives the desired non-increasing path from $\theta_0$ to strictly lower loss.
\end{proof}

\begin{remark}
This is weaker than \cref{thm:no-spurious-local-min}: the constant-loss segment may leave any neighborhood of $\theta_0$, so the corollary does not rule out spurious local minima in unrestricted parameter space.
\end{remark}

\section{The Minimum-Norm Condition}
\label{app:minimum-norm}

Our core result, \cref{thm:no-spurious-local-min}, applies only for points in parameter space satisfying the minimum-norm condition
(\cref{def:minimum-norm}). We prove two main results to support the use of this assumption; we believe it is a significant contribution rather than a limitation. First, minimum norm is
equivalent to a \emph{balancedness} condition, an easily-checkable algebraic identity (\cref{def:balanced} and \cref{prop:min-norm-equivalences}). This identity
generalizes the well-known balancedness condition of the deep linear network literature, where
it is conserved by gradient flow
\citep{arora2018optimization,du2018algorithmic} and assumed in standard
convergence analyses \citep{arora2018convergence,bah2022learning}. Second, gradient flow on any loss which depends only on the tensor $T(\theta)$ represented by the TTN (including the square loss we use) conserves balancedness, and \(\ell_2\)
regularization drives the balancedness defect to zero at an exponential rate
(\cref{prop:balancedness-conserved,cor:regularization-balances}).

\textit{The minimum-norm hypothesis is therefore equivalent to (a generalization of) a commonly-assumed and easily-checkable balancedness assumption, preserved by training dynamics, and produced by regularization}. We believe this is a significant contribution of our work, and one that will provide a basis for future work on training dynamics of tree tensor networks.

Significantly, our proofs in this appendix make substantial use of (real) geometric invariant theory (GIT) \citep{mumford1994geometric, richardson1990minimum}. The gauge freedom of a TTN is an action of the reductive group
\(G_{\Tt}=\prod_e\GL(r_e)\) on parameter space; minimum-norm parameters are equivalent to parameters which satisfy the \emph{Kempf-Ness condition} \citep{kempf1979length,richardson1990minimum}; balancedness is the vanishing
of its moment map; and the conservation law associated to balancedness is a Noether law for the gauge
symmetry. The core result is that every
fiber of the contraction map contains a unique closed gauge orbit
(\cref{prop:unique-closed-gauge-orbit-in-fiber}). This can be viewed as a canonical-form
statement: each representable tensor has a distinguished parameterization,
unique up to the orthogonal gauge group, and the contraction map separates
closed gauge orbits exactly as a GIT quotient map does. These conclusions require the tree topology
and fail on networks with cycles (\cref{rem:cycles}). The same machinery
appears in \citet{lindsey2025regularization}, who apply the Kempf--Ness
theorem to deep linear networks, and in \citet{acuaviva2023minimal}, who
define canonical forms of complex tensor networks by norm minimization over
the gauge group.

\subsection{The Fiber has a Unique Closed Gauge Orbit}
 
This subsection proves that every fiber of the contraction map contains a unique
closed gauge orbit (\cref{prop:unique-closed-gauge-orbit-in-fiber}). The
argument has three steps. We first treat the two-node network, where
contraction is just matrix multiplication, in \cref{lem:minimal-factorizations,cor:factorization-orbit-closures}. Second, we show that
contracting an internal edge of the tree commutes with taking orbit closures
(\cref{lem:edge-contraction-orbit-closures}). Third, we collapse the tree one edge at a time
to show that any two parameters with the same represented tensor have
intersecting orbit closures (\cref{prop:equal-tensors-meeting-orbit-closures}),
which gives uniqueness, since distinct orbits are disjoint (\cref{prop:unique-closed-gauge-orbit-in-fiber}). Existence of a closed gauge orbit is general, but \textit{uniqueness} requires the tree topology, as it
fails on networks with cycles (\cref{rem:cycles}).
 
\subsubsection{The Two-Node Case: Minimal Factorizations}
 
Throughout this subsection, fix dimensions \(m,n,r\) and let \(\GL(r)\) act on
pairs \((A,B)\in\R^{m\times r}\times\R^{r\times n}\) by
\begin{equation}
\label{eq:minimal-factorization-gauge}
    g\cdot(A,B)=(Ag^{-1},gB),
\end{equation}
leaving the product \(AB\) invariant.
 
\begin{definition}[Minimal factorization]
\label{def:minimal-factorization}
A pair \((A,B)\) is a \emph{minimal factorization} of \(C:=AB\) if
\[
    \rank A=\rank B=\rank C.
\]
\end{definition}

\begin{lemma}
\label{lem:minimal-factorizations}
Let \(A\in\R^{m\times r}\), \(B\in\R^{r\times n}\), and \(C=AB\).
\begin{enumerate}
    \item[(i)] The pair \((A,B)\) is a minimal factorization of \(C\) if and only if
    \(\R^r=\im B\oplus\ker A\).
    \item[(ii)] The orbit closure \(\overline{\GL(r)\cdot(A,B)}\) contains a
    minimal factorization of \(C\).
    \item[(iii)] The minimal factorizations of \(C\) form a single
    \(\GL(r)\)-orbit.
\end{enumerate}
\end{lemma}
 
\begin{proof}
\emph{(i)} Since \(A(\im B)=\im(AB)=\im C\), the restriction
\(A|_{\im B}:\im B\to\im C\) is surjective. Hence we have the equivalences
\[
    \rank B=\rank C
    \quad\Longleftrightarrow\quad
    A|_{\im B}\text{ is injective}
    \quad\Longleftrightarrow\quad
    \im B\cap\ker A=0 .
\]
If this holds, then \(\dim\im B=\rank C\). Under this condition, we have the further equivalences
\[
    \rank A=\rank C
    \quad\Longleftrightarrow\quad
    \dim\ker A=r-\dim\im B
    \quad\Longleftrightarrow\quad
    \dim(\im B+\ker A)=r .
\]
Combining the two equivalences, we have that \(\rank A = \rank B = \rank C\) is equivalent to \(\im B\cap\ker A=0\) and \(\dim(\im B+\ker A)=r\). Notice that the two conditions \(\im B\cap\ker A=0\) and \(\dim(\im B+\ker A)=r\) are together equivalent to the internal direct
sum decomposition \(\R^r=\im B\oplus\ker A\).
 
\emph{(ii)} Let \(K=\im B\cap\ker A\), and choose a complement \(L\) of \(K\) in
\(\im B\) and a complement \(M\) of \(\im B\) in \(\R^r\), so that
\(\R^r=L\oplus K\oplus M\). Let \(\pi_L,\pi_K,\pi_M\) be the associated
projections. Note \(\pi_MB=0\) (as \(\im B=L\oplus K\)) and \(A\pi_K=0\) (as
\(K\subseteq\ker A\)).

We will take the limit of a particular $\GL(r)$-action to eliminate $K$ and $M$, leading to a minimal factorization. Precisely, for \(t>0\) define \(g_t=\id_L\oplus\,t\,\id_K\oplus\,t^{-1}\id_M\in\GL(r)\), which acts on $(A, B)$ according to \cref{eq:minimal-factorization-gauge}. In the limit as $t$ goes to zero, this action yields
\[
    \lim_{t\rightarrow 0} \bigl(Ag_t^{-1},\,g_tB\bigr)
    = \lim_{t\rightarrow 0} \bigl(A\pi_L+tA\pi_M,\;\pi_LB+t\pi_KB\bigr)
    =
    \bigl(A\pi_L,\,\pi_LB\bigr).
\]
Label this limit by \(\bigl(A_0, B_0\bigr) :=\bigl(A\pi_L,\,\pi_LB\bigr)\). By construction, the pair \(\bigl(A_0, B_0\bigr)\) is a member of the orbit closure \(\overline{\GL(r)\cdot(A,B)}\). Moreover, it is a factorization of \(C\), since \((Ag_t^{-1})(g_tB) = AB = C\) for all finite $t$, and matrix multiplication is continuous.

It is a minimal factorization by~(i): \(\im B_0=\pi_L(L\oplus K)=L\), and
\(\ker A_0=K\oplus M\), since \(A\) is injective on \(L\) (indeed
\(L\cap\ker A\subseteq L\cap K=0\)); hence \(\R^r=\im B_0\oplus\ker A_0\).
 
\emph{(iii)} Let \((A_0,B_0)\) and \((A_1,B_1)\) be minimal factorizations of
\(C\). To show they lie on a single \(\GL(r)\)-orbit, we will construct an explicit \(g \in \GL(r)\) such that \(g\cdot(A_0,B_0)=(A_1,B_1)\).
By part~(i), we have
\(\R^r=\im B_i\oplus\ker A_i\) for \(i=0,1\). Observe that the restriction
\[
    A_i|_{\im B_i}:\im B_i\to\im C
\]
is an isomorphism: it is surjective because \(A_i(\im B_i)=\im C\), and
injective because \(\im B_i\cap\ker A_i=0\). Define
\[
    h:=\bigl(A_1|_{\im B_1}\bigr)^{-1}\circ\bigl(A_0|_{\im B_0}\bigr)
    :\im B_0\to\im B_1
\]
so that \(A_1h=A_0\) on \(\im B_0\). Since \(\ker A_0\) and \(\ker A_1\) have
the same dimension, extend \(h\) by any isomorphism
\(\ker A_0\to\ker A_1\). This gives an invertible map \(g\in\GL(r)\), using the
two decompositions \(\R^r=\im B_i\oplus\ker A_i\) from~(i).

We conclude by showing that \(g\cdot(A_0,B_0)=(A_1,B_1)\), which requires that \(A_0g^{-1}=A_1\) and \(gB_0 = B_1\). First, we show \(gB_0 = B_1\). Observe that
\[
    A_1(gB_0)=A_0B_0=C=A_1B_1.
\]
Both \(gB_0\) and \(B_1\) have columns in \(\im B_1\), where \(A_1\) is
injective, so \(gB_0=B_1\). To show that \(A_0g^{-1}=A_1\), again use the decomposition \(\R^r=\im B_1\oplus\ker A_1\) from~(i). On \(\im B_1\), we have \(g^{-1}=h^{-1}\) and \(A_0h^{-1}=A_1\), so obviously \(A_0g^{-1}=A_1\); on \(\ker A_1\), \(g^{-1}\) maps into \(\ker A_0\), so both
sides vanish. Therefore \((A_1,B_1)=g\cdot(A_0,B_0)\).
\end{proof}

\begin{corollary}
\label{cor:factorization-orbit-closures}
Let \(A\in\R^{m\times r}\), \(B\in\R^{r\times n}\), and \(C=AB\). Then \(\overline{\GL(r)\cdot(A,B)}\) contains \emph{every}
minimal factorization of \(C\), and in particular one with
\begin{equation}
\label{eq:factorization-orbit-closures}
    \|A_0\|_F^2=\|B_0\|_F^2=\|C\|_*,
\end{equation}
the nuclear norm (sum of singular values) of \(C\). Consequently, any two
factorizations of the same matrix have intersecting orbit closures.
\end{corollary}
 
\begin{proof}
The orbit closure \(\overline{\GL(r)\cdot(A,B)}\) is \(\GL(r)\)-invariant:
for any \(g\in\GL(r)\),
\[
    g\cdot\overline{\GL(r)\cdot(A,B)}
    =
    \overline{g\cdot(\GL(r)\cdot(A,B))}
    =
    \overline{\GL(r)\cdot(A,B)}.
\]
Thus, if the orbit closure contains one point of a \(\GL(r)\)-orbit, it contains
that entire orbit. By \cref{lem:minimal-factorizations}(ii), the orbit closure \(\overline{\GL(r)\cdot(A,B)}\) contains some
minimal factorization of \(C\). \cref{lem:minimal-factorizations}(iii) tells us that the \(\GL(r)\)-orbit of any minimal factorization contains every minimal factorization of \(C\). Together these three facts imply that the orbit closure \(\overline{\GL(r)\cdot(A,B)}\) contains every minimal
factorization of \(C\).
 
To show a representative satisfying \cref{eq:factorization-orbit-closures} exists, let
\(C=\sum_{j=1}^{k}\sigma_j\,u_jw_j^{\top}\) be a singular value decomposition
with \(k=\rank C\) and \(\sigma_j>0\). Since \(k\le\rank B\le r\), we may
choose orthonormal \(e_1,\dots,e_k\in\R^r\) and set
\[
    A_0=\sum_{j}\sqrt{\sigma_j}\;u_je_j^{\top},
    \qquad
    B_0=\sum_{j}\sqrt{\sigma_j}\;e_jw_j^{\top}.
\]
Then \(A_0B_0=C\), \(\im B_0=\Span(e_1,\dots,e_k)\), and
\(\ker A_0=\Span(e_1,\dots,e_k)^{\perp}\), so \((A_0,B_0)\) is minimal by
\cref{lem:minimal-factorizations}(i), with
\(\|A_0\|_F^2=\|B_0\|_F^2=\sum_j\sigma_j=\|C\|_*\).
 
Finally, the orbit closures of any two factorizations of \(C\) both contain the
(nonempty) set of minimal factorizations of \(C\), so they intersect.
\end{proof}
 
\subsubsection{The General Case}
 
Fix an internal edge \(e=\{u,v\}\). As \(\Tt\) is a tree, every edge connects two distinct nodes \(u \neq v\). Contracting \(e\) merges \(u\) and \(v\)
into a single node and yields a network \(\Tt/e\) on the same external
legs with one fewer internal edge; since \(\Tt\) is a tree, so is
\(\Tt/e\). Let \(c_e:\Theta_{\Tt}\to\Theta_{\Tt/e}\)
be the partial contraction that replaces \(W_u,W_v\) by their contraction
across \(e\) and leaves all other node tensors unchanged. Matricize the two
adjacent tensors along \(e\), with the bond as column index on the \(u\)-side
and as row index on the \(v\)-side:
\[
    A=\mat_e(W_u)\in\R^{m\times r_e},
    \qquad
    B=\mat_e(W_v)\in\R^{r_e\times n},
\]
where \(m\) and \(n\) are the products of the dimensions of the remaining modes
of \(u\) and of \(v\). The merged node tensor is then the matrix product
\(AB\), and the gauge factor \(\GL(r_e)\) acts on \((A,B)\) exactly as in the
previous subsection, leaving all other node tensors fixed. The gauge group
of the smaller network is
\(G_{\Tt/e}=\prod_{e'\in E_{\mathrm{int}}(\mathcal
T)\setminus\{e\}}\GL(r_{e'})\); the map \(c_e\) is
\(G_{\Tt/e}\)-equivariant, is invariant under the \(\GL(r_e)\)-factor,
and satisfies
\(T_{\Tt}=T_{\Tt/e}\circ c_e\) by associativity of contraction.
 
\begin{lemma}
\label{lem:edge-contraction-orbit-closures}
For every internal edge \(e=\{u,v\}\) and every \(\theta\in\Theta_{\Tt}\),
\[
    c_e\!\left(\overline{G_{\Tt}\cdot\theta}\right)
    =
    \overline{G_{\Tt/e}\cdot c_e(\theta)}.
\]
\end{lemma}
 
\begin{proof}
(\(\subseteq\)) Since \(c_e\) is invariant under the \(\GL(r_e)\)-factor and
equivariant for the remaining gauge group, we have
\(c_e(G_{\Tt}\cdot\theta)=G_{\Tt/e}\cdot c_e(\theta)\) pointwise;
continuity of \(c_e\) then gives the inclusion of orbit closures.
 
(\(\supseteq\)) An element \(\eta \in \overline{G_{\Tt/e}\cdot c_e(\theta)}\) is given by \(\eta=\lim_i h_i\cdot c_e(\theta)\) for some sequence \(h_i\in G_{\Tt/e}\). We aim to prove \(\eta \in c_e(\overline{G_{\Tt}\cdot\theta})\); that is, that there exists \(\tilde{\eta} \in \overline{G_{\Tt}\cdot\theta}\) such that \(c_e(\tilde{\eta}) = \eta\). Lift each \(h_i\) to
\(\widetilde h_i\in G_{\Tt}\) by letting it act trivially on the bond at
\(e\), and set \(\theta_i=\widetilde h_i\cdot\theta\in G_{\mathcal
T}\cdot\theta\). Then \(c_e(\theta_i)= h_i\cdot c_e(\theta) \to \eta\). If \(\theta_i\) was bounded, one could pass to a subsequence of \(\theta_i\) that converges to a point \(\tilde{\eta} \in \overline{G_{\Tt}\cdot\theta}\), with \(c_e(\tilde{\eta})=\eta\), and conclude the proof.

However, even though \(c_e(\theta_i)\) converges, the sequence \(\theta_i\) itself need not be bounded: gauges on the other bonds of
\(u\) and \(v\) can blow up the two tensors adjacent to \(e\) while their
contraction remains bounded. We correct this by replacing the pair of tensors at \(e\)
with a norm-controlled minimal factorization of the same merged tensor.
 
Precisely, let \((A_i,B_i)=(\mat_e W_u(\theta_i),\,\mat_e W_v(\theta_i))\) and let
\(C_i=A_iB_i\) be the matricized merged tensor of \(c_e(\theta_i)\). Applying
\cref{cor:factorization-orbit-closures} to the pair \((A_i,B_i)\), while
leaving all other node tensors fixed, gives a point \(\widetilde\theta_i\in
\overline{\GL(r_e)\cdot\theta_i}\) obtained from \(\theta_i\) by replacing
\((A_i,B_i)\) with a minimal factorization of \(C_i\) whose two factors each
have squared Frobenius norm equal to \(\|C_i\|_*\). Then
\(c_e(\widetilde\theta_i)=c_e(\theta_i)\), since the merged tensor and all other
node tensors are unchanged.

Moreover the sequence \(\widetilde\theta_i\) is bounded. The node tensors away
from \(u\) and \(v\) are unchanged when passing from \(\theta_i\) to
\(\widetilde\theta_i\), and they appear as node tensors in the convergent
sequence \(c_e(\theta_i)\), so they are bounded. The two tensors at \(u\) and
\(v\) have squared Frobenius norm equal to \(\|C_i\|_*\), which is bounded since
\(C_i\) converges.
Finally,
\[
    \widetilde\theta_i
    \in\overline{\GL(r_e)\cdot\theta_i}
    \subseteq
    \overline{G_{\Tt}\cdot\theta_i}
    =
    \overline{G_{\Tt}\cdot\theta},
\]
where the inclusion uses \(\GL(r_e)\subseteq G_{\Tt}\), and the equality
uses \(\theta_i\in G_{\Tt}\cdot\theta\).
 
Passing to a subsequence, \(\widetilde\theta_i\to\widetilde\eta\). Since
\(\overline{G_{\Tt}\cdot\theta}\) is closed,
\(\widetilde\eta\in\overline{G_{\Tt}\cdot\theta}\), and by continuity
\(c_e(\widetilde\eta)=\lim_i c_e(\widetilde\theta_i)=\lim_i
c_e(\theta_i)=\eta\). Hence
\(\eta\in c_e\!\left(\overline{G_{\Tt}\cdot\theta}\right)\).
\end{proof}
 
\begin{proposition}
\label{prop:equal-tensors-meeting-orbit-closures}
If \(\theta,\theta'\in\Theta_{\Tt}\) satisfy \(T(\theta)=T(\theta')\),
then
\[
    \overline{G_{\Tt}\cdot\theta}
    \cap
    \overline{G_{\Tt}\cdot\theta'}
    \neq\varnothing.
\]
\end{proposition}
 
\begin{proof}
We induct on the number of internal edges. With no internal edges the network
is a single node, \(T\) is the identity, and \(G_{\Tt}\) is trivial, so
\(\theta=\theta'\).
 
Otherwise choose an internal edge \(e\). Since \(\Tt\) is a tree, \(\Tt/e\) is again a tree and has one fewer internal edge, so it lies in the induction class. Moreover, since
\[
    T_{\Tt/e}(c_e(\theta))
    =T(\theta)=T(\theta')
    =T_{\Tt/e}(c_e(\theta')),
\]
the induction hypothesis applied to \(\Tt/e\) yields a point
\[
    \xi\in
    \overline{G_{\Tt/e}\cdot c_e(\theta)}
    \cap
    \overline{G_{\Tt/e}\cdot c_e(\theta')}.
\]

We now lift this common point back to the original tree. By
\cref{lem:edge-contraction-orbit-closures}, the membership of \(\xi\) in the
first orbit closure gives
\[
    \widehat\theta\in\overline{G_{\Tt}\cdot\theta}
    \qquad\text{with}\qquad
    c_e(\widehat\theta)=\xi.
\]
The membership of \(\xi\) in the second orbit closure similarly gives
\[
    \widehat\theta'\in\overline{G_{\Tt}\cdot\theta'}
    \qquad\text{with}\qquad
    c_e(\widehat\theta')=\xi.
\]
 
Thus \(c_e(\widehat\theta)=c_e(\widehat\theta')\). Equivalently, all node
tensors of \(\widehat\theta\) and \(\widehat\theta'\) away from \(u\) and \(v\)
agree, and their pairs at \(e\) are two factorizations of the same merged
tensor. Applying \cref{cor:factorization-orbit-closures} to these two
factorizations, while leaving the common node tensors away from \(u\) and \(v\)
fixed, gives a point \(\zeta\) lying in both \(\GL(r_e)\)-orbit closures.

It remains to check that \(\zeta\) lies in the two original orbit closures. We show \(\zeta\in\overline{G_{\Tt}\cdot\theta}\); the same argument gives \(\zeta\in\overline{G_{\Tt}\cdot\theta'}\).
To see this, observe the nested inclusions
\[
    \zeta
    \in
    \overline{\GL(r_e)\cdot\widehat\theta}
    \subseteq
    \overline{G_{\Tt}\cdot\widehat\theta}
    \subseteq
    \overline{G_{\Tt}\cdot\theta}.
\]
The first inclusion uses \(\GL(r_e)\subseteq G_{\Tt}\). For the second inclusion,
\(\widehat\theta\in\overline{G_{\Tt}\cdot\theta}\), and this orbit
closure is \(G_{\Tt}\)-invariant. Hence
\(G_{\Tt}\cdot\widehat\theta\subseteq
\overline{G_{\Tt}\cdot\theta}\); since the latter set is closed, it
contains \(\overline{G_{\Tt}\cdot\widehat\theta}\).
\end{proof}

\begin{proposition}
\label{prop:unique-closed-gauge-orbit-in-fiber}
Every fiber of the contraction map
\(T:\Theta_{\Tt}\to T(\Theta_{\Tt})\) contains a minimizer of the norm.
The gauge orbit of any such minimizer is the \emph{unique} closed
\(G_{\Tt}\)-orbit in that fiber.
\end{proposition}

\begin{proof}
\emph{Existence.} Fix a fiber
\[
    \mathcal F_\theta=\{\eta:T(\eta)=T(\theta)\}.
\]
This set is closed by continuity of \(T\) and nonempty by definition, so the
norm attains a minimum on it. (Indeed, choosing any \(\theta_1\in\mathcal
F_\theta\), it suffices to minimize over the compact set
\(\mathcal F_\theta\cap\{\eta:\|\eta\|\le \|\theta_1\|\}\).)

Let \(\theta_{\min}\) be any norm minimizer in \(\mathcal F_\theta\).
Gauge transformations preserve the represented tensor, hence
\(G_{\Tt}\cdot\theta_{\min}\subseteq\mathcal F_\theta\). Therefore
\[
    \|\theta_{\min}\|\le \|g\cdot\theta_{\min}\|
    \qquad\text{for all }g\in G_{\Tt}.
\]
Thus \(\theta_{\min}\) minimizes the norm on its gauge orbit, i.e. it is a
\emph{minimum vector} of the gauge action in the sense of
\citet{richardson1990minimum}.

We apply Richardson's closed-orbit theorem to this action. Its hypotheses hold:
\(G_{\Tt}\) is a product of real general linear groups acting rationally
and linearly on \(\Theta_{\Tt}\), and the Frobenius inner product is
compatible with the action. Indeed, the maximal compact subgroup
\(\prod_e O(r_e)\) acts by isometries, and the symmetric matrices in the Lie
algebra act by symmetric operators (a direct trace computation). By Theorem~4.4
of \citet{richardson1990minimum}, the orbit of a minimum vector is closed.
Hence \(G_{\Tt}\cdot\theta_{\min}\) is a closed orbit in the fiber.
 
\emph{Uniqueness.} If \(G_{\Tt}\cdot\eta\) and
\(G_{\Tt}\cdot\eta'\) are closed orbits in the same fiber, then
\(T(\eta)=T(\eta')\), so by
\cref{prop:equal-tensors-meeting-orbit-closures} their closures intersect.
Since the two orbits are closed, these closures are the orbits themselves.
Distinct orbits are disjoint, so
\(G_{\Tt}\cdot\eta=G_{\Tt}\cdot\eta'\).
\end{proof}

\begin{remark}
\label{rem:cycles}
The tree topology is required for the uniqueness half of \cref{prop:unique-closed-gauge-orbit-in-fiber}.
Consider the simplest cyclic network: a single node with one bond joining it to
itself, so that \(\theta\in\R^{r\times r}\), the gauge group \(\GL(r)\) acts by
conjugation, and the contraction is \(T(\theta)=\tr\theta\). Every diagonal
matrix \(D\) is a critical point of the squared norm on its orbit --- for
symmetric \(X\),
\(\langle XD-DX,D\rangle_F=\tr\bigl(X(DD^{\top}-D^{\top}D)\bigr)=0\) --- hence
a minimum vector by Theorem~4.3 of \citet{richardson1990minimum}, and its
conjugation orbit is closed by Theorem~4.4. Conjugation preserves eigenvalues,
so diagonal matrices with distinct spectra lie in distinct orbits; for
\(r\ge2\), a single trace fiber therefore contains a continuum of closed gauge
orbits. (Compare the ring-topology counterexamples of \citet{chen2020tensor}
discussed in the main text.)
\end{remark}

\begin{remark}
\label{rem:minimal-canonical-form}
Over the complex numbers, \citet{acuaviva2023minimal} apply the same
construction --- norm minimization over the gauge group, via the Kempf--Ness
theorem \citep{kempf1979length} --- to tensor networks of arbitrary
geometry, obtaining a \emph{minimal canonical form} that identifies two tensor networks exactly when their gauge orbit closures intersect. In the case of tree tensor networks, this minimal canonical form can be seen as closely related to the well-known \textit{Vidal gauge}
\citep{vidal2003efficient,shi2006classical}. We are not able to immediately apply the algebraic machinery of complex GIT in the real case (see \citet{acosta2020git} for details), so it was necessary to prove this result ourselves.
\end{remark}

\subsection{Equivalent Conditions to Minimum-Norm}

This subsection characterizes the minimum-norm parameterizations
(\cref{def:minimum-norm}). As stated, the condition is global, and difficult to check. \cref{prop:min-norm-equivalences} shows that
minimum norm is equivalent to \emph{balancedness}, an algebraic condition which can be
verified edge by edge.
 
\begin{definition}[Balanced parameterization]
\label{def:balanced}
A parameterization $\theta$ is \emph{balanced} if at every internal edge $e = \{u, v\}$,
\[
(\mat_e W_u)^{\top}\, \mat_e W_u \;=\; (\mat_e W_v)^{\top}\, \mat_e W_v.
\]
\end{definition}
 
Both sides are \(r_e\times r_e\) Gram matrices. Intuitively, balancedness says the two
endpoint tensors make use of the bond space identically.

This condition is a generalization of the well-known balancedness condition from the study
of deep linear networks. In this case, the adjacent-layer differences
\(W_{j+1}^{\top}W_{j+1}-W_j^{\vphantom{\top}}W_j^{\top}\) are conserved by
gradient flow \citep{arora2018optimization,du2018algorithmic}, and the balanced
manifold is a standard setting for convergence analyses
\citep{arora2018convergence,bah2022learning}.
\citet{lindsey2025regularization} interpret this condition through
geometric invariant theory, showing via the Kempf--Ness theorem that the
\(\ell_2\) regularizer of a deep linear network is minimized exactly on the
balanced manifold; \cref{prop:min-norm-equivalences} is the corresponding
statement for tree tensor networks.

\begin{proposition}
\label{prop:min-norm-equivalences}
For \(\theta\in\Theta_{\Tt}\) with fiber
\(\mathcal F_\theta=\{\theta'\in\Theta_{\Tt}:T(\theta')=T(\theta)\}\),
the following are equivalent:
\begin{enumerate}
    \item[(i)] \(\theta\) minimizes \(\|\cdot\|\) on \(\mathcal F_\theta\),
    i.e.\ \(\theta\) is minimum norm;
    \item[(ii)] \(\theta\) locally minimizes \(\|\cdot\|\) on
    \(\mathcal F_\theta\);
    \item[(iii)] \(\theta\) minimizes \(\|\cdot\|\) on its gauge orbit
    \(G_{\Tt}\cdot\theta\), i.e.\ \(\theta\) is a minimum vector;
    \item[(iv)] \(\theta\) is balanced.
\end{enumerate}
Moreover, the minimum-norm points of a fiber form a single orbit of the
orthogonal gauge group
\(K_{\Tt}=\prod_{e\in E_{\mathrm{int}}(\Tt)}O(r_e)\).
\end{proposition}
 
\begin{proof}
We prove \((i)\Rightarrow(ii)\Rightarrow(iii)\Rightarrow(i)\), then \((iii)\Leftrightarrow(iv)\), and finally identify the
minimum-norm points within a fiber.

\emph{(i) \(\Rightarrow\) (ii)} is immediate.
 
\emph{(ii) \(\Rightarrow\) (iii).} For \(g\) near the identity, gauge
invariance gives \(T(g\cdot\theta)=T(\theta)\), so \(g\cdot\theta\in
\mathcal F_\theta\). Thus \(g\cdot\theta\) is a point of \(\mathcal F_\theta\)
near \(\theta\). Since \(\theta\) locally minimizes the norm on
\(\mathcal F_\theta\), we have
\(\|g\cdot\theta\|\ge\|\theta\|\) for all such \(g\). Hence \(\rho_\theta:g\mapsto\|g\cdot\theta\|^2\) has
a local minimum, and therefore a critical point, at the identity. By
\citet[Theorem~4.3]{richardson1990minimum}, \(\theta\) is a minimum vector.
 
\emph{(iii) \(\Rightarrow\) (i).} Since \(\theta\) is a minimum vector,
\citet[Theorem~4.4]{richardson1990minimum} implies that
\(G_{\Tt}\cdot\theta\) is closed. By
\cref{prop:unique-closed-gauge-orbit-in-fiber}, there exists at least one norm
minimizer \(\theta_{\min}\) on \(\mathcal F_\theta\), and its orbit is the unique closed
gauge orbit in \(\mathcal F_\theta\). Hence
\[
    G_{\Tt}\cdot\theta
    =
    G_{\Tt}\cdot\theta_{\min}.
\]
Thus \(\theta_{\min}\in G_{\Tt}\cdot\theta\). Because \(\theta\)
minimizes the norm on its gauge orbit,
\[
    \|\theta\|\le \|\theta_{\min}\|
    =
    \min_{\eta\in\mathcal F_\theta}\|\eta\|.
\]
The reverse inequality is automatic because \(\theta\in\mathcal F_\theta\), so
\(\theta\) is minimum norm.

\emph{(iii) \(\Leftrightarrow\) (iv).} By
\citet[Theorem~4.3]{richardson1990minimum}, \(\theta\) is a minimum vector if
and only if \(\rho_\theta:g\mapsto\|g\cdot\theta\|^2\) is critical at the identity.
We now compute this criticality condition. The Lie algebra of \(G_{\Tt}\) is the direct sum of one copy of
\(\mathfrak{gl}(r_e)\) for each internal edge \(e\). Every \(X\in\mathfrak{gl}(r_e)\) decomposes as a sum of a skew-symmetric
matrix and a symmetric matrix. The skew-symmetric directions generate the
orthogonal gauge group \(K_{\Tt}\), which preserves the Frobenius norm.
Thus the derivative of \(\rho_\theta\) vanishes automatically in those
directions. Since directional derivatives at the identity form a linear
functional on the Lie algebra, and every Lie-algebra direction is a sum of a
skew-symmetric and a symmetric direction, \(\rho_\theta\) is critical at the
identity if and only if the derivative vanishes along every symmetric direction
at every internal edge.

Fix \(e=\{u,v\}\) and a symmetric matrix \(X\). The one-parameter gauge
\(\exp(tX)\) at \(e\) acts by
\(\mat_eW_u\mapsto\mat_eW_u\exp(-tX)\) and
\(\mat_eW_v\mapsto\mat_eW_v\exp(tX)\), leaving all other node tensors fixed,
so only these two terms in the total parameter norm vary. Thus
\[
\begin{aligned}
    \frac{d}{dt}\Big|_{t=0}
    &\Bigl(\bigl\|\mat_eW_u\exp(-tX)\bigr\|_F^2
    +\bigl\|\mat_eW_v\exp(tX)\bigr\|_F^2\Bigr)\\
    &\qquad=2\tr\Bigl[X\Bigl((\mat_eW_v)^{\top}\mat_eW_v
    -(\mat_eW_u)^{\top}\mat_eW_u\Bigr)\Bigr].
\end{aligned}
\]
The bracketed matrix is symmetric, so this vanishes for every symmetric \(X\)
if and only if that matrix is zero, i.e.\ if and only if \(\theta\) is balanced
at \(e\). Since this holds independently at each internal edge, criticality of
\(\rho_\theta\) at the identity is equivalent to balancedness.

\emph{The \(K_{\Tt}\)-orbit.} Let \(\theta,\theta'\) be minimum-norm
points of the same fiber. By (i)~\(\Rightarrow\)~(iii), both are minimum
vectors. Hence both gauge orbits are closed, and by
\cref{prop:unique-closed-gauge-orbit-in-fiber} they are the same orbit. In
particular, \(\theta'\in G_{\Tt}\cdot\theta\).

We now restrict from \(G_{\Tt}\) to its maximal compact subgroup. By
\citet[Theorem~4.3]{richardson1990minimum}, the minimum vectors in the orbit of
a minimum vector form a single \(K_{\Tt}\)-orbit. Since \(\theta\) and
\(\theta'\) are both minimum vectors in the same \(G_{\Tt}\)-orbit, this
gives \(\theta'\in K_{\Tt}\cdot\theta\). Conversely, every point of
\(K_{\Tt}\cdot\theta\) is minimum norm, since \(K_{\Tt}\) preserves
both the norm and the fiber.
\end{proof}
 
\begin{remark}
\label{rem:tree-enters}
The tree topology of the network is necessary for \cref{prop:min-norm-equivalences}. However, only the implication (iii) \(\Rightarrow\) (i) uses the tree structure,
via \cref{prop:unique-closed-gauge-orbit-in-fiber}. The chain
(i) \(\Rightarrow\) (ii) \(\Rightarrow\) (iii) and the equivalence
(iii) \(\Leftrightarrow\) (iv) hold for any tensor network whose
internal edges join distinct nodes. In particular, minimum norm implies
balancedness on any such network. The tree topology is needed for the converse direction:
in the self-loop network of \cref{rem:cycles} with \(r\ge2\),
\(\theta=\operatorname{diag}(c,0,\dots,0)\) is a minimum vector for $c \not= 0$, yet its trace
fiber contains \((c/r)\,I_r\), of strictly smaller norm.
\end{remark}

\subsection{Minimum-Norm is Preserved by Gradient Flow}
\label{app:gradient-flow}
 
We show the minimum-norm condition is stable over training: it is preserved by gradient flow on
the square loss, and more generally any loss that depends on the parameters only through the represented tensor (\cref{prop:balancedness-conserved}). More specifically, \cref{prop:balancedness-conserved} shows that gauge invariance of the loss leads to a Noether-type conservation law for the failure of balancedness, generalizing the classical conserved quantities of deep
linear networks \citep{arora2018optimization,du2018algorithmic}. Further, in the presence of $\ell_2$ regularization, the failure of balancedness is driven to zero at an exponential rate (\cref{cor:regularization-balances}).
 
Let \(\Loss=\ell\circ T\), where \(\ell\) is any continuously differentiable
function of the represented tensor --- for instance the squared loss
\(\ell=\frac{1}{2|\mathcal{X}|}\|\cdot-T^*\|_F^2\) against a fixed target \(T^*\) --- and
consider the gradient flow \(\dot\theta=-\nabla\Loss(\theta)\) with
respect to the Frobenius inner product on \(\Theta_{\Tt}\).
 
\begin{proposition}
\label{prop:balancedness-conserved}
Along every solution of the gradient flow
\(\dot\theta=-\nabla\Loss(\theta)\), the \emph{balancedness defect}
\[
    \Delta_e(\theta)
    :=(\mat_eW_u)^{\top}\mat_eW_u-(\mat_eW_v)^{\top}\mat_eW_v
    \;\in\;\R^{r_e\times r_e}
\]
is constant in time, at every internal edge \(e=\{u,v\}\).
\end{proposition}
 
\begin{proof}
Fix an internal edge \(e=\{u,v\}\). Write \(A:=\mat_e W_u, B:=\mat_e W_v,
\) with the \(e\)-mode as the column index in both matricizations. Then
\(\Delta_e=A^{\top}A-B^{\top}B\). Matricization is a linear isometry, so gradient flow in these variables is obviously
\(\dot A=-\nabla_A\Loss, \ \dot B=-\nabla_B\Loss\),
where the gradients are taken with all other parameters fixed. With this convention, contraction across \(e\) pairs the
two column indices, so the local contraction has the form \(AB^{\top}\). Hence
the gauge change
\(A\mapsto AS, \,B\mapsto BS^{-\top}\)
preserves the represented tensor for every \(S\in\GL(r_e)\), and therefore
preserves \(\Loss\):
\begin{equation}
\label{eq:loss-gauge-invariance}
    \Loss(AS,BS^{-\top}, \ldots)=\Loss(A,B,\ldots).
\end{equation}
Differentiate \cref{eq:loss-gauge-invariance} along the one-parameter family
\(S=I+tX\), where \(X\in\R^{r_e\times r_e}\) is arbitrary. Since
\[
    \frac{d}{dt}\Big|_{t=0}A(I+tX)=AX,
    \qquad
    \frac{d}{dt}\Big|_{t=0}B(I+tX)^{-\top}=-BX^{\top},
\]
we obtain
\[
    \langle \nabla_A\Loss,AX\rangle_F
    -
    \langle \nabla_B\Loss,BX^{\top}\rangle_F
    =0.
\]
Using Frobenius trace identities on both terms, we rewrite this as
\[
    \langle A^{\top}\nabla_A\Loss,X\rangle_F
    -
    \langle (\nabla_B\Loss)^{\top}B,X\rangle_F
    =0.
\]
Since \(X\) is arbitrary, this forces
\begin{equation}
\label{eq:balanced-equality1}
    A^{\top}\nabla_A\Loss=(\nabla_B\Loss)^{\top}B,
\end{equation}
and taking transposes gives the equivalent identity
\begin{equation}
\label{eq:balanced-equality2}
    (\nabla_A\Loss)^{\top}A=B^{\top}\nabla_B\Loss.
\end{equation}
By the product rule,
\[
    \frac{d}{dt}\Delta_e
    =
    \frac{d}{dt}(A^{\top}A-B^{\top}B)
    =
    -(\nabla_A\Loss)^{\top}A
    -A^{\top}\nabla_A\Loss
    +(\nabla_B\Loss)^{\top}B
    +B^{\top}\nabla_B\Loss.
\]
The first and fourth terms cancel by \cref{eq:balanced-equality2}, and the second and third terms cancel
by \cref{eq:balanced-equality1}. Hence \(\dot\Delta_e=0\).
\end{proof}

\minnormpreserved*
 
\begin{proof}
By \cref{prop:min-norm-equivalences}, minimum norm is equivalent to balancedness, so \(\theta(0)\) is balanced, i.e.\
\(\Delta_e(\theta(0))=0\) at every internal edge. By
\cref{prop:balancedness-conserved}, \(\Delta_e(\theta(t))=0\) for all \(t\),
and by continuity of \(\Delta_e\) the same holds at every limit point of \(\theta(t)\). By
\cref{prop:min-norm-equivalences} again, all these balanced points are minimum norm.
\end{proof}
 
Thus, if the network is initialized to a minimum-norm point in parameter space, it will remain minimum norm under gradient flow. This is an instance of a broader phenomenon: conservation laws
arising from parameterization symmetries have been studied systematically in deep learning \citep{kunin2020neural,marcotte2023abide}.

In fact, even if we do not initialize at a minimum-norm point, we show that we are driven exponentially toward one by the use of regularization (\cref{cor:regularization-balances}). Note this does not contradict \cref{prop:balancedness-conserved}, since regularization depends on the actual parameter values $\theta$ rather than just the contracted tensor $T(\theta)$.
 
\begin{corollary}
\label{cor:regularization-balances}
Fix \(\lambda>0\) and consider gradient flow of the regularized loss
\(\Loss+\tfrac{\lambda}{2}\|\cdot\|^2\), that is,
\(\dot\theta=-\nabla\Loss(\theta)-\lambda\theta\). Along every solution,
each balancedness defect satisfies
\(\dot\Delta_e=-2\lambda\,\Delta_e\), so
\(\Delta_e(t)=e^{-2\lambda t}\,\Delta_e(0)\). In particular, every limit
point of \(\theta(t)\) as \(t\to\infty\), if any, is balanced --- hence minimum
norm --- regardless of initialization.
\end{corollary}
 
\begin{proof}
Fix \(e=\{u,v\}\) and write \(A=\mat_eW_u\), \(B=\mat_eW_v\), so
\(\Delta_e=A^{\top}A-B^{\top}B\). Under the regularized flow,
\begin{equation}
\label{eq:regularization-balances1}
    \dot A=-\nabla_A\Loss-\lambda A,
    \qquad
    \dot B=-\nabla_B\Loss-\lambda B.
\end{equation}
Also, the product rule gives
\begin{equation}
\label{eq:regularization-balances2}
    \dot\Delta_e
    =
    \dot A^{\top}A+A^{\top}\dot A
    -\dot B^{\top}B-B^{\top}\dot B.
\end{equation}
Substituting \cref{eq:regularization-balances1}
into \cref{eq:regularization-balances2} and grouping the unregularized and regularization terms gives
\[
\begin{aligned}
    \dot\Delta_e
    &=
    \bigl[-(\nabla_A\Loss)^{\top}A
    -A^{\top}\nabla_A\Loss
    +(\nabla_B\Loss)^{\top}B
    +B^{\top}\nabla_B\Loss\bigr] \\
    &\qquad+
    \bigl[(-\lambda A)^{\top}A
    +A^{\top}(-\lambda A)
    -(-\lambda B)^{\top}B
    -B^{\top}(-\lambda B)\bigr].
\end{aligned}
\]
The first bracket is zero by \cref{prop:balancedness-conserved}. The second
bracket is
\[
    \bigl[(-\lambda A)^{\top}A
    +A^{\top}(-\lambda A)
    -(-\lambda B)^{\top}B
    -B^{\top}(-\lambda B)\bigr]
    =
    -2\lambda A^{\top}A+2\lambda B^{\top}B
    =
    -2\lambda\Delta_e.
\]
Therefore \(\dot\Delta_e=-2\lambda\Delta_e\), and solving this linear equation gives
\(\Delta_e(t)=e^{-2\lambda t}\Delta_e(0)\). At any limit point as
\(t\to\infty\), every \(\Delta_e\) vanishes by continuity, so the point is
balanced. By \cref{prop:min-norm-equivalences} balanced is equivalent to minimum
norm, so we have the desired conclusion.
\end{proof}

Similar results are known in the literature. \citet{kunin2020neural} show similar regularization effects for conserved quantities associated with scale symmetries in neural networks, and \citet{lindsey2025regularization} show that regularization drives the flow to the balanced manifold in deep linear networks.
 
\begin{remark}
The proof of \cref{prop:balancedness-conserved} uses only the gauge invariance of \(\Loss\) at the edge \(e\), making no reference to the
tree structure, so the conservation law holds on tensor networks of
arbitrary topology. The tree structure enters only through
\cref{prop:min-norm-equivalences} in the proof of \cref{prop:min-norm-preserved}: on a general
network, a balanced initialization still remains balanced under gradient flow, but
balancedness is no longer necessarily equivalent to minimum norm. See also \cref{rem:tree-enters}.
\end{remark}

\section{A Spurious Local Minimum Without the Minimum-Norm Assumption}
\label{app:counterexample}

This appendix establishes that the minimum-norm assumption in
\cref{thm:no-spurious-local-min} is necessary: without it, realizable TTNs can
have spurious local minima. We construct an explicit example on a \(3\)-leaf
star, equivalently a Tucker decomposition, with bond dimensions equal to the
target edge ranks. The example is not minimum norm, and it is a local minimum of
the squared loss.\footnote{Recall from \cref{app:setup} that throughout the
appendices \(\Loss\) denotes the un-normalized loss
\(\tfrac12\|T(\theta)-T^*\|_F^2\); the explicit loss values below use this
convention.} A related counterexample appears in
\citet{frandsen2022optimization}.
The example necessarily uses a multilinear interaction of order at least three:
when all tensors have order two, one obtains deep linear networks, which have no
spurious local minima for the square loss
\citep{baldi1989neural,kawaguchi2016deep,lu2017depth}.

\subsection{Setup}

Consider a \(3\)-leaf star TTN with core tensor
\(G\in\R^{3\times 3\times 3}\) and leaf matrices
\(A,B,C\in\R^{3\times 3}\). The represented tensor is
\[
    T
    =
    \sum_{i,j,k=0}^{2} G_{ijk}\,
    A_{:i}\otimes B_{:j}\otimes C_{:k}
    \in \R^3\otimes\R^3\otimes\R^3,
\]
where \(A_{:i}\) denotes the \(i\)-th column of \(A\), and similarly for
\(B,C\). The target is the diagonal tensor
\[
    T^* = e_{000}+e_{111}+e_{222},
\]
where \(e_0,e_1,e_2\) is the standard basis of \(\R^3\) and
\(e_{abc}:=e_a\otimes e_b\otimes e_c\). Every matricization of \(T^*\) has rank
\(3\), so \(T^*\) is realizable at tight capacity \(r_e=3\) on every edge.

Consider the parameter point
\[
    A^0=B^0=C^0=
    \begin{pmatrix}
    1 & 0 & 0 \\
    0 & 1 & 0 \\
    0 & 0 & 0
    \end{pmatrix},
\]
and
\[
    G^0
    =
    e_{000}+e_{111}
    +e_{012}+e_{021}+e_{102}+e_{120}+e_{201}+e_{210}.
\]
Since the third columns of
\(A^0,B^0,C^0\) vanish, every core term with an index equal to \(2\) drops out
of the contraction, and therefore
\[
    T(\theta_0)=T_0=e_{000}+e_{111},
    \qquad
    T_0-T^*=-e_{222},
    \qquad
    \Loss(\theta_0)=\tfrac12.
\]

This point is not minimum norm. Indeed, the same output \(T_0\) is represented
by
\[
    A'=B'=C'=
    \begin{pmatrix}
    1 & 0 & 0 \\
    0 & 1 & 0 \\
    0 & 0 & 0
    \end{pmatrix},
    \qquad
    G'=e_{000}+e_{111},
\]
for which \(\|\theta'\|^2=8<14=\|\theta_0\|^2\).

\begin{theorem}
\label{thm:counterexample}
The point \(\theta_0=(G^0,A^0,B^0,C^0)\) is a local minimum of \(\Loss\):
there exists a neighborhood of \(\theta_0\) on which
\(\Loss\ge \Loss(\theta_0)=\tfrac12\). Since \(T(\theta_0)\neq T^*\), this
local minimum is spurious.
\end{theorem}

\begin{proof}
Let \(T=T(\theta)\) and write \(\Delta T:=T-T_0\). The loss difference is
\[
    \Loss(\theta)-\Loss(\theta_0)
    =
    \langle \Delta T,T_0-T^*\rangle
    +\tfrac12\|\Delta T\|_F^2 .
\]
Since \(T_0-T^*=-e_{222}\), the alignment term $\langle \Delta T, T_0 - T^*\rangle$ is exactly \(-T_{222}\). This is the only term in the expansion that can be negative; the norm term $\|\Delta T\|_F^2$ is nonnegative. Thus any loss decrease must come from making \(T_{222}\) positive. However, we will show that attempting to make \(T_{222}\) positive forces an increase in the norm term, which dominates any possible decrease from the alignment term.

Precisely, let
\[
    S
    :=
    \sum_{(i,j,k)\notin\{(0,0,0),(1,1,1),(2,2,2)\}} T_{ijk}^2,
\]
the squared norm of the target-zero coefficients. It suffices to prove that,
after restricting to a sufficiently small neighborhood of \(\theta_0\),
\[
    |T_{222}|\le C S^{3/2}
\]
for some constant \(C>0\), since then \(\|\Delta T\|_F^2\ge S\) and
\[
    \Loss(\theta)-\Loss(\theta_0)
    \ge
    \tfrac12S-CS^{3/2}.
\]

For a \(3\times 3\) matrix \(A\), write \(u^{(a)}\in\R^3\) for its \(a\)-th
row, so \(u^{(a)}_i=A_{ai}\), and define \(v^{(b)},w^{(c)}\) analogously for
\(B,C\). Let
\[
    F(x,y,z):=\sum_{i,j,k=0}^2 G_{ijk}x_i y_j z_k.
\]
Then \(T_{abc}=F(u^{(a)},v^{(b)},w^{(c)})\). At the base point,
\[
    u^{(0)}=v^{(0)}=w^{(0)}=e_0,\qquad
    u^{(1)}=v^{(1)}=w^{(1)}=e_1,\qquad
    u^{(2)}=v^{(2)}=w^{(2)}=0.
\]
For a nearby parameter point, set
\[
    a:=u^{(2)},\qquad b:=v^{(2)},\qquad c:=w^{(2)}.
\]

We first prove that \(S\) controls the dormant rows \(a,b,c\). Consider
\[
    \Phi_A(a)
    :=
    \bigl(
    F(a,v^{(0)},w^{(0)}),
    F(a,v^{(1)},w^{(1)}),
    F(a,v^{(0)},w^{(1)}),
    F(a,v^{(1)},w^{(0)})
    \bigr).
\]
These four components are \(T_{200},T_{211},T_{201},T_{210}\). At
\(\theta_0\),
\[
    \Phi_A(a_0,a_1,a_2)=(a_0,a_1,a_2,a_2),
\]
so \(\|\Phi_A(a)\|^2\ge \|a\|^2\) at the base point. Since \(\Phi_A\) depends
continuously on \((G,B,C)\), after shrinking to a neighborhood
\(\mathcal N\) of \(\theta_0\) there is \(\sigma>0\) such that
\[
    T_{200}^2+T_{211}^2+T_{201}^2+T_{210}^2
    \ge \sigma^2\|a\|^2
\]
on \(\mathcal N\). The same argument with the roles of \(A,B,C\) permuted gives
analogous bounds for \(\|b\|\) and \(\|c\|\), using coefficients also included
in \(S\). Thus, shrinking \(\mathcal N\) if necessary,
\[
    S\ge \sigma^2(\|a\|^2+\|b\|^2+\|c\|^2)
\]
for all \(\theta\in\mathcal N\).

Next, \(T_{222}=F(a,b,c)\) is trilinear. On a sufficiently small neighborhood
there is a constant \(M>0\) such that
\[
    |F(x,y,z)|\le M\|x\|\|y\|\|z\|
    \qquad
    \text{for all } x,y,z\in\R^3 .
\]
Let \(m^2:=\|a\|^2+\|b\|^2+\|c\|^2\). By AM--GM,
\[
    \|a\|\|b\|\|c\|
    \le
    \left(\frac{m^2}{3}\right)^{3/2}.
\]
Since \(m^2\le \sigma^{-2}S\), there is a constant \(C>0\) such that
\[
    |T_{222}|\le C S^{3/2}
\]
throughout \(\mathcal N\).

Shrinking \(\mathcal N\) once more, we may assume \(CS^{1/2}\le\tfrac14\) on
\(\mathcal N\). Hence
\[
    \Loss-\tfrac12\ge \tfrac14S\ge0.
\]
Thus \(\theta_0\) is a local minimum. The minimum is spurious because
\(T(\theta_0)=T_0\neq T^*\).
\end{proof}

\section{The Parity Example}
\label{app:parity}

In this appendix we prove \cref{prop:parity-order}.

Recall the setup of \cref{subsec:parity-ttn}, including the parity target $T^*=\mathbf{1}[\,y=x_1\oplus\cdots\oplus x_n\,]$, the binary tree tensor network architecture $T(\theta)$, and the parameter point $\theta_0 \in \Theta$. We assume \(n \geq 4\) throughout. Also recall from the notation in \cref{app:setup} that throughout the appendices $\Loss$ denotes the un-normalized loss $\tfrac12\|T(\theta) - T^*\|_F^2$.

Let \(u_+=(1,1)\) and \(u_-=(1,-1)\), the orthogonal Hadamard basis of
\(\R^2\). The following \cref{lem:parity-hadamard-form} shows that the model and target are diagonal in Hadamard coordinates.

\begin{lemma}
\label{lem:parity-hadamard-form}
It holds that:
\begin{enumerate}
    \item \(W_v(\theta_0)=\tfrac12 u_+^{\otimes 3}\) for every node \(v \in V\);
    \item \(T(\theta_0)=\tfrac12 u_+^{\otimes(n+1)}\);
    \item \(T^*=\tfrac12\bigl(u_+^{\otimes(n+1)}+u_-^{\otimes(n+1)}\bigr)\).
\end{enumerate}
Consequently,
\[
    R_0:=T(\theta_0)-T^*
    =
    -\tfrac12 u_-^{\otimes(n+1)}.
\]
\end{lemma}

\begin{proof}
The tensor \(W_v(\theta_0)\) has every entry equal to \(1/2\), while
\(u_+^{\otimes 3}\) has every entry equal to \(1\), proving (1).
Since the tree has \(n-1\) nodes and \(n-2\) internal edges, a direct
calculation gives
\[
    T(\theta_0)
    =
    \bigl(\tfrac12\bigr)^{n-1}2^{n-2}u_+^{\otimes(n+1)}
    =
    \tfrac12u_+^{\otimes(n+1)},
\]
proving (2).

Since \((u_+)_b=1\) and
\((u_-)_b=(-1)^b\) for \(b\in\{0,1\}\),
\[
\begin{aligned}
    \left[
        \tfrac12\bigl(u_+^{\otimes(n+1)}
        +u_-^{\otimes(n+1)}\bigr)
    \right]_{x_1,\ldots,x_n,y}
    &=
    \tfrac12
    \left(
        \prod_{i=1}^n (u_+)_{x_i}(u_+)_{y}
        +
        \prod_{i=1}^n (u_-)_{x_i}(u_-)_{y}
    \right) \\
    &=
    \tfrac12\left(1+(-1)^{x_1+\cdots+x_n+y}\right) \\
    &=
    \mathbf{1}[\,y=x_1\oplus\cdots\oplus x_n\,].
\end{aligned}
\]
This is the definition of the parity target \(T^*\), proving (3).
\end{proof}

\begin{lemma}
\label{lem:parity-residual-projection}
There exist constants $C,\rho>0$ such that, for all $\theta \in \Theta$ with
\(\|\theta-\theta_0\|\leq \rho\),
\[
\left|\bigl\langle R_0,\; T(\theta)-T(\theta_0) \bigr\rangle\right|
\;\leq\;
C\|\theta-\theta_0\|^{n/2+1}.
\]
\end{lemma}

\begin{proof}
Set $\delta W_v := W_v(\theta)-W_v(\theta_0)$. For \(S\subseteq V\), define
\(\theta^S\) by
\[
    W_v(\theta^S)
    =
    \begin{cases}
        \delta W_v, & v\in S,\\
        W_v(\theta_0), & v\notin S.
    \end{cases}
\]
By multilinearity of the TTN contraction,
\[
T(\theta)-T(\theta_0)
\;=\;
\sum_{S \subseteq V, S \neq \emptyset } T(\theta^S).
\]
Then by \cref{lem:parity-hadamard-form}, we have
\begin{equation}
\label{eq:parity-sum}
\langle R_0, T(\theta)-T(\theta_0)
\rangle \;=\;
-\tfrac12 \sum_{S \subseteq V, S \neq \emptyset} \langle u_-^{\otimes(n+1)}, T(\theta^S) \rangle.
\end{equation}
We will show that the only nonzero terms in this sum are those where $S$ includes every leaf and root tensor. Suppose first that an
input leaf \(\ell \in V\) lies outside \(S\). Then by first performing all contractions away from \(\ell\), there exists a vector \(a \in \R^2\) such that
\[
    \langle u_-^{\otimes(n+1)}, T(\theta^S) \rangle
    =
    \bigl\langle W_\ell(\theta_0), u_-\otimes u_-\otimes a \bigr\rangle
    =
    \tfrac12\langle u_+,u_-\rangle^2\langle u_+,a\rangle
    =
    0.
\]
Similarly, if the root \(r \in V\) lies
outside \(S\), there exists \(a,b \in \R^2\) such that
\[
    \langle u_-^{\otimes(n+1)}, T(\theta^S) \rangle
    =
    \bigl\langle W_r(\theta_0), a\otimes b\otimes u_- \bigr\rangle
    =
    \tfrac12\langle u_+,a\rangle\langle u_+,b\rangle
    \langle u_+,u_-\rangle
    =
    0 .
\]
Thus $\langle u_-^{\otimes(n+1)}, T(\theta^S) \rangle = 0$ for any term of \cref{eq:parity-sum}
where $S$ does not contain every input leaf and the root.

There are \(n/2\) input leaves and one root, so every nonvanishing subset has size
\(|S|\geq n/2+1\). For each such \(S\), the scalar $\bigl\langle R_0, T(\theta^S) \bigr\rangle$
is multilinear in the tensors ${\delta W_v : v\in S}$, with all other
node tensors fixed. Since there are only finitely many subsets $S$ in \cref{eq:parity-sum}, there is
a constant $C'$ independent of $S$ such that, for all $\theta$ sufficiently close to $\theta_0$,
\[
\left|
\bigl\langle R_0, T(\theta^S) \bigr\rangle
\right|
\leq
C' \prod_{v\in S}\|\delta W_v\|_F
\leq
C'\|\theta-\theta_0\|^{|S|}.
\]
Shrinking $\rho$ if necessary so that \(\|\theta-\theta_0\|\leq 1\), each
surviving term in \cref{eq:parity-sum} is bounded by \(C'\|\theta-\theta_0\|^{n/2+1}\). Summing over
the finitely many surviving subsets gives the claim.
\end{proof}

\parityorder*

\begin{proof}
Write \(\Delta T:=T(\theta)-T(\theta_0)\). Since
\[
    \Loss(\theta)-\Loss(\theta_0)
    =
    \langle R_0,\Delta T\rangle+\tfrac12\|\Delta T\|_F^2,
\]
\cref{lem:parity-residual-projection} gives, for \(\theta\) sufficiently close
to \(\theta_0\),
\begin{equation}
\label{eq:parity-loss-bound}
\Loss(\theta)-\Loss(\theta_0)
\geq
-C\|\theta-\theta_0\|^{n/2+1},
\end{equation}
after possibly increasing \(C\). This is the asserted order bound. 

Since $n/2+1>1$, criticality follows from \cref{eq:parity-loss-bound} and the differentiability of $\Loss(\theta)$.
\end{proof}

\begin{remark}
\label{rem:sharp-order}
The bound is likely not tight. We conjecture the saddle order is in
fact $n - 2$, which we have checked numerically at $n \in \{4, 8\}$.
The exact order is not important for the present paper, and we leave
this to future work.
\end{remark}

\section{Implications for Learning With Stochastic Gradient Methods}
\label{app:empirical-population-gap}

The loss landscape results in the main body concern the geometry of the \textit{population} loss $\Loss(\theta) = \frac{1}{2|\mathcal{X}|}\|T(\theta) - T^*\|_F^2$. However, the population loss is a theoretical quantity. In practice, one does not know the target $T^*$, and instead only has access to a finite number of samples $(x_i, y_i)$ from the target. Instead of population gradient descent, one typically learns via a form of stochastic gradient descent using the per-sample losses $\ell(x_i,y_i;\theta)$.

The relevance of the population loss to stochastic gradient methods is standard and well-known \citep{robbins1951stochastic, bottou2018optimization}, but we explain this perspective here in our setting for the interested reader. We show how minibatch SGD can be viewed, under standard approximations, as noisy gradient descent on this population loss (\cref{app:empirical-population-gap-setup}). We then sketch how local minima (\cref{app:empirical-population-gap-benign}) and saddle points (\cref{app:empirical-population-gap-saddle}) can affect SGD.

\subsection{Setup}
\label{app:empirical-population-gap-setup}

Recall that
\[
    \E_{(x,y)}\!\left[\ell(x, y; \theta)\right] \;=\; \,\Loss(\theta),
    \qquad
    \E_{(x,y)}\!\left[\nabla \ell(x, y; \theta)\right] \;=\; \nabla\Loss(\theta),
\]
so the per-sample losses and gradients are unbiased noisy estimates of $\Loss(\theta)$ and $\nabla \Loss(\theta)$.

Stochastic gradient descent with step size $\tau$ and minibatches of size $m$ can then be written as the discrete-time stochastic process
\[
    \theta_{t+1} \;=\; \theta_t \;-\; \tau\,\nabla\Loss(\theta_t) \;-\; \tau\,\xi_t(\theta_t),
\]
where $\xi_t(\theta) = \frac{1}{m}\sum_i \nabla_\theta \ell(x_i, y_i; \theta) -
\nabla\Loss(\theta)$ is a mean-zero noise term, independent across $t$ conditional
on $\theta_t$, with covariance scaling as $1/m$.

Note that this is a sum of a population gradient term $\tau\,\nabla\Loss(\theta_t)$ and a noise term $\tau\,\xi_t(\theta_t)$; the only difference between SGD
and population gradient descent on $\Loss$ is the noise term. Thus SGD acts like noisy gradient descent on the population loss $\Loss$.

Up to this point we have not made any approximations or new assumptions; we now proceed to make several standard assumptions to make analysis tractable. First, it is useful to take a continuous-time limit, which approximates SGD via the stochastic differential equation:
\[
d\theta_t = - \nabla \Loss(\theta_t) dt + (\tau \, \Sigma(\theta_t))^{\frac{1}{2}} \, dW_t
\]
where $dW_t$ is the Wiener process (Brownian motion) and $\Sigma: \Theta \rightarrow \Theta \times \Theta$ is a position-dependent covariance matrix \citep{li2017stochastic}.
Note that we have assumed that the noise process is Brownian, which follows from the central limit theorem if the batch size is large and gradient noise is finite variance, but in some regimes this noise process can be more heavy-tailed \citep{simsekli2019tail, hodgkinson2021multiplicative}, a complication we temporarily put aside for ease of discussion.

Second, we make the assumption that the noise covariance is isotropic and homogeneous: $\Sigma(\theta_t) = \sigma I$ for some $\sigma \in \R$. This assumption is unrealistic, but it significantly eases discussion without fundamentally changing the qualitative conclusions. 

Under these assumptions, SGD is modeled by
the overdamped Langevin equation
\[
    d\theta_t
    =
    -\nabla \Loss(\theta_t)\,dt
    +
    \sqrt{2\varepsilon}\,dW_t,
\]
where the effective noise scale is \(\varepsilon = \tau\sigma/2\). Thus the stochastic process still
``sees'' the same potential function \(\Loss\): the population gradient
determines the deterministic drift, while the Brownian term allows occasional
moves against the gradient.

\subsection{Benign Loss Landscapes Help SGD}
\label{app:empirical-population-gap-benign}

This viewpoint makes clear why suboptimal local minima are a serious obstruction
for SGD, even though SGD is noisy. Suppose \(\theta_0\) is a strict local
minimum of \(\Loss\), but not a global minimum, and let \(B\) denote its basin of
attraction under gradient flow. To leave this basin, the process must cross the
boundary of \(B\). The relevant quantity is therefore the loss barrier\footnote{In the case where the noise covariance is not isotropic and homogeneous, the relevant quantity becomes the Freidlin-Wentzell quasipotential rather than the loss \citep{freidlin1998random}}.
\[
    \Delta_B
    :=
    \inf_{\theta\in\partial B} \Loss(\theta) - \Loss(\theta_0).
\]
If \(\Delta_B > 0\), then escaping the basin requires the stochastic dynamics to
move through points of higher loss. In the small-noise regime, this is a rare
event: the classical Eyring--Kramers law \citep{eyring1935activated, kramers1940brownian} predicts exit
times which scale, to leading exponential order, like
\[
    \E[t_{\text{exit}}]
    \asymp
    \exp\!\left(\frac{\Delta_B}{\varepsilon}\right),
\]
with problem-dependent prefactors determined by the local geometry near the
minimum and the geometry of the basin boundary
\citep{berglund2008eyring}.
In terms of SGD, this means that a suboptimal basin whose loss barrier is large
relative to the effective noise scale can trap the algorithm for extremely long
times. Increasing the batch size or decreasing the step size reduces the noise
scale and makes the dynamics closer to population gradient flow, so such barriers
become even harder to cross.

This is the basic sense in which a benign population loss landscape helps SGD.
If the population loss has a genuine suboptimal local minimum, then deterministic
gradient flow is trapped, and SGD can escape only by a sufficiently large noise
fluctuation, which can become vanishingly unlikely if the loss barrier is large. But if the population loss landscape has no suboptimal local minima, or at least admits non-increasing escape paths to lower loss, then this particular failure mode is absent.

Of course, one can instead increase stochasticity in SGD, for instance by
increasing the step size, decreasing the batch size, or considering noise models
beyond the Brownian approximation. Heavy-tailed gradient noise, in particular,
can lead to qualitatively different escape behavior from local minima
\citep{simsekli2019tail,nguyen2019first}. But this is not a substitute for a
benign population landscape. Large stochastic fluctuations can also move the
parameters to regions of higher loss, and increasing the noise scale weakens the
relative influence of the population gradient. Thus noise can help escape a bad
basin, but it is a costly and indirect mechanism; we may need it sometimes nevertheless\footnote{For instance, our results rely on several idealized assumptions (\cref{subsec:tightness}) - if in practice the loss landscape is not exactly benign but instead approximately benign, stochasticity may be able to make up the difference.}, but it is preferable not to.

Thus it is much better, from the perspective of optimization, for the population landscape not to contain suboptimal wells in the first place. This is why our benign loss-landscape results for $\Loss$ are not just pure geometry statements but actually relevant for the behavior and efficiency of SGD in TTNs.

\subsection{Degenerate Saddle Points Can Still Slow SGD}
\label{app:empirical-population-gap-saddle}

The preceding discussion explains why the absence of suboptimal local minima is
useful for SGD: it removes energetic barriers. However, it does not imply that
SGD makes rapid progress. A landscape can have no bad local minima while still
containing saddle points whose escaping directions are difficult to find or
exploit.

The important distinction is between nondegenerate and
degenerate saddles. At a strict saddle, the Hessian has a negative eigenvalue.
Equivalently, in suitable local coordinates there is a direction \(v\) such that
\[
\Loss(\theta_0 + s v)-\Loss(\theta_0)
=
-\frac{\lambda}{2}s^2 + O(s^3)
\]
with \(\lambda>0\). Along this direction, the linearization of gradient flow is
unstable: a small displacement in the \(v\)-direction grows approximately like
\(e^{\lambda t}\). Thus, once the dynamics acquire even a small component in the
unstable direction, the saddle can be escaped in time logarithmic in the inverse
size of that component. This is the regime captured by much of the saddle-escape
literature: explicit perturbations or stochastic gradients can seed a component
in a negative-curvature direction, after which the deterministic drift amplifies
it \citep{jin2017escape,daneshmand2018escaping}.

This means that ordinary strict saddles are not a plausible mechanism for worst-case hardness, at least when the negative curvature is bounded below
uniformly. A degenerate saddle (such as that exhibited by parity in \cref{subsec:parity-ttn}) is different, however. At such a point, there may be no useful
negative-curvature direction at quadratic order. Descent may appear only through
higher-order coordinated perturbations. A local form such as
\[
\Loss(\theta_0+s v)-\Loss(\theta_0) \approx -c s^k,
\qquad k>2,
\]
illustrates the change. The corresponding gradient-flow equation along this
coordinate has the form
\[
\dot{s} \approx c k s^{k-1}.
\]
Starting from a small seed \(s_0>0\), the time needed to reach an \(O(1)\)
distance is no longer logarithmic in \(1/s_0\); for \(k>2\), it scales like a
power of \(1/s_0\). The loss improvement
\(-c s^k\) is small, and the vector field itself is weak near the saddle. Higher order means that small perturbations are amplified much more slowly.

This is the sense in which degenerate saddles can reconcile benign landscapes
with computational hardness. A benign landscape can rule out suboptimal wells:
there may always exist a non-increasing path to lower loss. But the beginning of
that path may require a coordinated higher-order move, and the gradient signal
toward that move may vanish to high order. SGD is then not blocked by an
energetic barrier; rather, it can be slowed by a need for
coordination and the resulting loss flatness.

Note that in the presence of stochastic noise, high-order degeneracy may be necessary but not sufficient for optimization difficulty, and one may also need the escape path to be sufficiently unlikely to be found by chance. One must return to the classical Eyring--Kramers
picture discussed in \cref{app:empirical-population-gap-benign}. For metastable Langevin dynamics near a local minimum, the
leading exponential contribution to the exit time comes from the loss
barrier. When the minimum is instead a saddle point there is no loss barrier, and this exponential obstruction is absent.
However, the local geometry near the saddle still affects the prefactor, and
nonquadratic saddles require modified Eyring--Kramers laws
\citep{berglund2008eyring}.

Thus the absence of a loss barrier does not by
itself imply that escape is geometrically easy; degenerate saddles can still
change the relevant time scale through the prefactor term. We believe a rigorous analysis of the Eyring--Kramers prefactor term in such saddle points, extending \citet{berglund2008eyring}, to be an interesting direction for future work.

\section{Autoformalization in \texorpdfstring{\textnormal{\textsc{Lean}}}{Lean}}
\label{app:formalisation}

\refstepcounter{footnote}

Formalizing mathematics in proof assistants such as \textsc{Lean}~4 \citep{mouraLean4Theorem2021,LeanMathematicalLibrary2020} has traditionally required substantial expertise both in the underlying mathematics and in the proof assistant itself. Recent advances in large language model capabilities have made formal verification of research-level mathematics increasingly practical (see e.g.\citep{wuAutoformalizationLargeLanguage2022,ilinSemiAutonomousFormalizationVlasovMaxwellLandau2026}). In this appendix, we propose that suitable theoretical results in machine learning, particularly self-contained results about tractable toy models, should increasingly be accompanied by a machine-checked \textsc{Lean} formalization.

\par
We formalized \cref{thm:no-spurious-local-min} and \cref{thm:no-spurious}, the two results most central to the paper's loss landscape claims. We selected these results both for their mathematical importance and because their dependencies could be developed using just the existing Mathlib infrastructure. A clickable \textsc{Lean} logo in the right margin appears beside formalized statements. \Cref{subsec:formalisation-process} describes the autoformalisation and verification workflow; \cref{subsec:formalisation-correspondence} explains the paper-to-\textsc{Lean} correspondence and differences in proof organisation; and \cref{subsec:formalisation-position} argues more broadly for the adoption of \textsc{Lean} certificates for suitable theoretical research in machine learning. The supplementary library is available in the companion GitHub repository\textsuperscript{\thefootnote}.
\footnotetext[\value{footnote}]{\url{https://github.com/yangdabei/ttn-loss-landscape}}

\subsection{Development and Verification}
\label{subsec:formalisation-process}

The authors specified the target results, and the formalization was
generated and reviewed using Anthropic's Claude Opus~4.8 and
Claude Fable~5, and OpenAI's GPT-5.5 via subscriptions. We do not report an
API-equivalent cost because complete records of token usage,
elapsed time, and cost were not retained.

The library uses \textsc{Lean}~4.31.0 and Mathlib~4.31.0. The files contain no \texttt{sorry},
\texttt{admit}, \texttt{native\_decide}, or project \texttt{axiom}
declarations. The proofs of
\cref{thm:no-spurious-local-min,thm:no-spurious} are checked with
Comparator \citep{leanFroComparator}. A separate
\texttt{Challenge.lean} file states the target theorems without
importing the modules that prove them. Comparator checks that these
statements agree with the production statements, including the
definitions they reference, verifies that the proofs use only
explicitly permitted axioms, and replays the proof terms through
the \textsc{Lean} kernel. These checks guard against the substitution
of weaker statements and the introduction of hidden axioms.

Still, Comparator cannot establish whether the formal statements faithfully
express the manuscript's claims. The authors therefore reviewed the
statements of the two main results, including the relevant
definitions and hypotheses, and affirm their correspondence to the
manuscript. This review was at the theorem level: it did not
systematically match intermediate results or proof steps to
\textsc{Lean} declarations, nor did it include a line-by-line audit
of the generated proof scripts.

\subsection{Paper-to-\texorpdfstring{\textnormal{\textsc{Lean}}}{Lean} Correspondence and Proof Organization}
\label{subsec:formalisation-correspondence}

The formalization focuses on
\cref{thm:no-spurious-local-min,thm:no-spurious}, whose \textsc{Lean}
statements appear in \cref{fig:lean-principal-results}, with margin
markers linking to the corresponding code. It preserves the manuscript's
definitions and assumptions, including realizability and minimum norm,
but differs in the following representations and proof steps:

\begin{figure*}[t]
\centering
\begin{minipage}[t]{0.48\textwidth}
\vspace{0pt}
\begin{lstlisting}[style=leanstatement]
theorem minNorm_isLocalMin_isGlobalMin
    {Tstar : a.Ext → ℝ}
    (hreal : a.Realizable Tstar)
    {θ : a.Param}
    (hmn : a.MinNorm θ)
    (hloc : IsLocalMin (a.loss Tstar) θ) :
    a.IsGlobalMin Tstar θ ∧
      a.loss Tstar θ = 0
\end{lstlisting}
\end{minipage}
\hfill
\begin{minipage}[t]{0.48\textwidth}
\vspace{0pt}
\begin{lstlisting}[style=leanstatement]
theorem not_fullRank_of_critical_of_not_isGlobalMin
    {Tstar : a.Ext → ℝ}
    (hreal : a.Realizable Tstar)
    {θ : a.Param}
    (hcrit : a.Critical Tstar θ)
    (hng : ¬ a.IsGlobalMin Tstar θ) :
    ¬ a.FullTuckerRank θ
\end{lstlisting}
\end{minipage}
\caption{Formal statements of the two main results, with proof
bodies omitted. Left: \cref{thm:no-spurious-local-min}, stating that
a minimum-norm local minimum for a realizable target is global and
has zero loss. Right: \cref{thm:no-spurious}, stating that a critical point that is not
a global minimum cannot have full Tucker rank.}
\label{fig:lean-principal-results}
\end{figure*}

\begin{itemize}
    \item \textbf{Contraction coordinates.}
    The formalization defines the represented tensor directly
    as a sum over internal bond indices of products of node-tensor
    entries, implementing the undirected contraction description
    of \cref{def:ttn-alternate}.
    The correspondence with the directed multilinear-map
    definition of \cref{def:ttn} is explained mathematically in
    \cref{app:multilinear-background}; this correspondence is
    not itself formalized in \textsc{Lean}.
    
    \item \textbf{External modes.}
    \textsc{Lean} combines the external modes at each node into a single
    finite index type. The theorem
    \href{\leanbase/TTN/Landscape/External.lean\#L221-L223}
    {\texttt{representedLit\_extEquiv}}
    proves that contraction with individual external modes agrees
    with this aggregated representation under reindexing. 

    \item \textbf{Criticality.}
    The predicate \href{\leanbase/TTN/Landscape/CriticalityBasic.lean\#L52-L53}
    {\texttt{Critical}} is defined using Mathlib's
    \href{https://github.com/leanprover-community/mathlib4/blob/fabf563a7c95a166b8d7b6efca11c8b4dc9d911f/Mathlib/Analysis/Calculus/FDeriv/Defs.lean\#L121-L122}
    {\texttt{HasFDerivAt}}, requiring the loss to have zero
    Fréchet derivative at the parameter point.
    The theorem
    \href{\leanbase/TTN/Landscape/CriticalityBasic.lean\#L199-L202}
    {\texttt{critical\_iff\_nodewiseCritical}}
    proves that this is equivalent to the vanishing of
    the first-order loss variation with respect to each
    node tensor, the characterization used in the algebraic proofs.

    \item \textbf{The lemma used in \cref{thm:no-spurious}.}
    The manuscript's \cref{lem:full-env-rank} establishes full column
    rank for both contraction factors at any internal edge.
    The \textsc{Lean} proof uses only the leaf-edge case:
    \href{\leanbase/TTN/Landscape/FullRank.lean\#L1153-L1157}
    {\texttt{Hfun\_full\_col\_rank}} proves that the factor obtained by
    contracting the rest of the tree has full column rank.
    This suffices for the induction by leaf removal.
\end{itemize}

Not all statements were chosen to be included in the formalization. For example, we did not include \cref{cor:escape-paths} in the formalization. Its proof relies on results
from real algebraic geometry and geometric invariant theory
(see \cref{app:minimum-norm,app:escape}) whose formalization lies
outside our scope.
We instead focused on the two main results, whose proofs mostly use linear algebra and some graph theory that can be formalized using existing Mathlib infrastructure.

The formalization process did not uncover any major flaws with the proof, but it did help expose and fix one particular edge case. Specifically, the case where \(T^*=0\) presented an issue for the reduction argument of \cref{subsec:full-rank-reduction}, which replaces each bond
dimension \(r_e\) with \(\operatorname{rank}(T^{*(e)})\).
When \(T^*=0\), these ranks are zero, so the compressed network would
violate the requirement of \cref{def:ttn} that bond dimensions be positive.
The argument in \cref{subsec:full-rank-reduction} now includes an explicit argument to handle the \(T^*=0\) case.\footnote{In the \textsc{Lean} code, this case is treated directly in the proof of \cref{thm:no-spurious-local-min}.}

\subsection{Formalization as a Verification Check for Toy Models}
\label{subsec:formalisation-position}

Theoretical machine learning often uses simplified models to isolate
mechanisms that are difficult to study in full neural networks.
Results about these models inform our understanding of optimization,
generalization, and the effects of architectural assumptions.
Several such results have been formally verified, including an
expressiveness theorem for convolutional arithmetic circuits
\citep{bentkampFormalProofExpressiveness2019}, convergence results
for Hopfield networks and Boltzmann machines
\citep{cipollinaFormalizedHopfieldNetworks2025}, and a conversion
of networks with piecewise-affine activations into piecewise-affine
functions \citep{aleksandrovFormalizingPiecewiseAffine2023}.
These projects also provide infrastructure for further work.
For example, \citet{bentkampFormalProofExpressiveness2019}
developed a reusable tensor library and simplified and generalized
the expressiveness proof during formalization.

For research built around a toy model, this reuse is particularly
valuable. Results about expressiveness, optimization, and learning
dynamics can share formal definitions of the architecture and its
parameters, allowing assumptions to be compared directly and proofs
to build on previously checked lemmas. In the present library,
the two main theorems share definitions of the TTN architecture,
contraction, loss, and rank conditions. These definitions remain available for further
analysis of the model.

Furthermore, recent advances in language models and agent workflows support
increasingly autonomous formalization
\citep{zhangLeanMarathonReliableAI2026,milikicLeanFlowCaseStudy2026}.
Authors can specify the target results and delegate the implementation
to agents, while retaining responsibility for checking that the formal
statements express the intended mathematical claims. This was the
division of work in our project. For results whose dependencies are
available in existing libraries, we therefore propose including
machine-checked proofs as supplementary material.

Of course, verification can only establish the mathematical consequences of a
toy model's assumptions; whether those assumptions meaningfully capture phenomena of interest remains an empirical question.

\end{document}